%% file: _main.tex
\documentclass{article}

\input{_header}
\usepackage{autonum} %
\title{Instance-Optimal Estimation \\ with Multiple LLM Judges on a Budget}

\newcommand{\estivwe}{\textsc{Est-IVWE}}
\newcommand{\ivwe}{\textsc{IVWE}}

\author{%
Junghyun Lee\thanks{Equal contributions} \\
KAIST AI\\
\texttt{jh\_lee00@kaist.ac.kr}
\And
Sanghwa Kim$^*$ \\
KAIST AI\\
\texttt{tkdghk9667@kaist.ac.kr}
\And
Yassir Jedra \\
ICL EEE\\
\texttt{y.jedra@ic.ac.uk}
\And
Alexandre Prouti\`{e}re \\
KTH EECS, Digital Futures\\
\texttt{alepro@kth.se}
\And
Se-Young Yun \\
KAIST AI\\
\texttt{yunseyoung@kaist.ac.kr}
}

\begin{document}

\maketitle
\begin{abstract}
    Evaluating large language models increasingly relies on LLM-as-a-judge protocols, but such evaluations remain costly: different judges have different prices and reliabilities, and the difficulty of each prompt-response pair can vary substantially. This raises a basic allocation question: under a fixed budget, how should one distribute evaluation queries across heterogeneous judges and instances to obtain the most accurate score estimates? We formalize this question as \emph{budgeted heteroskedastic multi-judge estimation} under independent, unbiased score observations. Given $K$ prompt-response pairs and $J$ judges with known costs but unknown, potentially heterogeneous query–judge variances, our goal is to estimate the underlying bounded score vector while minimizing an $\ell_p$-error. Our first contribution is to analyze the inverse-variance weighted estimator (IVWE) and to derive the oracle allocation that minimizes its leading error bound. Since this allocation depends on the unknown variances, we then address the unknown-variance setting by proposing EST-IVWE, an adaptive algorithm that constructs and uses optimistic variance estimates to stabilize the empirical allocation. We prove that EST-IVWE matches the oracle IVWE rate up to lower-order terms in the budget. Our second and central theoretical contribution is a local minimax lower bound matching the leading allocation-dependent rate, establishing asymptotic instance-optimality up to logarithmic factors. An Assouad-type construction based on local, variance-preserving perturbations yields this bound both in expectation and with constant probability. A complementary Fano-type change-of-measure argument captures the logarithmic dependence on the confidence level with a different (suboptimal) allocation geometry. Finally, we numerically validate our approach on synthetic data and resampled HelpSteer2 evaluations, demonstrating improvements over naïve uniform allocation.
\end{abstract}

\input{001Introduction}

\input{002ProblemSetting}

\input{003KnownUpper}
\input{004UnknownUpper}

\input{005Lower}

\input{006Gaussian}

\input{007Experiments}

\input{008Conclusion}

\begin{ack}
    We thank the authors of \citet{saha2026budget} for providing their code and dataset for our experiments.
    We also thank B. Kveton for helpful discussions and comments.
    Last but not least, we thank the anonymous reviewers for helpful comments and suggestions that helped us improve the manuscript.
    Junghyun Lee, Sanghwa Kim and Se-Young Yun are partly supported by the Institute of Information \& Communications Technology Planning \& Evaluation (IITP) grant funded by the Korean government (MSIT) (No. RS-2019-II190075 Artificial Intelligence Graduate School Support Program (KAIST), 5\%; No. RS-2024-00457882, National AI Research Lab Project, 5\%; No. RS-2022-II220311 Development of Goal-Oriented Reinforcement Learning Techniques for Contact-Rich Robotic Manipulation of Everyday Objects, 90\%).
    Alexandre Prouti\`{e}re is supported by Vetenskaprådet, Digital Futures, and the Wallenberg AI, Autonomous Systems and Software program.
    The authors declare no competing interests.
\end{ack}

\bibliographystyle{plainnat}
\bibliography{references}

\newpage
\appendix

\crefname{appendix}{Appendix}{Appendices}
\Crefname{appendix}{Appendix}{Appendices}
\crefalias{section}{appendix}
\crefalias{subsection}{appendix}
\crefalias{subsubsection}{appendix}

\newpage
\tableofcontents
\newpage
\input{901Proof-Known}

\newpage
\input{902Proof-Unknown}
\newpage
\input{903Proof-Lower-whp}

\newpage
\input{904Proof-Lower-inexp}

\newpage
\input{905BetaConstruction}
\newpage
\input{906Gaussian}

\newpage
\input{907Proof-Gaussian}
\newpage
\input{908Experiments}
\newpage
\input{909Limitation}

\end{document}

%% file: _header.tex
\PassOptionsToPackage{sort,compress,round}{natbib}
\usepackage[final,main]{neurips_2026}

\usepackage[utf8]{inputenc}
\usepackage[T1]{fontenc}
\usepackage{url}
\usepackage{booktabs}
\usepackage{tabularx}
\usepackage{nicefrac}
\usepackage{microtype}
\usepackage{doi}
\usepackage{blindtext}
\usepackage{multicol, multirow, makecell}
\usepackage[normalem]{ulem}
\usepackage{algorithm}

\usepackage{enumerate}
\usepackage{subcaption}
\usepackage{rotating}
\usepackage{tablefootnote, threeparttable}
\usepackage[dvipsnames,table]{xcolor}

\definecolor{darkblue}{rgb}{0.0,0.0,0.65}
\definecolor{darkred}{rgb}{0.65,0.0,0.0}
\definecolor{darkgreen}{rgb}{0.0,0.5,0.0}
\definecolor{tab:blue}{RGB}{31,119,180}
\definecolor{tab:red}{RGB}{214,39,40}
\definecolor{tab:green}{RGB}{44,160,44}
\definecolor{tab:orange}{RGB}{255,127,14}
\definecolor{thmframe}{RGB}{0,114,178}
\definecolor{thmback}{RGB}{248,251,255}
\definecolor{propframe}{RGB}{0,158,115}
\definecolor{propback}{RGB}{246,252,250}
\definecolor{corframe}{RGB}{204,121,167}
\definecolor{corback}{RGB}{252,248,251}
\definecolor{assframe}{RGB}{213,94,0}
\definecolor{assback}{RGB}{255,249,245}
\definecolor{defframe}{RGB}{145,145,145}
\definecolor{defback}{RGB}{250,250,250}
\definecolor{rmkframe}{RGB}{120,120,120}
\definecolor{rmkback}{RGB}{249,249,249}
\definecolor{algframe}{RGB}{86,180,233}
\definecolor{algback}{RGB}{247,252,255}

\usepackage{hyperref}
\hypersetup{
    colorlinks = true,
    citecolor  = darkblue,
    linkcolor  = darkred,
    filecolor  = darkblue,
    urlcolor   = darkblue,
}

\usepackage{amsfonts, amsmath, amssymb, amsthm}
\usepackage{dsfont}
\usepackage{thmtools, thm-restate}
\usepackage{tcolorbox}
\tcbuselibrary{theorems,skins,breakable}
\usepackage[capitalize,noabbrev]{cleveref}
\crefname{assumption}{assumption}{assumptions}
\Crefname{assumption}{Assumption}{Assumptions}
\crefname{algocf}{algorithm}{algorithms}
\Crefname{algocf}{Algorithm}{Algorithms}
\input{math_commands} %

\newcommand{\bignorm}[1]{\left\lVert #1 \right\rVert}

\newcommand{\bigabs}[1]{\left\lvert #1 \right\rvert}

\newcommand{\indicator}{\mathds{1}}

\usepackage{pifont}

\usepackage{adjustbox}

\allowdisplaybreaks

\usepackage{enumitem}
\setlist[itemize]{leftmargin=*, labelindent = 15pt}
\setlist[enumerate]{leftmargin=*, labelindent = 15pt}

\usepackage[algo2e, linesnumbered, ruled, vlined]{algorithm2e}
\SetKwInput{KwInput}{Input}
\SetKwInput{KwOutput}{Output}

\newcommand{\junghyun}[1]{\textcolor{tab:red}{\bfseries JL: #1}}

\newcommand{\debug}[1]{\textcolor{red}{#1}}

\newif\ifFINAL
\FINALfalse %

\ifFINAL
    \renewcommand{\junghyun}[1]{}
    \renewcommand{\hanseul}[1]{}
    \renewcommand{\debug}[1]{}
    \newcommand{\kj}[1]{}
    
\else
\fi

%% file: math_commands.tex
\usepackage{amsmath,amsfonts,bm,mathtools}

\newcommand{\Beta}{\mathrm{Beta}}

\def\1{\bm{1}}

\def\vzero{{\bm{0}}}

\def\vc{{\bm{c}}}

\def\ve{{\bm{e}}}

\def\vg{{\bm{g}}}
\def\vh{{\bm{h}}}

\def\vs{{\bm{s}}}

\def\vu{{\bm{u}}}
\def\vv{{\bm{v}}}

\def\vx{{\bm{x}}}
\def\vy{{\bm{y}}}

\def\mI{{\bm{I}}}

\DeclareMathAlphabet{\mathsfit}{\encodingdefault}{\sfdefault}{m}{sl}
\SetMathAlphabet{\mathsfit}{bold}{\encodingdefault}{\sfdefault}{bx}{n}

\def\gA{{\mathcal{A}}}
\def\gB{{\mathcal{B}}}

\def\gE{{\mathcal{E}}}
\def\gF{{\mathcal{F}}}
\def\gG{{\mathcal{G}}}
\def\gH{{\mathcal{H}}}
\def\gI{{\mathcal{I}}}
\def\gJ{{\mathcal{J}}}

\def\gL{{\mathcal{L}}}

\def\gN{{\mathcal{N}}}
\def\gO{{\mathcal{O}}}

\def\gY{{\mathcal{Y}}}

\def\sN{{\mathbb{N}}}

\def\sP{{\mathbb{P}}}

\def\sR{{\mathbb{R}}}

\newcommand{\E}{\mathbb{E}}

\newcommand{\KL}{D_{\mathrm{KL}}}
\newcommand{\kl}{\mathrm{kl}}
\newcommand{\Var}{\mathrm{Var}}

\DeclareMathOperator*{\argmin}{arg\,min}

\tcbset{
    sharpthm/.style n args={2}{
        enhanced,
        breakable,
        sharp corners,
        boxrule=0.6pt,
        left=6pt,
        right=6pt,
        top=5pt,
        bottom=5pt,
        before skip=8pt,
        after skip=8pt,
        colback=#1,
        colframe=#2,
        fonttitle=\bfseries,
        coltitle=black,
        colbacktitle=#1
    },
    algobox/.style={
        enhanced,
        breakable,
        sharp corners,
        boxrule=0.6pt,
        left=6pt,
        right=6pt,
        top=5pt,
        bottom=5pt,
        before skip=8pt,
        after skip=8pt,
        colback=algback,
        colframe=algframe,
        coltitle=black,
        colbacktitle=algback,
        fonttitle=\bfseries
    }
}
\theoremstyle{plain}
\newtheorem{theorem}{Theorem}[section]
\newtheorem{proposition}[theorem]{Proposition}
\newtheorem{lemma}[theorem]{Lemma}

\theoremstyle{definition}
\newtheorem{definition}[theorem]{Definition}
\theoremstyle{plain}
\newtheorem{assumption}{Assumption}
\newtheorem{remark}{Remark}

\tcolorboxenvironment{theorem}{sharpthm={thmback}{thmframe}}
\tcolorboxenvironment{corollary}{sharpthm={corback}{corframe}}
\tcolorboxenvironment{definition}{sharpthm={defback}{defframe}}
\tcolorboxenvironment{assumption}{sharpthm={assback}{assframe}}
\tcolorboxenvironment{lemma}{sharpthm={defback}{defframe}}

%% file: 001Introduction.tex
\section{Introduction}
\label{sec:introduction}

Evaluating the quality of responses produced by large language models (LLMs) is a central component of modern model development, particularly for post-training~\citep{Ouyang2022TrainingLM, rafailov2023direct, guo2025deepseek, bai2022constitutional} and model selection~\citep{fernandez2026radar}. Traditional evaluation methods, however, face important limitations. Human annotation remains the gold standard for capturing human preferences and judgments~\citep{Ouyang2022TrainingLM, rafailov2023direct}, but it is costly, time-consuming, and difficult to scale with the rapid pace of model development. Standardized benchmarks provide a more scalable alternative, but they are often too rigid for emerging applications where suitable benchmarks do not yet exist. These limitations have motivated the growing use of the ``LLM-as-a-judge'' framework, in which strong frontier models are used to evaluate model responses across diverse tasks~\citep{zheng2023judging, kim-etal-2024-prometheus,gu2025surveyllmasajudge}.

While LLM-as-a-judge reduces reliance on costly human feedback, it still has computational and statistical concerns.
Querying large frontier models at scale remains expensive, and different judges may have very different costs, expertise, and reliability~\citep{kim-etal-2024-prometheus, li2024llms, raju-etal-2024-constructing}. A uniform allocation of evaluation queries is therefore typically inefficient: a small specialized judge may be more reliable than a costly general-purpose model on some domains, while using a frontier model for straightforward instances may waste budget~\citep{fernandez2026radar}.
The difficulty is that, although the cost of each judge is usually known, the variability of its evaluations is not.
This variability is \emph{heteroskedastic}: it depends both on the judge and on the intrinsic difficulty of the prompt-response pair being evaluated. This raises the central question studied in this work:
\begin{center}
    {\it Given a fixed evaluation budget, how should one allocate queries across prompt-response pairs and multiple LLM judges in order to obtain the most accurate estimates of response quality?}
\end{center}

We formalize this question as a new statistical learning problem, referred to as \emph{{budgeted heteroskedastic multi-judge estimation}}. This problem can be viewed as an adaptive, cost-aware generalization of Neyman allocation \citep{neyman1934}.
In classical Neyman allocation, samples are allocated across strata with known variances.
Here, the learner must allocate evaluation queries across query--judge pairs with known costs but unknown, heteroskedastic variances, with the goal of estimating the full score vector as accurately as possible. Formally, given a total budget \(B\), a set of queries \(k \in [K] \coloneq \{1,\ldots,K\}\), and a set of LLM judges \(j \in [J]\), each judge \(j\) has a known query cost \(c_j\), while each query--judge pair \((k,j)\in [K]\times[J]\) has an unknown variance \(\sigma_{k,j}\).
The learner’s objective is to adaptively allocate queries to judges so as to estimate the true score vector ${\vs} = (s_k)_{k\in[K]} \in \mathbb{R}^K$ while minimizing the \(\ell_p\)-error, for a fixed \(p \in [1,\infty]\) specified in advance.
As we discuss later, this subsumes the recently studied budget allocation problem for a single LLM-as-a-judge in \cite{saha2026budget} that only considered $J=1$ and $p = \infty$.
Furthermore, while our setting is motivated by the LLM-as-a-judge framework~\citep{saha2026budget}, it addresses a broader class of adaptive resource allocation problems across heterogeneous information sources under budget constraints.

\textbf{Contributions.} We provide a complete theoretical and algorithmic solution to our problem:

\emph{(a) Optimal allocation and an adaptive algorithm.} We first analyze the inverse-variance weighted estimator (IVWE) and to derive the oracle allocation that minimizes its error rate (\textbf{\Cref{sec:known}}). Since this allocation depends on the unknown variances, we then address the practical unknown-variance setting by proposing \estivwe{} (\Cref{alg:estimate-then-ivwe}), an adaptive algorithm that constructs and leverages \emph{optimistically biased} variance estimates to stabilize the empirical allocation. We prove that \estivwe{} matches the oracle IVWE rate up to lower-order terms in the budget (\textbf{\Cref{sec:unknown}}).

\emph{(b) Local minimax lower bounds.} We derive \emph{local} minimax lower bounds matching the leading allocation-dependent rate of \estivwe{}, hence establishing its asymptotic instance-optimality up to logarithmic factors (\textbf{\Cref{sec:lower-bound}}). A key technical insight is that our Fano-type change-of-measure and Assouad-type arguments capture different allocation geometries. The former characterizes the $\log\frac{1}{\delta}$ dependence, while the latter, based on local, variance-preserving perturbations, yields the sharp oracle allocation-dependent rate both in expectation and with constant probability.

\emph{(c) Empirical validation.} We numerically validate \estivwe{} on both synthetic data and the real-world \emph{HelpSteer2} dataset~\citep{NEURIPS2024_02fd91a3} processed by \citet{saha2026budget}. Our results demonstrate that \estivwe{} consistently and significantly outperforms na\"{i}ve uniform allocation (\textbf{\Cref{sec:experiments}}).

\paragraph{Related Work.}

While we contextualize our contributions with specific literature throughout the text, we briefly situate our structural formulation within broader statistical frameworks.
In classical statistics, estimating a shared mean from multiple heteroskedastic sources is known as the \emph{common mean problem}~\citep{cochran1937problems,graybill-deal,norwood-hinkelmann}.
This line of work explores how to optimally aggregate samples, though traditionally without budget constraints.
More recently, theoretical computer science has studied a variant featuring a single sample per population -- referred to as \emph{entangled mean estimation}~\citep{chierichetti2014commonmean,liang2020entangled,diakonikolas2025entangled}—largely motivated by crowdsourcing applications.
Despite this rich literature on heteroskedastic aggregation, incorporating a fixed budget constraint across sources with varying costs has, to the best of our knowledge, remained unexplored prior to our work.

%% file: 002ProblemSetting.tex
\section{Problem Setting}
\label{sec:setting}
There are \(K\) queries, e.g., prompt--answer pairs, to be evaluated by \(J\) LLM judges. Each query \(k\in[K]\) has an unknown ground-truth score \(s_k \in \mathbb{R}\). Whenever the learner assigns query \(k\) to judge \(j\), she incurs a cost \(c_j \in \mathbb{R}_{>0}\) and observes a noisy score \(s_{k,j}\in[0,R]\), independently of past observations.
We assume that $\mathbb{E}[s_{k,j}] = s_k$ and $\sigma_{k,j}^2 \coloneq \operatorname{Var}(s_{k,j}) > 0$. The variances \(\sigma_{k,j}^2\) are unknown, while the score upper bound \(R\) is known. The learner operates online: at each step, she selects a query--judge pair based on all observations collected so far. Given a total budget \(B\in\mathbb{R}_{>0}\) and a known cost vector \(\vc=(c_j)_{j\in[J]}\in\mathbb{R}_{>0}^J\), her goal is to provide an estimator of the score vector $\vs=(s_k)_{k\in[K]}\in\mathbb{R}^K$ as accurate as possible, i.e., minimizing the \(\ell_p\)-error, for a fixed \(p \in [1,\infty]\) specified in advance.
With this, each problem instance is characterized by a triple $(\vs, \bm\sigma, \vc) \in \sR^K \times \sR_{> 0}^{K\times J} \times \sR^J_{> 0}$, where $\bm\sigma \coloneq (\sigma_{k,j})_{k\in [K],j\in [J]}$ is the variance profile.

To assess and compare the performance, we introduce the following notion of learnability:
\begin{definition}[$(B, \delta, \ell_p)$-Budget Efficient Algorithm]
\label{def:budget-efficient}
    Let $p \in [1, \infty]$, $B \in \sR_{> 0}$ be a total budget, and $\delta \in (0, 1)$ be a confidence level.
    An algorithm $\gA$ with score estimate $\hat{\vs} \in \sR^K$ is said to be \textbf{$(B, \delta, \ell_p)$-budget efficient} at $(\vs, \bm\sigma, \vc)$ with error rate $\varepsilon_p > 0$ if:
$$
\sP\left( \bignorm{\hat{\vs} - \vs}_p \le \varepsilon_p \right) \ge 1 - \delta\quad \hbox{ and } \quad \sum_{j \in [J]} c_j \sum_{k \in [K]} N_{k,j} \le B, \hbox{ a.s.}
$$
where $N_{k,j}$ denotes the (random) number of times $\gA$ queries pair $(k, j) \in [K] \times [J]$.
        
\end{definition}

\paragraph{Discussion.}
We view this model as a first step toward a statistical foundation for budget-constrained estimation from multiple LLM judges.
It captures two key characteristics: heterogeneous per-judge costs and query-dependent reliability through the variance profile.
The $K$ queries are fixed and given; the learner chooses which query--judge pairs to evaluate and must estimate $\vs$.

The shared-mean assumption idealizes judges as unbiased estimators of a common target score, which may serve as an approximation after calibration.
Likewise, the deterministic cost $c_j$ can represent a conservative per-judge budget charge under fixed input and output token limits, even when the realized cost is smaller.
These assumptions simplify practical evaluation pipelines, where residual bias, dependence, and query-dependent stochastic costs can be important.
We next discuss the boundedness assumption and the role of heteroskedasticity in our model; further discussion of its limitations and possible extensions is provided in \Cref{app:limitation}.

The boundedness assumption $s_{k,j}\in[0,R]$ is quite realistic as the LLM-judges are typically required to assign scores in fixed rating scales (e.g., binary, $\{0,1,2,3,4\}$, or $[0,100]$) based on scoring rubrics to ensure consistent evaluation~\citep{gu2025surveyllmasajudge}.
As we will see later, such a boundedness assumption necessitates fundamentally different analysis techniques and algorithmic design compared to \citet{saha2026budget}, which assumes Gaussianity for algorithm design.
Additionally, we emphasize that nonnegativity is not necessary; all statements hold for any interval of length $R$.

We now comment on the heteroskedasticity of the problem setting.
By Popoviciu's inequality on variances~\citep{popoviciu}, a trivial bound of $\sigma_{k,j}^2 \leq R^2 / 4$ holds, but we note that it can be significantly loose; for instance, if the pair $(k, j)$ is ``easy-to-evaluate'' (the judge has sufficient expertise to ``confidently'' evaluate the query), $\sigma_{k,j}^2 \ll R^2 / 4$.
Accounting for such variation in the variances is a central challenge in the heteroskedastic learning literature~\citep{chaudhuri2017heteroskedastic,kirschner2018heteroskedastic,fontaine2021heteroskedastic}.

%% file: 003KnownUpper.tex
\section{Oracle Allocation for Inverse-Variance Weighted Estimation}
\label{sec:known}

In this section, we first provide high-probability guarantees for the $\ell_p$-error of the celebrated Inverse-Variance Weighted Estimator (IVWE) \citep{cochran1937problems, cochran1954combination} under any fixed allocation, and then derive the allocation minimizing this error. This allocation is referred to as the {\it Oracle} allocation, because it requires the knowledge of the variances $\sigma_{k,j}^2$.  

Consider a fixed allocation and let $N_{k,j}$ be the number sampled scores obtained for query $k$ from judge $j$. Alternatively, the allocation can be represented by the cost-weighted proportions of the budget allocated to the $(k,j)$-pairs by $\bm\omega \coloneq ( \omega_{k,j} \triangleq \frac{c_j N_{k,j}}{B})_{k,j}$. The budget constraint $B$ is satisfied if and only if $\bm\omega\in \Lambda^{KJ}:=\{\bm\lambda\ge 0: \sum_{k,j}\lambda_{k,j}=1\}$. Let $\{ s_{k,j}^{(i)} \}_{i=1}^{N_{k,j}}$ denote sampled scores of query $k$ from the judge $j$. The \ivwe{} is defined by: 
\begin{equation}
    \hat{s}_k \coloneq \left( \sum_{j=1}^J \frac{N_{k,j}}{\sigma_{k,j}^2} \right)^{-1} \sum_{j=1}^J \frac{N_{k,j} \hat{s}_{k,j}}{\sigma_{k,j}^2}, \quad
    \hat{s}_{k,j} \coloneq \frac{1}{N_{k,j}} \sum_{i=1}^{N_{k,j}} {s}_{k,j}^{(i)}, \ \forall j \in [J],
\end{equation}
where by convention, $\hat{s}_{k,j} = 0$ if $N_{k,j} = 0$. The next proposition provides error-rate guarantees for \ivwe{}; its full statement and proof are given in \Cref{app:upper-bound-allocation}.

\begin{proposition}
    For any instance $(\vs, \bm\sigma, \vc)$ and $p\in[1,\infty]$, \ivwe{} is $(B, \delta, \ell_p)$-budget efficient with the following error rate $\varepsilon_p$:
    \begin{equation}
        \varepsilon_p \approx_B \sqrt{\frac{2 \gA_p(\bm\omega; \bm\sigma, \vc) \log\frac{2K}{\delta}}{B}}, \qquad
        \gA_p(\bm\omega; \bm\sigma, \vc) \coloneq \left( \sum\nolimits_{k \in [K]} \left( \sum\nolimits_{j \in [J]} \frac{\omega_{k,j}}{c_j \sigma_{k,j}^2} \right)^{-\frac{p}{2}} \right)^{\frac{2}{p}},
    \end{equation}
    where $\approx_B$ omits lower-order terms in $B$. 
\end{proposition}
In the above result, \(\gA_p(\cdot)\) is derived from the variance of \ivwe{}. We refer to it as the \emph{allocation objective}, since minimizing this quantity yields the allocation that minimizes the estimator's error rate. The following theorem, proved in \Cref{app:optimal-allocation}, characterizes this optimal allocation.
\begin{theorem}
\label{thm:optimal-allocation}
    For each query $k \in [K]$, let $j^*(k) \coloneq \argmin_{j \in [J]} c_j \sigma_{k,j}^2$ be the optimal judge. For $p \in [1, \infty]$, consider the following optimal allocation problem: $\gA_p^*(\bm\sigma, \vc) \coloneq \min_{\bm\omega\in \Lambda^{KJ}} \gA_p(\bm\omega; \bm\sigma, \vc)$. 
    The value of this problem and the corresponding optimal allocation $\bm\omega^* = (\omega_{k,j}^*)_{k,j}$ are given by
    \begin{equation}
        \gA_p^*(\bm\sigma, \vc) = \left( \sum_{k \in [K]} \left( c_{j^*(k)} \sigma_{k,j^*(k)}^2 \right)^{\frac{p}{p+2}} \right)^{\frac{p + 2}{p}}, \
        \omega_{k,j}^* = \frac{{\indicator[j = j^*(k)]} \left( c_{j^*(k)} \sigma_{k,j^*(k)}^2 \right)^{\frac{p}{p+2}}}{\sum_{k' \in [K]} \left( c_{j^*(k')} \sigma_{k',j^*(k')}^2 \right)^{\frac{p}{p+2}}}.
    \end{equation}
\end{theorem}

A few remarks are in order. First, note that when \(J=1\) (a single judge), \(c_1=1\) (unit cost), and \(p=\infty\), the above result reduces exactly to \citet[Theorem~3]{saha2026budget}. Our setting therefore generalizes the framework studied therein. An interesting consequence of the theorem is that, regardless of \(p\), an optimal allocation can be chosen to be \emph{sparse}: for each query \(k\), the budget is assigned to a single optimal judge \(j^*(k)\).
This follows from the direct comparison in \Cref{app:optimal-allocation}: for a fixed budget assigned to query \(k\), its contribution to the allocation objective is minimized by assigning that budget to a judge minimizing \(c_j\sigma_{k,j}^2\). 
Intuitively, once the budget allocated to a query is fixed, the most accurate estimate \(\hat{s}_k\) is obtained by assigning all of that budget to the judge with the best cost--variance trade-off. 
Allocating any positive budget to a judge \(j\) with \(c_j\sigma_{k,j}^2>c_{j^*(k)}\sigma_{k,j^*(k)}^2\) can only worsen, or at best leave unchanged, the resulting estimate.
Thus, while the choice of \(p\) changes how budget is distributed across queries, a single optimal judge per query always suffices. If several judges attain the same minimum, splitting the query's budget among them is also optimal.

At this stage, the optimality statement concerns the minimization of our upper bound. In \Cref{sec:lower-bound}, we show that this bound is in fact minimax optimal.

%% file: 004UnknownUpper.tex
\section{\estivwe{}: An Adaptive Algorithm and its Error Rates}
\label{sec:unknown}

In this section, we consider unknown variances, and propose an adaptive algorithm \estivwe{}. It proceeds in two phases: the first phase estimates the variances, while the second allocates the remaining budget according to the empirical optimal allocation induced by these estimates. 
The pseudo code of \estivwe{} is given in \Cref{alg:estimate-then-ivwe}.
Before providing error-rate guarantees for \estivwe{}, we first motivate its design and discuss the main challenge in its analysis.

\begin{algorithm2e}[!ht]
    \SetKwInput{Input}{Input}
    \SetKwComment{Comment}{$\triangleright$\ }{}

    \Input{Total budget $B > 0$, Cost vector $\vc \in \mathbb{R}^J_+$, Error norm $p\in[1,\infty]$, Confidence level $\delta\in(0,1)$, Number of forced exploration per pair $N_0 \geq 2$}

    \BlankLine
    \tcp{Phase I: Forced Exploration}
    \For{each $(k, j) \in [K] \times [J]$}{
        Pull arm $(k,j)$ exactly $N_0$ times and observe noisy evaluation scores $\left\{ s_{k,j}^{(n)} \right\}_{n \in [N_0]}$\;

        Compute the empirical variance estimator
        \begin{eqnarray}
            \hat{\sigma}_{k,j}^2 &:=& \frac{1}{2 N_0 (N_0 - 1)} \sum_{1 \leq n \neq n' \leq N_0} \left( s_{k,j}^{(n)} - s_{k,j}^{(n')} \right)^2;
        \end{eqnarray}
    }

    \tcp{Phase II: IVWE}
    Compute the optimal allocation $\left( \widehat{\omega}_{k,j}^* \right)_{k,j}$ as in \Cref{thm:optimal-allocation} using the remaining budget $B' \triangleq B - N_0 K \sum_{j \in [J]} c_j - K\max_{j\in[J]} c_j$, based on $\left\{ \left( \hat{\sigma}_{k,j} + \tau \right)^2 \right\}_{k,j}$ with $\tau$ as in Eqn.~\eqref{eq:optimistic-variance}\;
    Set $\widehat N_{k,j}^* = \left\lceil {B'\widehat\omega_{k,j}^*}/{c_j} \right\rceil$, and query each pair $(k, j)$ exactly $\widehat N_{k,j}^*$ times\;
    
    \Return \textbf{IVWE} with $\left( \widehat{N}_{k,j}^* \right)_{k,j}$ as allocation and $(\tau + \hat{\sigma}_{k,j})^2$ as variance proxies\;
    \caption{Estimate-then-IVWE (Est-IVWE)}\label{alg:estimate-then-ivwe}
\end{algorithm2e}

{\bf Optimistically biased variance estimates.} A key challenge in analyzing the performance of \Cref{alg:estimate-then-ivwe} is quantifying \emph{how accurately the variances are estimated from Phase I}.
Ideally, we seek a multiplicative guarantee of the form $(1 - \alpha) \sigma_{k,j}^2 \leq \hat{\sigma}_{k,j}^2 \leq (1 + \alpha) \sigma_{k,j}^2$ for some $\alpha \in (0,1)$.
The common factors $(1\pm\alpha)$ then give a direct comparison between empirical and oracle allocations (\Cref{thm:optimal-allocation}). Under Gaussian scores, \citet[Lemma~6]{saha2026budget} use such multiplicative bounds, obtained from $\chi^2$ concentration \citep[Eqn.~(4.4) \& (4.5)]{laurent-massart}. These bounds do not apply to general bounded scores.
Instead, we utilize the empirical Bernstein inequality~\citep[Theorem 10]{maurer-pontil}, which provides an additive guarantee of $| \sigma_{k,j} - \hat{\sigma}_{k,j}| \lesssim \sqrt{\log(1/\delta)/N_0}$ ($N_0$ is the length of the first phase).
Now, if we were to use $\hat{\sigma}_{k,j}$ directly and require it to satisfy analogous multiplicative guarantee, we would require $N_0 \gtrsim \frac{1}{\sigma_{k,j}^2} \log\frac{1}{\delta}$.
This is problematic because we require the knowledge of $\sigma_{k,j}$ to set $N_0$, and for \emph{easy} pairs in which $\sigma_{k,j}$ is near-zero, the requirement on the number of exploratory samples $N_0$ becomes prohibitively large.

To overcome these limitations, we use an \emph{optimistic variance estimator}, 
\begin{equation}\label{eq:optimistic-variance}
    \bar{\sigma}_{k,j} = \hat{\sigma}_{k,j} + \tau, \qquad\text{where}\quad  \tau:=R\sqrt{\frac{2}{N_0-1}\log \frac{4KJ}{\delta}}.
\end{equation}
Here, the bias $\tau$ comes from the empirical Bernstein inequality so that \(\sigma_{k,j} \leq \bar{\sigma}_{k,j} \leq \sigma_{k,j} + 2\tau\) with high probability; see \Cref{lem:easy-hard}. This optimistic correction stabilizes the empirical allocation, especially for query--judge pairs with near-zero variance, whose estimates are otherwise highly unstable. It also ensures that such pairs receive a sufficient amount of budget for reliable estimation. By tuning \(\tau\) and \(N_0\) appropriately as functions of the total budget \(B\), the extra cost induced by this stabilization remains lower order relative to the leading error term.

The following theorem formalizes the resulting error guarantees.

\begin{theorem}[Error Rates of \estivwe{}]
\label{thm:est-ivwe}
    Let $p \in [1, \infty]$, $B \in \sR_{> 0}$, and $\delta \in (0, 1)$.
    Let us set
    \begin{eqnarray}
        N_0 = 
        \begin{cases}
            \max \left\{ 2, \left\lceil (2 B)^{1/3} \left( R^2 \log\frac{4 K J}{\delta} \right)^{2/3} \right\rceil\right\} & \hbox{ if }p \geq 2,\\
            \\
            \max \left\{ 2, \left\lceil 2^{\frac{p}{2p+2}} R^{\frac{3p+2}{2p+2}} \left( \log\frac{4KJ}{\delta} \right)^{\frac{3p+2}{4p+4}} B^{\frac{p+2}{4p+4}} \right\rceil \right\} & \hbox{ if }1 \leq p < 2.
        \end{cases}
    \end{eqnarray}
    Further suppose that $B \geq 2 \left( N_0  K \sum_{j\in[J]} c_j + K \max_{j\in[J]} c_j \right)$.
    Then, \estivwe{} is $(B, \delta, \ell_p)$-budget efficient at any $(\vs, \bm\sigma, \vc)$ with the following error rate $\varepsilon_p$:
    \begin{equation}
        \varepsilon_p = \sqrt{\frac{2 \gA_p^*(\bm\sigma, \vc) \log\frac{4K}{\delta}}{B}} +
        \begin{cases}
            \gO\left( \frac{R^{\frac{1}{3}} \gA_p^*(\bm\sigma, \vc) \left( \log\frac{4KJ}{\delta} \right)^{\frac{2}{3}}}{B^{\frac{2}{3}}} \right) & \hbox{ if }p \geq 2, \\
            \gO\left( \frac{R^{\frac{p}{2p+2}} \gA_p^*(\bm\sigma, \vc) \left( \log\frac{4KJ}{\delta} \right)^{\frac{3p+2}{4p+4}}}{B^{\frac{3p+2}{4p+4}}} \right) & \hbox{ if } 1 \leq p < 2.
        \end{cases}
    \end{equation}
\end{theorem}
\begin{proof}[Proof sketch]
    We divide the $(k,j)$-pairs into two cases using the bias $\tau$ as an implicit threshold.
    For ``easy-to-evaluate" pairs with $\sigma_{k,j} \leq \tau$, the bias stabilizes the ratio $(\hat{\sigma}_{k,j}+\tau)/(\sigma_{k,j}+\tau)$, avoiding the numerical instability of $\hat{\sigma}_{k,j}/ \sigma_{k,j}$ when $\sigma_{k,j} \approx 0$.
    For ``hard-to-evaluate" pairs with $\sigma_{k,j} > \tau$, the multiplicative guarantee between $\hat{\sigma}_{k,j}+\tau$ and $\sigma_{k,j}+\tau$ follows, as we choose $\tau$ to be the concentration radius of \citep[Theorem 10]{maurer-pontil}.
    Since the selection rule $\hat{j}^*(k)=\argmin_{j} c_j \overline{\sigma}_{k,j}^2$ in Phase II implies $c_{\hat{j}^*(k)} \overline{\sigma}_{k,\hat{j}^*(k)}^2 \leq c_{j^*(k)}\overline{\sigma}_{k,j^*(k)}^2$, we can upper bound the error $\varepsilon_p$ of the empirical allocation; and then we decouple the effect of additive $\tau$ using Minkowski's inequality.
    We then conclude by separating the impact of initial exploration $N_0$ from the total budget $B$, to recover the \emph{same} leading term as the error rate of \ivwe{}.
    The full proof is deferred to \Cref{app:upper-bound-unknown}.
\end{proof}

We conclude the section with two remarks on \estivwe{} and its performance.

{\bf (1) Oracle performance without known variances.} Remarkably, \Cref{thm:est-ivwe} shows that, for every fixed instance and all
$p\in[1,\infty]$, the adaptive algorithm \estivwe{} (\Cref{alg:estimate-then-ivwe}) matches the leading $B^{-1/2}$ term of the oracle strategy that knows the variances in advance (\Cref{thm:optimal-allocation}). 
Thus, in the sufficiently large budget $B$, there is no price for not knowing the variances: a simple two-stage procedure recovers the oracle rate.

{\bf (2) Relations to Graybill-Deal Estimator.} Estimating \(s_k\) from multiple sources is an instance of the classical \emph{common mean problem} \citep{graybill-deal,pal2007revisit}. Our two-stage estimator is closely related to the \emph{Graybill--Deal} estimator, which uses all available samples to estimate both means and variances. Reusing samples induces correlations between these estimates, complicating sharp performance analysis. Even for Gaussian populations, tight variance characterizations are technically involved \citep{nair1980,voinov1984}, and admissibility is known only for certain classes of problems \citep{ghosh-sinha,sinha-mouqadem,pal1997admissible}. We avoid this dependence through \emph{sample splitting}: Phase-I samples estimate the variances \(\hat{\sigma}_{k,j}^2\), while fresh Phase-II samples yield the final \ivwe{} estimates \(\hat{s}_k\). This separates the randomness used in the two phases and ensures that, conditioned on the estimated variances, \ivwe{} remains unbiased,
\(
\mathbb{E}\!\left[\hat{s}_k \,\middle|\, \{\hat{\sigma}_{k,j}\}_{k,j}\right] = s_k .
\)

%% file: 005Lower.tex
\section{Local Minimax Lower Bounds}
\label{sec:lower-bound}

We now establish fundamental local minimax lower bounds for our budgeted heteroskedastic multi-judge estimation problem. The term \emph{local} means that the lower bound holds in a small neighborhood of any fixed problem instance, in the same spirit as local minimax results for logistic bandits~\citep[Theorem~2]{abeille2021logistic}, online LQR~\citep[Theorem~1]{simchowitz2020lqr}, and generalized trace regression~\citep[Theorem~4.1]{lee2026gl-lowpopart}. These results show that the instance-dependent allocation quantity \(\gA_p^*(\bm\sigma, \vc)\), initially derived from upper bounds for \ivwe{}, is in fact the correct information-theoretic complexity measure.
A key technical insight in our lower-bound analysis is that our two arguments capture complementary aspects of the problem: a Fano-type change-of-measure argument characterizes the $\log\frac{1}{\delta}$ dependence, while an Assouad-type construction recovers the oracle allocation-dependent complexity.

\subsection{A Fano-Type High-Probability Lower Bound}

We begin by stating a high-probability minimax lower bound that characterizes the $\log\frac{1}{\delta}$ dependence through a Fano-type change-of-measure argument:

\begin{theorem}[High Probability Lower Bound]
\label{thm:lower-bound-whp}
    Let $p \in [1, \infty]$, $B > 0$, $\delta \in (0, 1/16]$, and $(\vs, \bm\sigma, \vc)$ be a given problem instance.
    Let $R_k \coloneq \min\{s_k, R - s_k\} > 0$, $R_{\min} = \min_k R_k$, and suppose that $\max_j \sigma_{k,j}^2 \leq \frac{R_k^2}{2}$.
    Then, any algorithm that is $(B, \delta, \ell_p)$-budget efficient with an error rate of $\varepsilon_p \in \left( 0, \frac{R_{\min}}{8} \right]$ at any $(\vs', \bm\sigma, \vc)$ with $\bignorm{\vs' - \vs}_p \leq 2\varepsilon_p$\footnote{This in the same spirit as locally stable algorithms used for proving local \& high-probability lower bounds~\citep{jedra2023identification,yun2019optimal}.} must satisfy the following: for an absolute constant $C_1 > 0$,
    \begin{equation}
        \varepsilon_p \geq \sqrt{\frac{\log\frac{1}{\delta}}{C_1 B} {\gA_{\frac{2p}{2-p}}^*(\bm\sigma, \vc)}} \ \ \ \text{for $1 \leq p < 2$}, \quad 
        \varepsilon_p \geq \sqrt{\frac{ \log\frac{1}{\delta}}{C_1 B} {\gA^*_{\infty}(\bm\sigma, \vc)}} \ \ \ \text{for $p \geq 2$}.
    \end{equation}
\end{theorem}
\begin{proof}[Proof sketch]
    We suppose that the learner has knowledge of the true $\bm\sigma$.
    The proof uses a Fano-type change-of-measure argument, following the hypothesis testing framework~\citep{jedra2023identification} based on the binary KL data processing inequality~\citep[Lemma 1]{garivier2019lower}.
    The main technical novelty is the construction of the prior over $\vs$ to retain the instance-specific nature.
    Specifically, when constructing alternate models, one needs to ensure that $(i)$ its mean stays close to $\vs$, $(ii)$ its variance remains the \emph{same} as $\bm\sigma$, and $(iii)$ it is bounded in $[0, R]$.
    This is done by considering a family of \emph{beta distributions} ({\Cref{prop:beta-construction}}).
    The proof ``bifurcates'' at $p = 2$: the function $z \mapsto z^{2/p}$ is convex for $p \leq 2$ and concave for $p \geq 2$, and is linear at $p=2$.
    Consequently, minimizing the quadratic KL bound at a prescribed $\ell_p$ distance admits a sparse solution for $p \geq 2$ and a dense solution for $1 \leq p < 2$ ({\Cref{lem:infimum-hard}}).
    We provide the full proof in \Cref{app:whp-lower-bound}.
\end{proof}

The high-probability bound matches the upper bound only in the case \(p=\infty\), up to a factor of $\sqrt{\log K}$. For \(p<\infty\), however, it does not recover the allocation quantity \(\gA_p^*(\bm\sigma, \vc)\) appearing in the upper bound of \ivwe{}.
Indeed, the lower bound only attains $\gA_{\frac{2p}{2-p}}^*(\bm\sigma, \vc)$ for $1 \leq p < 2$ and $\gA_\infty^*(\bm\sigma, \vc)$ for $p \geq 2$.
This is \emph{not} the oracle allocation dependence for finite $p$.
Its explicit $\log\frac{1}{\delta}$ dependence nevertheless remains informative, as we discuss in \Cref{subsec:geometries}.

\subsection{An Assouad-Type In-Expectation Lower Bound}

We now recover the allocation dependence missing from the previous Fano-type construction.
Using an Assouad-type argument, we obtain a local in-expectation lower bound that matches \(\gA_p^*(\bm\sigma, \vc)\) for all \(p\in[1,\infty]\).
The same construction also yields a constant-probability lower bound, as noted below.

\begin{theorem}[In-Expectation Lower Bound]
\label{thm:lower-bound-exp}
    Let ${p\in[1,\infty]}$, $B > 0$ and $(\vs^\star, \bm\sigma, \vc)$ be any given problem instance with $\vs^\star = (s_k^\star)_k$.
    Suppose that $R_k \coloneq \min\{s_k^\star, R - s_k^\star\}{>0}$ and {$0<\sigma_{k,j}^2\leq \tfrac{R_k^2}{2}$ for all $(k,j)$.}
    Define the ``radius'' $\xi(B) \coloneq \sqrt{\frac{\gA_p^*(\bm\sigma, \vc)}{16B}}$ and the ``threshold''\footnote{For $p=\infty$, the expressions are interpreted by continuity, giving $B_0=\gA_\infty^*(\bm\sigma,\vc)\max_kR_k^{-2}$.} as $B_0 \coloneq \left( \sum_{k \in [K]} \left( c_{j^*(k)} \sigma_{k,j^*(k)}^2 \right)^{\frac{p}{p+2}} \right) \max_{k \in [K]} \frac{\left( c_{j^*(k)} \sigma_{k,j^*(k)}^2 \right)^{\frac{2}{p+2}}}{{R_k^2}}$.
    Then, there exists a universal constant $C_3 > 0$ such that for $B \geq B_0$,
    \begin{equation}
        \inf_{\hat{\vs}} \sup_{\vs : \bignorm{\vs - \vs^\star}_p \leq \xi(B)} \E^{\hat{\vs}}_{\vs}\left[ \bignorm{\hat{\vs} - \vs}_p \right] \geq
        C_3 \sqrt{\frac{\gA_p^*(\bm\sigma, \vc)}{B}},
    \end{equation}
    and furthermore for $1 \leq p < \infty$ and $B \geq p B_0$,
    \begin{equation}
        \inf_{\hat{\vs}} \sup_{\vs : \bignorm{\vs - \vs^\star}_p \leq \sqrt{p}\xi(B)} \E^{\hat{\vs}}_{\vs}\left[ \bignorm{\hat{\vs} - \vs}_p^p \right]^{\frac{1}{p}} \geq
        4^{1 - \frac{1}{p}} C_3 \sqrt{\frac{p \gA_p^*(\bm\sigma, \vc)}{B}}.
    \end{equation}
\end{theorem}
\begin{proof}[Proof sketch]
    The full proof is deferred to \Cref{app:lower-bound-exp}.
    We follow the standard nonparametric lower bound framework based on the seminal {{Assouad's Lemma}}~\citep{assouad1983,yu1997lecam} combined with the Bretagnolle-Huber inequality~\citep{bretagnolle-huber}.
    We first construct a hypercube of $2^K$ hard instances, and each vertex corresponding to the same beta distribution construction as in the proof of \Cref{thm:lower-bound-whp} (\Cref{prop:beta-construction}).
    
    For $p > 1$, $\E[ \bignorm{\hat{\vs} - \vs}_p ]$ is \emph{not} decomposable coordinate-wise, hindering the direct application of Assouad's lemma.
    {For finite $p$, we bypass this using a secant bounding argument (leveraging the concavity of $x \mapsto x^{1/p}$), decoupling the expectation to establish a lower bound proportional to $\bignorm{\bm\Delta}_p$. For $p=\infty$, we use equal perturbations and lower bound the probability of at least one incorrect sign.}
    {For finite $p$, the $p$-th moment is coordinate-wise decomposable, so we can directly apply Assouad's Lemma.}
    Finally, by tuning the adversarial shifts $\bm\Delta$ under KL-divergence constraints and minimizing the bound over all valid algorithm budget allocations, we recover the optimal rate $\gA_p^*(\bm\sigma, \vc)$.
\end{proof}

\paragraph{Beyond In-Expectation.}
For finite $p$, the same Assouad construction also gives a constant-probability lower bound via the reverse Markov inequality; see \citet{ma2024minimax} for a general conversion from local minimax risk to minimax quantiles.
(All proofs here are deferred to \Cref{app:lower-bound-exp}.)
Under the assumptions of \Cref{thm:lower-bound-exp}, with the same $C_3$, for every $\delta\in(0,4C_3)$,
\begin{equation}
    \inf_{\hat{\vs}} \sup_{\vs : \bignorm{\vs - \vs^\star}_p \leq \xi(B)}  \sP_{\vs}^{\hat{\vs}}\!\left(\bignorm{\hat{\vs}-\vs}_p\geq\xi(B)\left(\frac{4C_3-\delta}{1-\delta}\right)^{1/p}\right)\geq\delta.
\end{equation}
Taking $\delta=2C_3$ recovers the allocation-dependent rate with constant probability, but the threshold stays bounded as $\delta\downarrow0$: this conversion does not recover the $\log\frac{1}{\delta}$ dependence of \Cref{thm:lower-bound-whp}.
For $p=\infty$, the construction directly gives $\sP_{\vs}^{\hat{\vs}}(\bignorm{\hat{\vs}-\vs}_\infty\geq\xi(B))\geq4C_3$ at some vertex in the same local neighborhood. Both arguments are given in \Cref{app:lower-bound-exp}.

We make two further remarks concerning our lower bounds:

{\bf (1) Geometry of Constrained Lower Bounds.} Optimizing over algorithms subject to a budget constraint is reminiscent of techniques used in statistical estimation under communication constraints \citep{zhu2014quantized,chen2021vantrees,chen2024vantrees}. In such settings, restricting the algorithm class naturally gives rise to an optimal allocation problem. Our analysis reveals a related phenomenon, with different allocation quantities governing the bounds obtained from our Fano/Assouad-type arguments. We discuss their complementary roles in the next subsection.

{\bf (2) In-Expectation Upper Bounds.}
The constant-probability consequence above permits comparison with the high-probability upper bounds in \Cref{thm:optimal-allocation,thm:est-ivwe} at sufficiently small fixed $\delta$.
For known variances and fixed $1\leq p<\infty$, a standard moment argument based on Bernstein's inequality \citep{bernstein} yields an oracle expected-error bound with leading rate $\sqrt{\gA_p^*(\bm\sigma,\vc)/B}$, without logarithmic factors and up to constants depending on $p$.
In particular, at the oracle allocation, independence gives $\E[\bignorm{\hat{\vs}-\vs}_2^2]=\gA_2^*(\bm\sigma,\vc)/B$; Gaussian scores further permit exact finite-$p$ moment calculations (\Cref{thm:in-exp-upper-bound-known}).
For $p=\infty$, $\log K$ dependencies may remain.

\subsection{Different Allocation Geometries for Assouad and Fano}
\label{subsec:geometries}

Our two lower-bound arguments capture complementary aspects of the estimation problem. The Fano-type change-of-measure argument in \Cref{thm:lower-bound-whp} yields an explicit $\log\frac{1}{\delta}$ dependence, but its allocation quantity differs from $\gA_p^*$.
The Assouad-type construction in \Cref{thm:lower-bound-exp} recovers $\gA_p^*$ both in expectation and with constant probability, but lacks the $\log\frac{1}{\delta}$ factor.

The difference comes from how the two constructions aggregate errors across queries. Our Fano-type change-of-measure argument searches for an alternative with small KL divergence at a prescribed $\ell_p$ distance. The resulting quadratic optimization admits a dense solution for $1\leq p<2$ and a sparse solution for $p\geq2$ (\Cref{lem:infimum-hard}), leading to $\gA_{2p/(2-p)}^*$ and $\gA_\infty^*$, respectively. The Assouad-type argument instead \emph{first} lower bounds the weighted sign errors coordinate-wise, \emph{then} aggregates them, and \emph{finally} optimizes over the algorithm's budget allocation. Its hypercube of perturbations retains the simultaneous contribution of all queries. Both constructions preserve the same variance profile; the distinction concerns the aggregation of coordinate errors.

Then one may wonder if the correct lower bound should be of the form $\sqrt{\gA_p^*\log(1/\delta)/B}$.
We argue that this may not be the case.
Consider the Gaussian setting with one unit-cost judge, variance $\sigma^2$, and $p=2$.
With $B/K$ samples per query, assumed integer for simplicity, the sample mean satisfies $\sP\left( \bignorm{\hat{\vs}-\vs}_2
    \leq \frac{\sigma}{\sqrt B}\left(K+\sqrt{2K\log\frac1\delta}\right) \right) \geq 1 - \delta$; see Eqn.~\eqref{eq:gaussian-l2-confidence} in \Cref{app:gaussian}.
Here, $\gA_2^*=\sigma^2K^2$ and $\gA_\infty^*=\sigma^2K$: the first term reflects the error accumulated across all queries, while the second reflects the additional cost of demanding higher confidence. Thus, the two terms naturally involve different allocation quantities. In particular, a lower bound of order $\sqrt{\gA_2^*\log(1/\delta)/B}$ would not hold uniformly in $K$ and $\delta$ even in this simple Gaussian example.

A similar decomposition is established by \citet[Proposition 10]{ma2024minimax} for Gaussian mean estimation: the minimax squared-error quantile has a trace term measuring the overall estimation difficulty and an operator-norm term multiplied by $\log\frac1\delta$. Their proof combines an Assouad risk bound converted to a probability statement with a separate two-point confidence argument. This supports viewing our two lower bounds as complementary, but we remark that a sharp characterization jointly in the variance profile, $p$, and $\delta$ remains open in our budgeted setting.

%% file: 006Gaussian.tex
\section{Beyond Boundedness: Gaussian Scores}

We briefly describe how the results presented so far extend to \emph{Gaussian scores}, as considered in \citep{saha2026budget}, namely \(s_{k,j} \sim \mathcal{N}(s_k,\sigma_{k,j}^2)\). In this setting, all results carry over, and in some cases take a sharper form. Technically, the Gaussian case is of independent interest because it allows the use of tools that exploit distribution-specific structure.
For the interest of space, we only provide an overview, highlighting the key technical ingredients. Full details are deferred to \Cref{app:gaussian}.

\paragraph{Upper Bounds.} 
For known variances, we replace Bernstein's inequality with the standard Gaussian tail bound applied separately for each query $k$. This yields a cleaner upper bound, with no lower-order term from Bernstein's inequality. For unknown variances, we can exploit \(\chi^2\)-concentration for the Gaussian sample variance \citep[Eq.~(4.4)]{laurent-massart}. This allows us to use the \ivwe{} directly, without introducing the additional optimistic bias \(\tau\).

\paragraph{Lower Bounds.} Here, we consider global rather than local lower bounds.
The high-probability lower bound is essentially unchanged, except that we use Gaussian constructions instead of Beta distributions, which affects only constants. The in-expectation lower bounds, for both \(\mathbb{E}[\|\hat{\vs}-\vs\|_p]\) and \(\mathbb{E}[\|\hat{\vs}-\vs\|_p^p]^{1/p}\), are substantially different and technically more interesting, which we sketch here.
We place a Gaussian prior over the instances and use a convexity argument to show that the posterior mean minimizes the expected \(\ell_p\)-risk. Since the posterior distribution is Gaussian with inverse Fisher information as the covariance, rearranging yields the same optimal allocation in the lower bound. 
Remarkably, for \(\mathbb{E}[\|\hat{\vs}-\vs\|_p^p]^{1/p}\), this lower bound is optimal up to universal constants.

%% file: 007Experiments.tex
\section{Experiments}
\label{sec:experiments}

We now numerically evaluate the performance of our algorithms on synthetic and real-world datasets.
The codes are available in our \href{https://github.com/sanghwakim96/llm-judge}{GitHub repository}.

\paragraph{Settings.}
For the synthetic experiment, we set \((K,J)=(1000,10)\) and consider the Beta distribution for each query--judge pair's score distribution ($R = 1$).
The instance, characterized by $(\vs, \bm\sigma, \vc)$, is randomly sampled from an appropriate prior distribution, detailed in \Cref{app:synthetic_data_extended}.

We conduct real-world experiments following \citet{saha2026budget}, in which queries from HelpSteer2~\citep{NEURIPS2024_02fd91a3} are evaluated by three LLM judges --- Llama-3-8B~\citep{Llama}, GPT-4~\citep{GPT4}, and Qwen~\citep{Qwen2.5} --- along four axes: complexity, correctness, helpfulness, and verbosity.
For each axis, we utilize the pre-generated dataset from \citet{saha2026budget}, where each query is repeatedly evaluated with each judge.
During the experiments, we simulate each judge by sampling with replacement from the generated dataset.
We set the judge costs uniformly to \(0.1\) for simplicity, and we consider $p = 2$ for the error metric.
Further details of the preprocessing procedure are provided in \Cref{app:real_data_extended}.

\paragraph{Baselines.}
We compare \estivwe{} (Bounded) (\Cref{alg:estimate-then-ivwe}) and its practical Gaussian variant, \estivwe{} (Gaussian) (\Cref{alg:estimate-then-ivwe-gaussian}), against two baselines: \textsc{Uniform}, which assigns each query--judge pair an equal number of times subject to the budget constraint, and \textsc{Oracle}, which is IVWE with the optimal allocation \(\bm\omega^*\) under known variances (\Cref{thm:optimal-allocation}).
In implementing \estivwe{} (\Cref{alg:estimate-then-ivwe}) and \estivwe{} (Gaussian) (\Cref{alg:estimate-then-ivwe-gaussian}), the exploration budget $N_0$ and the bias $\tau$ are set to the theoretical values as in \Cref{thm:est-ivwe,thm:est-ivwe-Gaussian}, rather than being heuristically tuned.

\paragraph{Results and Discussions.}
The results for synthetic and HelpSteer2 datasets are shown in~\Cref{fig:exp-results}.
First, \estivwe{} and \textsc{Oracle} both outperform the \textsc{Uniform} baseline, confirming the efficacy of adaptive allocation under the heteroskedasticity of the query--judge pairs.
Amongst the adaptive baselines, we note that \estivwe{} rapidly approaches the performance of the \textsc{Oracle} as the budget increases, validating our theoretical results.
Notably, in both settings, \estivwe{} (Gaussian) outperforms or performs on par with \estivwe{} (Bounded), despite the fact that the score distributions are strictly bounded.
This shows that the Gaussian variant may be more effective in practice, although it lacks a rigorous theoretical guarantee.
On HelpSteer2, the error curves appear to plateau at non-zero values, possibly reflecting differences between the mean judge scores and the consensus scores $s_k$.
Additional results from different $p$ and different datasets are provided in \Cref{app:synthetic_data_extended,app:real_data_extended}.
Furthermore, in \Cref{app:cost_analysis}, we provide an additional cost-sensitivity analysis on synthetic data to demonstrate that our algorithm remains robust to different choices of the cost vector.

\begin{figure}[!t]
    \centering
    \begin{subfigure}[b]{0.497\textwidth}
        \centering
        \includegraphics[width=\linewidth]{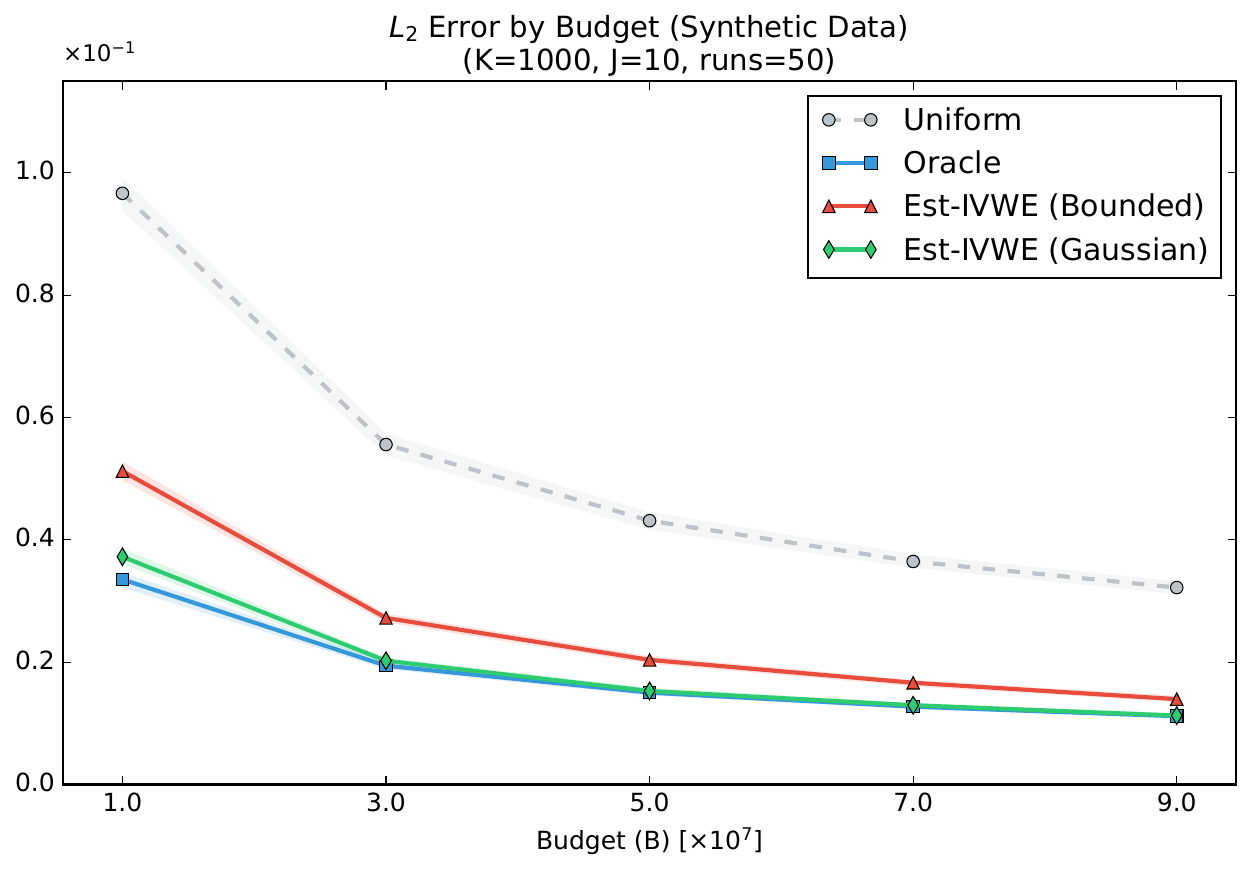}
        \caption{Synthetic data ($K=1000$, $J=10$)}
        \label{fig:synthetic}
    \end{subfigure}
    \hfill
    \begin{subfigure}[b]{0.497\textwidth}
        \centering
        \includegraphics[width=\linewidth]{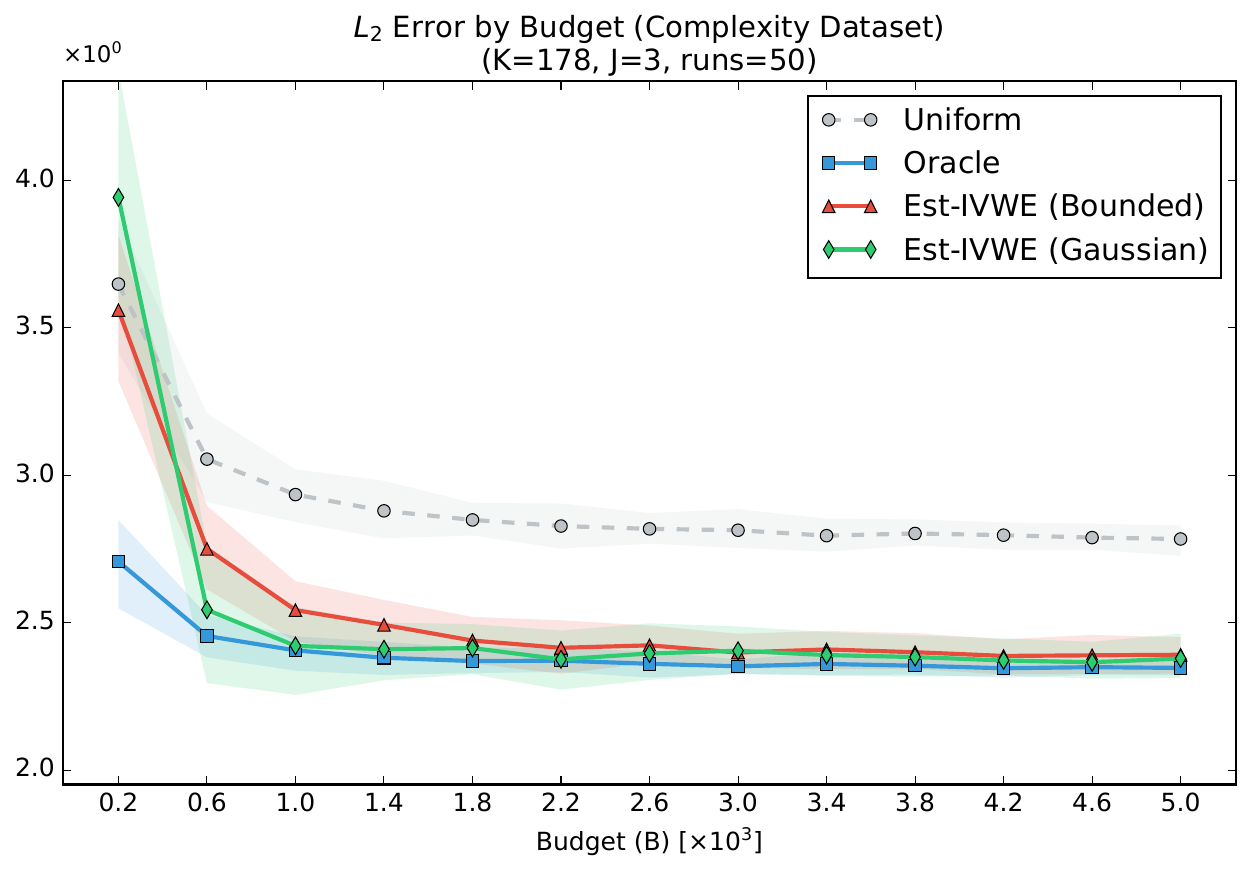}
        \caption{``Complexity" dataset ($K=178, J=3$)}
        \label{fig:real}
    \end{subfigure}
    \caption{Experimental results on both synthetic and real-world datasets. All results are averaged over $50$ independent runs, with the $10\%$ and $90\%$ quantiles (shaded regions)}
    \label{fig:exp-results}
\end{figure}

%% file: 008Conclusion.tex
\section{Conclusion and Future Work}
\label{sec:conclusion}
This paper introduces the problem of \emph{budgeted heteroskedastic multi-judge estimation}, motivated by budget-constrained evaluation via multiple LLM judges of varying costs.
Our main algorithmic template is the inverse-variance weighted estimator (IVWE), whose error rate we analyze to derive the optimal oracle allocation.
Building on this, we develop Est-IVWE for unknown variances via a two-stage approach: Stage~I, the key ingredient, computes \emph{optimistic} variance estimates, and Stage~II applies IVWE with those estimates.
We establish the asymptotic \emph{instance-wise} optimality of our algorithm up to logarithmic factors by proving a local minimax lower bound matching its leading allocation-dependent rate.
Interestingly, our Fano-type change-of-measure and Assouad-type arguments capture different allocation geometries: the former captures the $\log\frac{1}{\delta}$ dependence, while the latter recovers the sharp allocation-dependent quantity; this distinction may be of independent interest.
We numerically validate IVWE and Est-IVWE on synthetic and HelpSteer2 datasets, demonstrating improvements over na\"{i}ve uniform allocation.

Our work's practical applicability is limited by the modeling assumptions introduced in \Cref{sec:setting}.
In \Cref{app:limitation}, we clarify their scope and practical justification and discuss concrete directions for future work.
Two further directions are worth highlighting.
On the theoretical side, as discussed in \Cref{sec:lower-bound}, sharply characterizing the joint dependence on the variance profile, $p$, and $\log\frac{1}{\delta}$ remains open.
Empirically, extending to production-scale LLM-as-a-judge pipelines is an important next step.

%% file: 901Proof-Known.tex
\section{\texorpdfstring{Proofs from \Cref{sec:known} (Known Variances)}{Proofs from Section 3 (Known Variances)}}\label{app:known}

\subsection{Error Rate of \textbf{IVWE}}
\label{app:upper-bound-allocation}
\begin{proposition}
\label{prop:upper-bound-allocation}
    Let $p \in [1, \infty]$, $B \in \sR_{> 0}$ be a total budget, and $\delta \in (0, 1)$ be a target confidence level.
    Also, let $(N_{k,j})_{k,j}$ be any given allocation.
    Then, \textbf{IVWE} is $(B, \delta, \ell_p)$-budget efficient at any $(\vs, \bm\sigma, \vc)$ with the following error rate $\varepsilon_p$:
    \begin{equation}
        \varepsilon_p = \sqrt{\frac{2 \gA_p(\bm\omega; \bm\sigma, \vc) \log\frac{2K}{\delta}}{B}} + \frac{R 
        \gB_p(\bm\omega; \bm\sigma, \vc) \log\frac{2K}{\delta}}{3B},
    \end{equation}
    where $\bm\omega \coloneq \left( \omega_{k,j} \triangleq \frac{c_j N_{k,j}}{B} \right)_{k,j}$, the \textbf{allocation objective} $\gA_p(\bm\omega; \bm\sigma, \vc)$ is defined as\footnote{Note that the allocation objective's limit as $p \rightarrow \infty$ is well defined as $\max_{k \in [K]} \left( \sum_{j \in [J]} \frac{\omega_{k,j}}{c_j \sigma_{k,j}^2} \right)^{-1}.$}
    \begin{equation}
        \gA_p(\bm\omega; \bm\sigma, \vc) \coloneq \left( \sum_{k \in [K]} \left( \sum_{j \in [J]} \frac{\omega_{k,j}}{c_j \sigma_{k,j}^2} \right)^{-\frac{p}{2}} \right)^{\frac{2}{p}},
    \end{equation}
    and $\gB_p(\bm\omega; \bm\sigma, \vc) \coloneq \left( \sum_{k \in [K]} \max_{j: \omega_{k,j} > 0} \left( \sigma_{k,j}^2 \sum_{j'\in[J]} \frac{\omega_{k,j'}}{c_{j'} \sigma_{k,j'}^2} \right)^{-p} \right)^{\frac{1}{p}}$.
\end{proposition}

\begin{proof}
    Let $W_k \coloneq \sum_{j \in [J]} \frac{N_{k,j}}{\sigma_{k,j}^2}$, and an allocation $(N_{k,j})_{k,j}$ be given.
    Then, first note that
    \begin{equation}
        \hat{s}_k - s_k = W_k^{-1} \sum_{j \in [J]} \frac{1}{\sigma_{k,j}^2} \sum_{i=1}^{N_{k,j}} (s_{k,j}^{(i)} - s_k) = \sum_{j \in [J]} \sum_{i=1}^{N_{k,j}} \underbrace{\frac{s_{k,j}^{(i)} - s_k}{W_k \sigma_{k,j}^2}}_{\triangleq X_j^{(i)}}
    \end{equation}
    
    Note that $\E[X_j^{(i)}] = 0$, $|X_j^{(i)}| \leq \frac{R}{W_k \min_{j : \omega_{k,j}>0} \sigma_{k,j}^2}$, and $\Var[X_j^{(i)}] = \frac{\Var[s_{k,j}^{(i)} - s_k]}{W_k^2 \sigma_{k,j}^4} = \frac{1}{W_k^2 \sigma_{k,j}^2}.$
    
    We now utilize the Bernstein's inequality for bounded random variables~\citep{bernstein} to the aggregated $\hat{s}_k$: with probability at least $1 - \delta$, the following holds for for any $k \in [K]$:
    \begin{align}
        \left| \hat{s}_k - s_k \right| &\leq \sqrt{2 \left( \sum_{j,i} \frac{1}{W_k^2 \sigma_{k,j}^2} \right) \log\frac{2 K}{\delta}} + \frac{R \log\frac{2 K}{\delta}}{3 W_k \min_{j : \omega_{k,j}>0} \sigma_{k,j}^2} \\
        &= \sqrt{2 \log\frac{2 K}{\delta}} W_k^{-\frac{1}{2}} + \frac{R \log\frac{2 K}{\delta}}{3} \left( W_k \min_{j : \omega_{k,j}>0} \sigma_{k,j}^2 \right)^{-1}.
    \end{align}
    We then apply Minkowski's inequality for $\ell_p$-norm as follows:
    \begin{align}
        &\bignorm{\hat{\vs} - \vs}_p \\
        &\leq \left( \sum_{k \in [K]} \left( \sqrt{2 \log\frac{2 K}{\delta}} W_k^{-\frac{1}{2}} + \frac{R \log\frac{2 K}{\delta}}{3} \left( W_k \min_{j : \omega_{k,j}>0} \sigma_{k,j}^2 \right)^{-1} \right)^p \right)^{\frac{1}{p}} \\
        &\leq \left( \sum_{k \in [K]} \left( \sqrt{2 \log\frac{2 K}{\delta}} W_k^{-\frac{1}{2}} \right)^p \right)^{\frac{1}{p}} + \left( \sum_{k \in [K]} \left( \frac{R \log\frac{2 K}{\delta}}{3} \left( W_k \min_{j : \omega_{k,j}>0} \sigma_{k,j}^2 \right)^{-1} \right)^p \right)^{\frac{1}{p}} \tag{Minkowski's inequality} \\
        &= \sqrt{2 \log\frac{2K}{\delta}} \left( \sum_{k \in [K]} W_k^{-\frac{p}{2}} \right)^{\frac{1}{p}} + \frac{R \log\frac{2 K}{\delta}}{3} \left( \sum_{k \in [K]} \left( W_k \min_{j : \omega_{k,j}>0} \sigma_{k,j}^2 \right)^{-p} \right)^{\frac{1}{p}} \\
        &\leq \sqrt{\frac{2}{B} \log\frac{2K}{\delta}} \underbrace{\left( \sum_{k \in [K]} \left( \sum_{j=1}^J \frac{\omega_{k,j}}{c_j \sigma_{k,j}^2} \right)^{-\frac{p}{2}} \right)^{\frac{1}{p}}}_{\triangleq \sqrt{\gA_p(\bm\omega; \bm\sigma, \vc)}}  + \frac{R \log\frac{2 K}{\delta}}{3 B} \underbrace{\left( \sum_{k \in [K]} \max_{j: \omega_{k,j} > 0} \left( \sigma_{k,j}^2 \sum_{j'\in[J]} \frac{\omega_{k,j'}}{c_{j'} \sigma_{k,j'}^2} \right)^{-p} \right)^{\frac{1}{p}}}_{\triangleq \gB_p(\bm\omega; \bm\sigma, \vc)},
    \end{align}
    where in the last inequality, we reparametrize as a ratio of budget allocated to $(k,j)$-pair, $\omega_{k,j} \coloneq \frac{c_j N_{k,j}}{B}$, which is invariant to the amount of budget $B$.
\end{proof}

\subsection{\texorpdfstring{Proof of \Cref{thm:optimal-allocation}: Optimal Allocation}{Proof of Theorem 3.2: Optimal Allocation}}
\label{app:optimal-allocation}
For notational simplicity, let
\[
    g_k\coloneq c_{j^*(k)}\sigma_{k,j^*(k)}^2
    =\min_{j\in[J]}c_j\sigma_{k,j}^2,
    \qquad
    b_k\coloneq\sum_{j\in[J]}\omega_{k,j}.
\]
Here, $b_k$ is the fraction of the budget allocated to query $k$.
Since $g_k>0$, any allocation with finite objective must have $b_k>0$ for every $k$.
Increasing all allocations by a common factor decreases the objective, so the budget is fully used at an optimum: $\sum_kb_k=1$.

For a fixed budget $b_k$ assigned to query $k$, we have
\begin{equation}
    \sum_{j\in[J]}\frac{\omega_{k,j}}{c_j\sigma_{k,j}^2}
    \leq\frac{b_k}{g_k}.
\end{equation}
Equality holds when the budget is assigned only to judges minimizing $c_j\sigma_{k,j}^2$.
In particular, we may assign all of $b_k$ to a single optimal judge $j^*(k)$; if several judges attain the minimum, splitting the budget among them also gives equality.
It remains to optimize the budget fractions across queries.\footnote{The KKT analysis yields the same results but is more complicated.}

\paragraph{Optimal Allocation minimizing $\gA_\infty$}
Using the preceding comparison and $\sum_kb_k=1$, we obtain
\begin{equation}
    \gA_\infty(\bm\omega;\bm\sigma,\vc)
    \geq\max_{k\in[K]}\frac{g_k}{b_k}
    \geq\sum_{k\in[K]}b_k\frac{g_k}{b_k}
    =\sum_{k\in[K]}g_k.
\end{equation}
Equality is attained by setting $b_k=g_k/\sum_{k'}g_{k'}$ and assigning each query's budget to an optimal judge.
Thus,
\begin{equation}
    \omega_{k,j}^*
    =\indicator[j=j^*(k)]\frac{g_k}{\sum_{k'}g_{k'}},
    \qquad
    \gA_\infty^*(\bm\sigma,\vc)=\sum_kg_k.
\end{equation}

\paragraph{Optimal Allocation minimizing $\gA_p$}
For $1\leq p<\infty$, the same comparison gives
\[
    \gA_p(\bm\omega;\bm\sigma,\vc)
    \geq\left(\sum_k \left(\frac{g_k}{b_k}\right)^{p/2}\right)^{2/p}.
\]
By H\"older's inequality with conjugate exponents $(p+2)/2$ and $(p+2)/p$,
\begin{align}
    \sum_k g_k^{\frac{p}{p+2}}
    =\sum_k
    \left[\left(\frac{g_k}{b_k}\right)^{p/2}\right]^{\frac{2}{p+2}}
    b_k^{\frac{p}{p+2}}
    \leq
    \left(\sum_k\left(\frac{g_k}{b_k}\right)^{p/2}\right)^{\frac{2}{p+2}}
    \left(\sum_kb_k\right)^{\frac{p}{p+2}}
    \leq\gA_p(\bm\omega;\bm\sigma,\vc)^{\frac{p}{p+2}}.
\end{align}
Rearranging yields
\[
    \gA_p(\bm\omega;\bm\sigma,\vc)
    \geq\left(\sum_k g_k^{\frac{p}{p+2}}\right)^{\frac{p+2}{p}}.
\]
Equality in H\"older's inequality holds when
$b_k=g_k^{p/(p+2)}/\sum_{k'}g_{k'}^{p/(p+2)}$.
Combining this with the optimal judge choice gives
\begin{equation}
    \omega_{k,j}^*
    =\indicator[j=j^*(k)]
    \frac{g_k^{\frac{p}{p+2}}}{\sum_{k'}g_{k'}^{\frac{p}{p+2}}},
    \qquad
    \gA_p^*(\bm\sigma,\vc)
    =\left(\sum_k g_k^{\frac{p}{p+2}}\right)^{\frac{p+2}{p}}.
\end{equation}
This recovers \Cref{thm:optimal-allocation}, with the convention $p/(p+2)=(p+2)/p=1$ when $p=\infty$.
\qed

%% file: 902Proof-Unknown.tex
\section{\texorpdfstring{Proof of \Cref{thm:est-ivwe}: Error Rate of \textbf{Est-IVWE}}{Proof of Theorem 4.1: Error Rate of \textbf{Est-IVWE}}}
\label{app:upper-bound-unknown}

\subsection{Dealing with Variance Mismatch during Aggregation}
To show the effect of $N_0$ and $\tau$, we will tune them as we go along.
Let $\overline{\sigma}_{k,j} \coloneq \tau + \hat{\sigma}_{k,j}$.
Then the algorithm computes the allocation $\widehat{\omega}_{k,j}^*$ as
\begin{equation}
    \widehat{\omega}_{k,j}^* = \frac{\indicator[j = \hat{j}^*(k)] \left( c_{\hat{j}^*(k)} \overline{\sigma}_{k,\hat{j}^*(k)}^2 \right)^{\frac{p}{p+2}}}{\sum_{k\in[K]} \left( c_{\hat{j}^*(k)} \overline{\sigma}_{k,\hat{j}^*(k)}^2 \right)^{\frac{p}{p+2}}},
\end{equation}
where $\hat{j}^*(k) \coloneq \argmin_{j \in [J]} c_j \overline{\sigma}_{k,j}^2$, and allocates each query $k$ to $\hat{j}^*(k)$ for $\widehat{N}_k \coloneq \left\lceil\frac{B' \widehat{\omega}_{k,\hat{j}^*(k)}^*}{c_{\hat{j}^*(k)}}\right\rceil$ times. Here, $B'=B-CN_0-Kc_{\max}$ with $C:=K\sum_jc_j$. Since the budget allocated to Phase-II is at most $B'+\sum_k c_{\hat j^*(k)}\leq B-CN_0$, the total budget constraint is satisfied under this allocation. Conditional on Phase I, the allocation is fixed and Phase-II samples are independent and fresh. All expressions for $p=\infty$ are interpreted by taking limits.

We first follow the proof of \Cref{prop:upper-bound-allocation} while accounting for the variance mismatch: the empirical variance used in \textbf{Est-IVWE} does not match the true variance.
Applying the Bernstein's inequality to $\hat{s}_k - s_k = \frac{1}{\widehat{N}_k} \sum_{i=1}^{\widehat{N}_k} \left( s_{k,\hat{j}^*(k)}^{(i)} - s_k \right)$, the following holds with probability at least $1 - \delta$: for all $k \in [K]$,
\begin{equation}
    |\hat{s}_k - s_k| \leq \sqrt{\frac{2 \sigma_{k,\hat{j}^*(k)}^2 \log\frac{2K}{\delta}}{\widehat{N}_k}} + \frac{ 2 R \log\frac{2K}{\delta}}{3 \widehat{N}_k}.
\end{equation}
By Minkowski's inequality, we then have
\begin{align}
    \varepsilon_p &\leq \sqrt{2 \log\frac{2K}{\delta}} \left( \sum_{k \in [K]} \left( \frac{\widehat{N}_k}{\sigma_{k,\hat{j}^*(k)}^2} \right)^{-\frac{p}{2}} \right)^{\frac{1}{p}} + \frac{2 R \log\frac{2K}{\delta}}{3} \left( \sum_{k \in [K]} \widehat{N}_k^{-p} \right)^{\frac{1}{p}} \\
    & \leq \sqrt{\frac{2 \log\frac{2K}{\delta}}{B'}} \underbrace{\left( \sum_{k \in [K]} \left( \frac{\widehat{\omega}_{k,\hat{j}^*(k)}^*}{c_{\hat{j}^*(k)} \sigma_{k,\hat{j}^*(k)}^2} \right)^{-\frac{p}{2}} \right)^{\frac{1}{p}}}_{(a)} + \frac{2 R \log\frac{2K}{\delta}}{3 B'} \underbrace{\left( \sum_{k \in [K]} \left( \frac{\widehat{\omega}_{k,\hat{j}^*(k)}^*}{c_{\hat{j}^*(k)}} \right)^{-p} \right)^{\frac{1}{p}}}_{(b)}.
\end{align}

We bound each $(a)$ and $(b)$ separately.
We utilize the useful confidence bound on variance estimates given by~\citet[Theorem 10]{maurer-pontil}.
By applying a union bound across all query-judge pairs, we obtain the following result based on the empirical Bernstein inequality:
\begin{lemma}[Theorem 10 of \citet{maurer-pontil}]
\label{lem:empirical-bernstein}
    With $\hat{\sigma}_{k,j}^2$ computed in Phase I of Algorithm~\ref{alg:estimate-then-ivwe},
    \begin{equation}
        \sP\left( \left| \hat{\sigma}_{k,j} - \sigma_{k,j} \right| \leq R \sqrt{\frac{2 \log\frac{2 K J}{\delta}}{N_0 - 1}}, \quad \forall (k, j) \in [K] \times [J] \right) \geq 1 - \delta.
    \end{equation}
\end{lemma}

We then derive the following lemma that sandwiches our optimistic biased estimates $\overline{\sigma}_{k,j}$'s with the true $\sigma_{k,j}$'s, up to $\tau$:
We then derive the following lemma that sandwiches our optimistic biased estimates $\overline{\sigma}_{k,j}$'s with the true $\sigma_{k,j}$'s, up to $\tau$:
\begin{lemma}\label{lem:easy-hard}
    Suppose that $\tau \geq R \sqrt{\frac{2}{N_0 - 1} \log\frac{2 K J}{\delta}}$. Then,
    \begin{equation}
        \sP\left( \sigma_{k,j} \leq \overline{\sigma}_{k,j} \leq \sigma_{k,j} + 2 \tau, \ \forall (k, j) \in [K] \times [J] \right) \geq 1 - \delta.
    \end{equation}
\end{lemma}
\begin{proof}
    By \Cref{lem:empirical-bernstein} and the bound on $\tau$, with probability at least $1-\delta$ we have $|\hat\sigma_{k,j}-\sigma_{k,j}|\leq\tau$ simultaneously for all pairs. On this event,
    \begin{equation}
        \sigma_{k,j}\leq\hat\sigma_{k,j}+\tau=\overline\sigma_{k,j}\leq\sigma_{k,j}+2\tau,
    \end{equation}
    which proves the claim.
\end{proof}
From the above lemma, we choose $\tau = R \sqrt{\frac{2}{N_0 - 1} \log\frac{2 K J}{\delta}}$ before rescaling $\delta$ to $\delta/2$ below.

\subsection{\texorpdfstring{Bounding $(a)$}{Bounding (a)}}
Substituting the definition of $\widehat{\omega}_{k,\hat{j}^*(k)}^*$ into the term $(a)$, we get
\begin{align}
    (a) &= \left( \sum_{k \in [K]} \left( \frac{1}{c_{\hat{j}^*(k)} \sigma_{k,\hat{j}^*(k)}^2} \cdot \frac{\left( c_{\hat{j}^*(k)} \overline{\sigma}_{k,\hat{j}^*(k)}^2 \right)^{\frac{p}{p+2}}}{\sum_{k\in[K]} \left( c_{\hat{j}^*(k)} \overline{\sigma}_{k,\hat{j}^*(k)}^2 \right)^{\frac{p}{p+2}}} \right)^{-\frac{p}{2}} \right)^{\frac{1}{p}} \\
    &= \sqrt{\sum_{k\in[K]} \left( c_{\hat{j}^*(k)} \overline{\sigma}_{k,\hat{j}^*(k)}^2 \right)^{\frac{p}{p+2}}}
    \left( \sum_{k \in [K]} \frac{\sigma_{k,\hat{j}^*(k)}^p}{\overline{\sigma}_{k,\hat{j}^*(k)}^p} \left( c_{\hat{j}^*(k)} \overline{\sigma}_{k,\hat{j}^*(k)}^2 \right)^{\frac{p}{p+2}} \right)^{\frac{1}{p}} \\
    &\leq \sqrt{\sum_{k\in[K]} \left( c_{j^*(k)} \overline{\sigma}_{k,j^*(k)}^2 \right)^{\frac{p}{p+2}}}
    \left( \sum_{k \in [K]} \frac{\sigma_{k,\hat{j}^*(k)}^p}{\overline{\sigma}_{k,\hat{j}^*(k)}^p} \left( c_{j^*(k)} \overline{\sigma}_{k,j^*(k)}^2 \right)^{\frac{p}{p+2}} \right)^{\frac{1}{p}}, \tag{Minimality of $\hat{j}^*(k)$}
\end{align}
where $j^*(k) \triangleq \argmin_{j \in [J]} c_j \sigma_{k,j}^2$ is the true optimal judge per query; it holds that $c_{\hat{j}^*(k)} \overline{\sigma}_{k,\hat{j}^*(k)}^2\leq c_{j^*(k)} \overline{\sigma}_{k,j^*(k)}^2$ by construction of \Cref{alg:estimate-then-ivwe}.
With these, we have that with probability at least $1 - 2 \delta$,
\begin{align}
    (a) &\leq \sqrt{\sum_{k\in[K]} \left( c_{j^*(k)} \overline{\sigma}_{k,j^*(k)}^2 \right)^{\frac{p}{p+2}}}
    \left( \sum_{k \in [K]} \left( c_{j^*(k)} \overline{\sigma}_{k,j^*(k)}^2 \right)^{\frac{p}{p+2}} \right)^{\frac{1}{p}} \\
    &= \left( \sum_{k\in[K]} \left( c_{j^*(k)} \overline{\sigma}_{k,j^*(k)}^2 \right)^{\frac{p}{p+2}} \right)^{\frac{p+2}{2p}} \label{eqn:empirical-A} \\
    &\leq \left( \sum_{k\in[K]} \left( \sqrt{c_{j^*(k)}} \sigma_{k,j^*(k)} + 2 \tau \sqrt{c_{j^*(k)}} \right)^{\frac{2p}{p+2}} \right)^{\frac{p+2}{2p}}. \tag{\Cref{lem:easy-hard}}
\end{align}
Let us denote $q := \frac{2p}{p + 2}$.
When $p \geq 2$, we have that $q \geq 1$, and thus by Minkowski's inequality of $L_q$,
\begin{align}
    &\left( \sum_{k\in[K]} \left( \sqrt{c_{j^*(k)}} \sigma_{k,j^*(k)} + 2 \tau \sqrt{c_{j^*(k)}} \right)^q \right)^{\frac{1}{q}} \\
    &\leq \underbrace{\left( \sum_{k\in[K]} \left( \sqrt{c_{j^*(k)}} \sigma_{k,j^*(k)} \right)^q \right)^{\frac{1}{q}}}_{= \sqrt{\gA_p^*(\bm\sigma, \vc)}} + \left( \sum_{k\in[K]} \left( 2 \tau \sqrt{c_{j^*(k)}} \right)^q \right)^{\frac{1}{q}} \\
    &= \sqrt{\gA_p^*(\bm\sigma, \vc)} + 2 \tau \underbrace{\left( \sum_{k\in[K]} \left( \sqrt{c_{j^*(k)}} \right)^q \right)^{\frac{1}{q}}}_{\triangleq \bignorm{\sqrt{\vc^*}}_{q}}.
\end{align}
When $1 \leq p < 2$, we have that $q \in [2/3, 1)$.
Then,
\begin{align}
    &\left( \sum_{k\in[K]} \left( \sqrt{c_{j^*(k)}} \sigma_{k,j^*(k)} + 2 \tau \sqrt{c_{j^*(k)}} \right)^q \right)^{\frac{1}{q}} \\
    &\leq \left( \sum_{k\in[K]} \left( \sqrt{c_{j^*(k)}} \sigma_{k,j^*(k)} \right)^q + \sum_{k\in[K]} \left( 2 \tau \sqrt{c_{j^*(k)}} \right)^q \right)^{\frac{1}{q}} \tag{subadditivity of $L_q$ when $q \in (0, 1)$} \\
    &\leq \underbrace{\left( \sum_{k\in[K]} \left( \sqrt{c_{j^*(k)}} \sigma_{k,j^*(k)} \right)^q \right)^{\frac{1}{q}}}_{= \sqrt{\gA_p^*(\bm\sigma, \vc)}} \\
    &\quad\quad + \frac{1}{q} \left( \sum_{k\in[K]} \left( \sqrt{c_{j^*(k)}} \sigma_{k,j^*(k)} \right)^q + \sum_{k\in[K]} \left( 2 \tau \sqrt{c_{j^*(k)}} \right)^q \right)^{\frac{1}{q} - 1} \sum_{k\in[K]} \left( 2 \tau \sqrt{c_{j^*(k)}} \right)^q \tag{$x \mapsto x^{1/q}$ is convex for $q \in (0, 1)$, For convex $f$, $f(x + y) \leq f(x) + f'(x + y) y$} \\
    &\leq \sqrt{\gA_p^*( \bm\sigma, \vc)} + \frac{1}{q} \left\{ \left( \sum_{k\in[K]} \left( \sqrt{c_{j^*(k)}} \sigma_{k,j^*(k)} \right)^q \right)^{\frac{1}{q} - 1} + (2 \tau)^{1 - q} \bignorm{\sqrt{\vc^*}}_q^{1-q} \right\} (2 \tau)^q \bignorm{\sqrt{\vc^*}}_q^q \tag{$x \mapsto x^{\frac{1}{q} - 1}$ is subadditive for $q \in [2/3, 1)$} \\
    &= \sqrt{\gA_p^*(\bm\sigma, \vc)} + \underbrace{\frac{\bignorm{\sqrt{\vh^*}}_q^{1-q} 2^q \bignorm{\sqrt{\vc^*}}_q^q}{q}}_{\triangleq C_{p,2}^*} \tau^q + \underbrace{\frac{2 \bignorm{\sqrt{\vc^*}}_q}{q}}_{\triangleq C_{p,1}^*} \tau. \tag{$\sqrt{\vh^*} \coloneq \left( \sqrt{c_{j^*(k)}} \sigma_{k,j^*(k)} \right)_{k \in [K]}$}
\end{align}
We also abuse the notation slightly and denote $C_{p,1}^* = 2 \bignorm{\sqrt{\vc^*}}_q$ for $p \geq 2$.

Combining the two cases, we have that
\begin{equation}
    (a) \leq \sqrt{\gA_p^*(\bm\sigma, \vc)} + C_{p,1}^* \tau + \indicator[1 \leq p < 2] C_{p,2}^* \tau^{\frac{2p}{p+2}}.
\end{equation}

\subsection{\texorpdfstring{Bounding $(b)$}{Bounding (b)}}
Again substituting in the definition of $\widehat{\omega}_{k,\hat{j}^*(k)}^*$, we go through a somewhat different (from when we proved $(a)$) series of algebraic manipulations as follows:
\begin{align}
    (b) &= \left( \sum_{k \in [K]} \left( \frac{1}{c_{\hat{j}^*(k)}} \frac{\left( c_{\hat{j}^*(k)} \overline{\sigma}_{k,\hat{j}^*(k)}^2 \right)^{\frac{p}{p+2}}}{\sum_{k\in[K]} \left( c_{\hat{j}^*(k)} \overline{\sigma}_{k,\hat{j}^*(k)}^2 \right)^{\frac{p}{p+2}}} \right)^{-p} \right)^{\frac{1}{p}} \\
    &= \left( \sum_{k\in[K]} \left( c_{\hat{j}^*(k)} \overline{\sigma}_{k,\hat{j}^*(k)}^2 \right)^{\frac{p}{p+2}} \right)
    \left( \sum_{k \in [K]} \left( \frac{c_{\hat{j}^*(k)}}{\overline{\sigma}_{k,\hat{j}^*(k)}^p} \right)^{\frac{2p}{p+2}} \right)^{\frac{1}{p}} \\
    &\leq \left( \sum_{k\in[K]} \left( c_{j^*(k)} \overline{\sigma}_{k,j^*(k)}^2 \right)^{\frac{p}{p+2}} \right)
    \left( \sum_{k \in [K]} \left( \frac{c_{j^*(k)} \overline{\sigma}_{k,j^*(k)}^2}{\overline{\sigma}_{k,\hat{j}^*(k)}^{p+2}} \right)^{\frac{2p}{p+2}} \right)^{\frac{1}{p}} \tag{Minimality of $\hat{j}^*(k)$} \\
    &\leq \tau^{-2} \left( \sum_{k\in[K]} \left( c_{j^*(k)} \overline{\sigma}_{k,j^*(k)}^2 \right)^{\frac{p}{p+2}} \right)
    \left( \sum_{k \in [K]} \left( c_{j^*(k)} \overline{\sigma}_{k,j^*(k)}^2 \right)^{\frac{2p}{p+2}} \right)^{\frac{1}{p}} \tag{$\overline{\sigma}_{k,\hat{j}^*(k)} \geq \tau$ by \Cref{lem:easy-hard}} \\
    &\leq \tau^{-2} \left( \sum_{k\in[K]} \left( c_{j^*(k)} \overline{\sigma}_{k,j^*(k)}^2 \right)^{\frac{p}{p+2}} \right)
    \left( \left( \sum_{k \in [K]} \left( c_{j^*(k)} \overline{\sigma}_{k,j^*(k)}^2 \right)^{\frac{p}{p+2}} \right)^2 \right)^{\frac{1}{p}} \tag{$\sum_k y_k^2 \leq (\sum_k y_k)^2$ for $y_k \geq 0$} \\
    &= \left\{ \tau^{-1} \underbrace{\left( \sum_{k\in[K]} \left( c_{j^*(k)} \overline{\sigma}_{k,j^*(k)}^2 \right)^{\frac{p}{p+2}} \right)^{\frac{p+2}{2p}}}_{\text{same as Eqn.~\eqref{eqn:empirical-A}!}} \right\}^2.
\end{align}

Thus, borrowing our computations from bounding $(a)$ and using $(a + b + c)^2 \leq 3(a^2 + b^2 + c^2)$, we have that
\begin{equation}
    (b) \leq 3 \left( \gA_p^*(\bm\sigma, \vc) \tau^{-2} + (C_{p,1}^*)^2 + \indicator[1 \leq p < 2] (C_{p,2}^*)^2 \tau^{\frac{2p - 4}{p + 2}} \right).
\end{equation}

\subsection{Combining Everything and the Large-Budget Regime}
\label{app:estivwe-regime}

We now combine the preceding bounds and specify the choice of $N_0$.
Recall that $C=K\sum_{j}c_j$ and
\[
    \tau
    =
    R\sqrt{\frac{2}{N_0-1}\log\frac{4KJ}{\delta}},
\]
where we have rescaled $\delta$ to $\delta/2$. With this rescaling, the variance-concentration in Phase I fails with probability at most $\delta/2$, and, conditional on Phase I, the estimation concentration in Phase II fails with probability at most $\delta/2$.
Thus, by a union bound, both events hold simultaneously with probability at least $1-\delta$.

Suppose that
\begin{equation}
    B
    \geq
    2\left(
        CN_0+K\max_{j\in[J]}c_j
    \right),
    \label{eq:estivwe-feasibility}
\end{equation}
then we have $B'\geq B/2$.
Combining the bounds on $(a)$ and $(b)$ gives, with probability at least
$1-\delta$,
\begin{align}
\varepsilon_p
\leq{}&
\sqrt{\frac{2\log\frac{4K}{\delta}}{B'}}
\left(
    \sqrt{\gA_p^*(\bm\sigma,\vc)}
    +C_{p,1}^*\tau
    +\indicator[1\leq p<2]C_{p,2}^*\tau^q
\right)
\nonumber\\
&+
\frac{2R\log\frac{4K}{\delta}}{B'}
\left(
    \gA_p^*(\bm\sigma,\vc)\tau^{-2}
    +(C_{p,1}^*)^2
    +\indicator[1\leq p<2](C_{p,2}^*)^2
        \tau^{2q-2}
\right),
\label{eq:estivwe-combined}
\end{align}
where $q=\frac{2p}{p+2}$, with $q=2$ when $p=\infty$.

To separate the oracle term from the cost of learning the variances, note that
by \Cref{lem:algebraic-inequality} and
\eqref{eq:estivwe-feasibility},
\begin{align}
\varepsilon_p
\leq{}&
\sqrt{
    \frac{
        2\gA_p^*(\bm\sigma,\vc)
        \log\frac{4K}{\delta}
    }{B}
}
\nonumber\\
&+
\frac{
    \left(CN_0+K\max_{j\in[J]}c_j\right)
    \sqrt{
        2\gA_p^*(\bm\sigma,\vc)
        \log\frac{4K}{\delta}
    }
}{B^{3/2}}
\nonumber\\
&+
2\sqrt{\frac{\log\frac{4K}{\delta}}{B}}
\left(
    C_{p,1}^*\tau
    +\indicator[1\leq p<2]C_{p,2}^*\tau^q
\right)
\nonumber\\
&+
\frac{4R\log\frac{4K}{\delta}}{B}
\left(
    \gA_p^*(\bm\sigma,\vc)\tau^{-2}
    +(C_{p,1}^*)^2
    +\indicator[1\leq p<2](C_{p,2}^*)^2
        \tau^{2q-2}
\right).
\label{eq:estivwe-uncombined-remainder}
\end{align}

The first term is exactly the leading oracle error rate.
The remaining terms capture the loss from the initial exploration,
the error in estimating the variances, and the lower-order Bernstein term.
We now choose $N_0$ so that these terms decay strictly faster than
$B^{-1/2}$.

\paragraph{Case I: $p\geq 2$.}
In this case, the terms involving $C_{p,2}^*$ vanish. Among the lower-order terms in
Eqn.~\eqref{eq:estivwe-uncombined-remainder}, the two dominant terms depending on
$N_0$ are
\[
    2C_{p,1}^*
    \sqrt{\frac{\log\frac{4K}{\delta}}{B}}\,\tau
    =
    O\left(
        \frac{
            R\log\frac{4KJ}{\delta}
        }{\sqrt{BN_0}}
    \right),
\]
where we used
$\log(4K/\delta)\leq \log(4KJ/\delta)$, and
\[
    \frac{4R\log\frac{4K}{\delta}}{B}
    \gA_p^*(\bm\sigma,\vc)\tau^{-2}
    =
    O\left(
        \frac{\gA_p^*(\bm\sigma,\vc)N_0}{RB}
    \right).
\]
Since $N_0 =2 \vee \left\lceil(2B)^{1/3}\left(R^2\log\frac{4KJ}{\delta}\right)^{2/3}\right\rceil$, we have $N_0 = \Theta(B^{1/3})$, and $\tau = \Theta(B^{-1/6})$ ignoring factors independent of $B$.
Thus, the two terms above are both of order $B^{-2/3}$. The remaining lower-order terms decay faster as
\begin{equation}
    \frac{CN_0 + K \max_j c_j}{B^{3/2}} = O(B^{-7/6})
\end{equation}
up to fixed instance-dependent factors.
Therefore, every term after the leading oracle term in Eqn.~\eqref{eq:estivwe-uncombined-remainder} is of order at most $B^{-2/3}$.
More precisely, using
$\log(4K/\delta)\leq \log(4KJ/\delta)$, we finally obtain
\begin{equation}
    \varepsilon_p
    =
    \sqrt{
        \frac{
            2\gA_p^*(\bm\sigma,\vc)
            \log\frac{4K}{\delta}
        }{B}
    }
    +
    \gO\left(
        \frac{
            R^{1/3}
            \mathcal{A}_p^*(\bm\sigma,\vc)
            \left(\log\frac{4KJ}{\delta}\right)^{2/3}
        }{B^{2/3}}
    \right).
\end{equation}

\paragraph{Case II: $1\leq p<2$.}
Recall that $q=\frac{2p}{p+2}<1$.
Among the lower-order terms in Eqn.~\eqref{eq:estivwe-uncombined-remainder}, the two dominant terms depending on $N_0$ are
\[
    2C_{p,2}^*
    \sqrt{\frac{\log\frac{4K}{\delta}}{B}}\,\tau^q
    =
    O\left(
        \frac{
            R^q
            \left(\log\frac{4KJ}{\delta}\right)^{\frac{q+1}{2}}
        }{
            \sqrt{B}N_0^{q/2}
        }
    \right),
\]
where we used $\log(4K/\delta)\leq\log(4KJ/\delta)$, and
\[
    \frac{4R\log\frac{4K}{\delta}}{B}
    \gA_p^*(\bm\sigma,\vc)\tau^{-2}
    =
    O\left(
        \frac{\gA_p^*(\bm\sigma,\vc)N_0}{RB}
    \right).
\]
The choice of $N_0 = 2 \vee \left\lceil 2^{\frac{p}{2p+2}} R^{\frac{3p+2}{2p+2}} \left( \log\frac{4KJ}{\delta} \right)^{\frac{3p+2}{4p+4}} B^{\frac{p+2}{4p+4}} \right\rceil$ balances the orders of these two terms to the order of $B^{-\frac{3p+2}{4p+4}}$. 
The remaining lower-order terms in Eqn.~\eqref{eq:estivwe-uncombined-remainder} decay strictly faster:
\[
    B^{-1/2}\tau
    =
    O\left(
        B^{-\frac{5p+6}{8p+8}}
    \right),
    \qquad
    \frac{\tau^{2q-2}}{B}
    =
    O\left(
        B^{-\frac{5p+2}{4p+4}}
    \right),
\]
and
\[
    \frac{CN_0+K\max_j c_j}{B^{3/2}}
    =
    O\left(
        B^{-\frac{5p+4}{4p+4}}
    \right),
    \qquad
    \frac{1}{B}=O(B^{-1}).
\]
Therefore, every term after the leading oracle term is of order at most $B^{-\frac{3p+2}{4p+4}}$.
Finally, we obtain
\begin{equation}
    \varepsilon_p
    =
    \sqrt{
        \frac{
            2\gA_p^*(\bm\sigma,\vc)
            \log\frac{4K}{\delta}
        }{B}
    }
    +
    \gO\left(
        \frac{
            R^{\frac{p}{2p+2}}
            \mathcal{A}_p^*(\bm\sigma,\vc)
            \left(
                \log\frac{4KJ}{\delta}
            \right)^{\frac{3p+2}{4p+4}}
        }{
            B^{\frac{3p+2}{4p+4}}
        }
    \right).
\end{equation}

We choose
\begin{equation}
    N_0
    =
    \max\left\{
        2,\,
        \left\lceil
        2^{\frac{p}{2p+2}}
        R^{\frac{3p+2}{2p+2}}
        \left(
            \log\frac{4KJ}{\delta}
        \right)^{\frac{3p+2}{4p+4}}
        B^{\frac{p+2}{4p+4}}
        \right\rceil
    \right\}.
\end{equation}
Here the dominant variance-estimation term is
$B^{-1/2}\tau^q$, which scales as
$B^{-1/2}N_0^{-q/2}$.
Balancing this term with the exploration loss $N_0/B$ gives the above
choice of $N_0$.
It follows that
\[
    N_0
    =
    \Theta \left(
        B^{\frac{p+2}{4p+4}}
    \right),
    \qquad
    \tau
    =
    \Theta \left(
        B^{-\frac{p+2}{8p+8}}
    \right),
\]
and every remainder term in
\eqref{eq:estivwe-uncombined-remainder} is of order at most
$B^{-\frac{3p+2}{4p+4}}$.
Therefore,
\begin{equation}
    \varepsilon_p
    =
    \sqrt{
        \frac{
            2\gA_p^*(\bm\sigma,\vc)
            \log\frac{4K}{\delta}
        }{B}
    }
    +
    \gO\left(
        \frac{
            R^{\frac{p}{2p+2}}
            \mathcal{A}_p^*(\bm\sigma,\vc)
            \left(
                \log\frac{4KJ}{\delta}
            \right)^{\frac{3p+2}{4p+4}}
        }{
            B^{\frac{3p+2}{4p+4}}
        }
    \right).
\end{equation}
For either choice of $N_0$, we have $N_0=o(B)$; hence the condition
\eqref{eq:estivwe-feasibility} is satisfied for all sufficiently large $B$.
For $p\geq2$, the remainder is $O(B^{-2/3})$, while for $1\leq p<2$ it is
$O(B^{-(3p+2)/(4p+4)})$.
Both decay strictly faster than the $B^{-1/2}$ oracle term.
Therefore, for every fixed instance with
$\gA_p^*(\bm\sigma,\vc)>0$, the ratio of the remainder to the leading oracle
term vanishes as $B\to\infty$.
Consequently, there exists a finite instance-dependent budget $B_0$ beyond
which the oracle term dominates the remainder.

Thus, beyond a finite instance-dependent budget $B_0$, the oracle term dominates the remainder, yielding the leading-order instance-optimality of \Cref{thm:est-ivwe}.
This threshold is neither needed by the learner nor used to set $N_0$, and need not be uniform over variance profiles: it may become large as $\gA_p^*(\bm\sigma,\vc)$ approaches zero.
\qed

We conclude with the proof of the simple algebraic inequality that we have utilized:
\begin{lemma}\label{lem:algebraic-inequality}
    For $x \in [0, 1/2]$, it holds that $(1-x)^{-1/2} \leq 1 + x$.
\end{lemma}
\begin{proof}
Let's define a function $f(x) = 1+x - (1-x)^{-1/2}$ on $(-\infty, 1)$, then $f(0)=0$ and $f(0.5)>0$.
Its derivatives are 
\begin{equation}
    f'(x) = 1 - \frac{1}{2}(1-x)^{-3/2} \quad \text{and} \quad f''(x) = -\frac{3}{4} (1-x)^{-5/2},
\end{equation}
which imply that $f$ is concave, increasing at $x=0$ ($f'(0)>0$), and decreasing at $x=0.5$ ($f'(0.5)$).
Considering the shape of the function $f$, which has a single mode (maximum) in $(0, 1/2)$, $f(0) = 0$ and $f(0.5)>0$ imply that $f >0$ on $[0, 1/2]$.
\end{proof}

%% file: 903Proof-Lower-whp.tex
\section{\texorpdfstring{Proof of \Cref{thm:lower-bound-whp}: Fano-Type High-Probability Lower Bound}{Proof of Theorem 5.1: Fano-Type High-Probability Lower Bound}}
\label{app:whp-lower-bound}

Our Fano-type change-of-measure argument relies on the binary KL data processing inequality~\citep{garivier2019lower}, as utilized in \citet{jedra2023identification} to derive a $\log\frac{1}{\delta}$-dependent minimax lower bound.\footnote{This is somewhat reminiscent of Birg\'{e}'s refinement of Fano's inequality~\citep{tsybakov,gerchinovitz2020fano}.}
The proof ``bifurcates'' at whether $1 \leq p < 2$ or $p \geq 2$.
For the former, we construct a pair of dense perturbations, while for the latter, we construct a pair of 1-sparse perturbations for each query and combine their information constraints using the budget.

Throughout, we assume that the learner knows the given cost vector $\vc$ and the true variance profile $\bm\sigma$ with $\max_j \sigma_{k,j}^2 \leq R_k^2/2\text{ for every }k\in[K]$.

\paragraph{Construction of Alternate Models.}
Let $\gA$ be any $(B, \delta, \ell_p)$-budget efficient algorithm (\Cref{def:budget-efficient}) with error rate $\varepsilon_p$, \emph{potentially adaptive}.
$\gA$ interacts with the true environment $\vs^\star$ sequentially\footnote{Even though the algorithm is batched, one could view it as sequentially observing i.i.d. samples. Note that this includes adaptive algorithms.} as follows at each timestep $t$, when the learner selects query $k_t$ and judge $j_t$, they observe a score $\tilde{s}_t$.
Let us denote $N_{k,j}(T) \coloneq \sum_t \indicator[k_t = k, j_t = j]$ as the number of times $\gA$ queries the pair $(k, j)$ throughout the interaction.
We note that the knapsack constraint is in place as follows:
\begin{equation}
    \sum_j c_j \sum_k N_{k,j}(T) \leq B.
\end{equation}

Let $(\vs^\star, \bm\sigma, \vc)$ be given and fixed, as we aim for an instance-wise lower bound.
The following proposition, whose proof is deferred to \Cref{app:beta-construction}, provides the desired construction of hypotheses while satisfying the necessary properties:
\begin{proposition}
\label{prop:beta-construction}
    There exists a family of probability measures $\sP_{\vs^\star + \bm\Delta} = \left( \sP_{\vs^\star + \bm\Delta, k, j} \right)_{k,j}$ indexed by perturbation $\bm\Delta = (\Delta_k)_{k \in [K]}$ with $|\Delta_k| \leq \frac{R_k}{4}$ that satisfies the following: for each $(k, j) \in [K] \times [J],$
    \begin{enumerate}
        \item[$(i)$] $\E[X] = s_k^\star + \Delta_k$ and $\Var[X] = \sigma_{k,j}^2$ for $X \sim \sP_{\vs^\star + \bm\Delta, k, j}$, and $\mathrm{supp}(\sP_{\vs^\star + \bm\Delta, k, j}) \subseteq [0, R]$,
        \item[$(ii)$] $\KL\left( \sP_{\vs^\star, k, j}, \sP_{\vs^\star + \bm\Delta, k, j} \right) \vee \KL\left( \sP_{\vs^\star + \bm\Delta, k, j}, \sP_{\vs^\star, k, j} \right) \leq C_1 \frac{\Delta_k^2}{\sigma_{k,j}^2}$, and
        \item[$(iii)$] $\KL\left( \sP_{\vs^\star + \bm\Delta, k, j}, \sP_{\vs^\star - \bm\Delta, k, j} \right) \vee \KL\left( \sP_{\vs^\star - \bm\Delta, k, j}, \sP_{\vs^\star + \bm\Delta, k, j} \right) \leq C_2 \frac{\Delta_k^2}{\sigma_{k,j}^2}$.
    \end{enumerate}
    where $C_1, C_2 > 0$ are absolute constants.
\end{proposition}
The construction is rather tedious as it is based on appropriate family of beta distributions (which ensures boundedness) while making sure that the variance remains invariant across perturbations, and we need to make sure that the KL-divergence is quadratically bounded as above.
(Without the boundedness assumption, one can simply utilize location family of Gaussian distributions.)
For this high-probability lower bound proof, we will primarily utilize $(i)$ and $(ii)$, and later for the in-expectation lower bound, we will utilize $(ii)$ and $(iii)$.

We consider the canonical bandit model (See \citet[Section 4.6]{banditalgorithms}) over the observations, and let $\sP_{\vs^\star}$ and $\sP_{\vs \triangleq \vs^\star + \bm\Delta}$ be the two probability measures induced by the true and alternate instance, respectively.
The adversary's choice of perturbation $\bm\Delta$ will be optimized for later.

\paragraph{KL-Divergence Decomposition.}
By the Divergence Decomposition Lemma~\citep[Lemma 15.1]{banditalgorithms}, the KL divergence between the probability measures $\sP_{\vs^\star}$ and $\sP_{\vs}$ is given by
\begin{align}
    \KL(\sP_{\vs^\star}, \sP_{\vs}) &= \sum_{k \in [K]} \sum_{j \in [J]} \E^\gA_{\vs^\star}[N_{k,j}(T)] \KL\left( \sP_{\vs^\star,k,j}, \sP_{\vs,k,j} \right) \\
    &\leq C_1 \sum_{k \in [K]} \underbrace{\sum_{j \in [J]} \frac{\E^\gA_{\vs^\star}[N_{k,j}(T)]}{\sigma_{k,j}^2}}_{\triangleq w_k} (s_k - s_k^\star)^2. \tag{\Cref{prop:beta-construction} $(ii)$}
\end{align}

We first minimize the RHS over perturbations with $\bignorm{\vs-\vs^\star}_p\geq2\varepsilon_p$. The minimizers below lie on the boundary and satisfy $|s_k-s_k^\star|\leq2\varepsilon_p\leq R_k/4$, so both signs of each perturbation are valid.
This is characterized in the next lemma, whose proof is deferred to \Cref{app:infimum-hard}:
\begin{lemma}\label{lem:infimum-hard}
    Let $w_k > 0$ and $\vs^\star \in \sR^K$ be fixed. Let $\vs = (s_k)_{k \in [K]} \in \sR^K$ be an optimal solution of the following constrained optimization:
    \begin{equation}
        V^\star = \min_{\vs \in \sR^K} \sum_k w_k (s_k - s_k^\star)^2, \text{ subj. to } \bignorm{\vs - \vs^\star}_p \geq 2 \varepsilon_p.
    \end{equation}
    Then,
    \begin{enumerate}
        \item \textbf{Case I. $1 \leq p < 2$:}
        \begin{equation}
            s_k = s_k^\star + 2\varepsilon_p \frac{w_k^{-\frac{1}{2-p}}}{\left( \sum_k w_k^{-\frac{p}{2-p}} \right)^{\frac{1}{p}}}, \quad V^\star = \frac{4 \varepsilon_p^2}{\left( \sum_k w_k^{-\frac{p}{2-p}} \right)^{\frac{2-p}{p}}}.
        \end{equation}
        \item \textbf{Case II. $p \geq 2$:} denoting $k^\star \in \argmin_k w_k$,
        \begin{equation}
            \vs = \vs^\star + 2 \varepsilon_p \ve_{k^\star}, \quad V^\star = 4 \varepsilon_p^2 \min_k w_k.
        \end{equation}
    \end{enumerate}
    In both cases, note that as $\varepsilon_p \leq \frac{\min_k R_k}{8}$ (as given in the theorem statement), $\vs$ satisfies $\vs \in [0, R]^K$ and leads to a valid probability distribution.
\end{lemma}

\paragraph{Optimal Hard Instances.}
Now the idea is to construct a set of optimal hard instances and utilize the Bernoulli KL-variant of the data processing inequality~\citep[Lemma 1]{garivier2019lower}, which we recall here:
\begin{lemma}[Data Processing Inequality; Lemma 1 of \cite{garivier2019lower}]
\label{lem:data-processing}
    Consider a measurable space $(\Gamma, \gG)$ equipped with two probability measure $\sP_1$ and $\sP_2$.
    Then, we have that
    \begin{equation}
        \KL(\sP_1, \sP_2) \geq \sup_Z \kl(\E_1[Z], \E_2[Z]),
    \end{equation}
    where $\sup_Z$ is over all possible $\gG$-measurable random variable $Z : \Omega \rightarrow [0, 1]$.
\end{lemma}

First, let $\bm\Delta$ satisfy $\bignorm{\bm\Delta}_p=2\varepsilon_p$ and consider the two instances $\vs^\pm:=\vs^\star\pm\bm\Delta$. Both lie in the prescribed neighborhood and satisfy $|\Delta_k|\leq2\varepsilon_p\leq R_k/4$. Define their success events by
\[
 \gE_\pm:=\{\bignorm{\hat{\vs}-\vs^\pm}_p\leq\varepsilon_p\}.
\]
These events are disjoint, since $\bignorm{\vs^+-\vs^-}_p=4\varepsilon_p>2\varepsilon_p$. Consequently, for at least one sign $\eta\in\{+,-\}$,
\[
 \sP_{\vs^\star}(\gE_\eta)\leq\tfrac12,
 \qquad \sP_{\vs^\eta}(\gE_\eta)\geq1-\delta.
\]
The data processing inequality and monotonicity of binary KL give
\begin{equation}\label{eqn:lower-bound}
 C_1\sum_k w_k\Delta_k^2
 \geq\KL(\sP_{\vs^\star},\sP_{\vs^\eta})
 \geq\kl\left(\frac12,1-\delta\right).
\end{equation}

For $\delta\leq 1/16$, we have
\begin{equation}
  \kl\left(\frac12,1-\delta\right) = \frac12\log\frac{1}{4\delta(1-\delta)} \geq \frac12\log\frac1\delta-\log2  \geq \frac14\log\frac1\delta.
\end{equation}
If any $w_k=0$, applying Eqn.~\eqref{eqn:lower-bound} to $\bm\Delta=2\varepsilon_p\ve_k$ gives a contradiction. Hence every budget-efficient algorithm considered here has $w_k>0$.

\paragraph{Case I. $1 \leq p < 2$.}
Choose $\bm\Delta$ from \textbf{Case I} of \Cref{lem:infimum-hard}. Applying Eqn.~\eqref{eqn:lower-bound} to the corresponding pair of instances yields
\[
 \frac{4C_1\varepsilon_p^2}
 {\left(\sum_k w_k^{-\frac{p}{2-p}}\right)^{\frac{2-p}{p}}}
 \geq\frac14\log\frac{1}{\delta}.
\]
Relaxing the expected counts to all nonnegative allocations satisfying the budget, we obtain
\begin{align}
 \varepsilon_p^2
 &\geq\frac{\log(1/\delta)}{16C_1}
 \min_{\substack{N_{k,j}\geq0\\\sum_{k,j}c_jN_{k,j}\leq B}}
 \left(\sum_k\left(\sum_j\frac{N_{k,j}}{\sigma_{k,j}^2}\right)^{-\frac{p}{2-p}}\right)^{\frac{2-p}{p}}\nonumber\\
 &=\frac{\log(1/\delta)}{16C_1B}
 \underbrace{\left(\sum_k
 (c_{j^*(k)}\sigma_{k,j^*(k)}^2)^{\frac p2}\right)^{\frac2p}}
 _{=\gA_{\frac{2p}{2-p}}^*(\bm\sigma,\vc)},
\end{align}
where the last step follows from the allocation optimization in \Cref{app:optimal-allocation}.

\paragraph{Case II. $p \geq 2$.}
For each $k$, apply Eqn.~\eqref{eqn:lower-bound} to the pair $\vs^\star\pm2\varepsilon_p\ve_k$. This gives
\[
 4C_1\varepsilon_p^2w_k\geq\frac14\log\frac1\delta,
 \qquad\forall k\in[K].
\]
By the budget constraint and the definition of $j^*(k)$,
\begin{align}
 \sum_k c_{j^*(k)}\sigma_{k,j^*(k)}^2w_k
 &=\sum_{k,j}\E^\gA_{\vs^\star}[N_{k,j}(T)]
 \frac{c_{j^*(k)}\sigma_{k,j^*(k)}^2}{\sigma_{k,j}^2}\nonumber\\
 &\leq\sum_{k,j}c_j\E^\gA_{\vs^\star}[N_{k,j}(T)]\leq B.
\end{align}
Multiplying the coordinatewise information constraints by $c_{j^*(k)}\sigma_{k,j^*(k)}^2$ and summing therefore yields
\[
 \varepsilon_p^2\geq\frac{\log(1/\delta)}{16C_1B}
 \sum_k c_{j^*(k)}\sigma_{k,j^*(k)}^2
 =\frac{\log(1/\delta)}{16C_1B}\gA_\infty^*(\bm\sigma,\vc).
\]
This argument also covers $p=\infty$; it requires disjointness only between the two opposite perturbations of a given coordinate. Absorbing numerical factors into the universal constants proves the theorem.
\qed

\begin{remark}[Why $\gA_\infty^*$ appears for $p\geq2$]
For $p\geq2$, each alternative used above perturbs only one coordinate. The resulting constraints give
\[
 \varepsilon_p^2\geq\frac{\log(1/\delta)}{32C_1\min_k w_k}.
\]
Thus, minimizing this lower-bound expression over feasible allocations amounts to maximizing $\min_k w_k$. Let $a_k\coloneq c_{j^*(k)}\sigma_{k,j^*(k)}^2$. Since the budget implies $\sum_k a_kw_k\leq B$, the associated continuous relaxation satisfies
\[
 \max_{\substack{w_k\geq0\ \forall k\\\sum_k a_kw_k\leq B}}
 \min_k w_k
 =\frac{B}{\sum_k a_k}.
\]
Indeed, $\min_k w_k\leq B/\sum_k a_k$, with equality when all $w_k$ are equal. This allocation is realized in the continuous relaxation by assigning query $k$ the budget fraction $a_k/\sum_\ell a_\ell$ and using only a judge minimizing $c_j\sigma_{k,j}^2$. Substitution yields the factor $\sum_k a_k=\gA_\infty^*(\bm\sigma,\vc)$, explaining why this factor appears for every $p\geq2$.

The equalizing allocation minimizes the lower-bound expression obtained from these single-coordinate tests. The tests provide necessary conditions only, so they do not establish the optimal allocation for finite-$p$ estimation. In particular, the allocation minimizing the aggregate objective $\gA_p$ in \Cref{thm:optimal-allocation} assigns query budgets proportional to $a_k^{p/(p+2)}$ for finite $p$.
\end{remark}

\subsection{\texorpdfstring{Proof of \Cref{lem:infimum-hard}: Optimal Hard Instance}{Proof of Lemma C.2: Optimal Hard Instance}}
\label{app:infimum-hard}
For notational simplicity, let $\vx \coloneq \vs - \vs^\star$.
By the symmetry of the objective function and the constraint, we can assume without loss of generality that $x_k \geq 0$ for all $k \in [K]$.
For $p=\infty$, the constraint requires $\max_k|x_k|\geq2\varepsilon_p$, so the minimum is $4\varepsilon_p^2\min_k w_k$, attained at $\vx=2\varepsilon_p\ve_{k^\star}$ for $k^\star\in\argmin_k w_k$. Henceforth assume $p<\infty$.
Then our optimization becomes:
\begin{equation}
    V^\star = \min_{\vx \in \sR_{\geq 0}^K} \left\{ F(\vx) \triangleq \sum_k w_k x_k^2 \right\}, \text{ subj. to } \  \sum_k x_k^p \geq (2\varepsilon_p)^p.
\end{equation}
Scaling any feasible $\vx$ with $\bignorm{\vx}_p>2\varepsilon_p$ down to the boundary strictly decreases $F(\vx)$. Thus, the minimum is attained where $\sum_k x_k^p=(2\varepsilon_p)^p$.

We introduce a (bijective) change of variables $y_k = x_k^p$ for $y_k \geq 0$, which maps the non-convex $L_p$ exterior into a standard linear simplex (a convex set).
Thus, our optimization becomes:
\begin{equation}
    V^\star = \min_{\vy \in \sR_{\geq 0}^K} \left\{ \tilde{F}(\vy) \triangleq \sum_k w_k y_k^{\frac{2}{p}} \right\}, \text{ subj. to } \  \sum_k y_k = (2\varepsilon_p)^p.
\end{equation}
Note that the geometric properties of this optimization, and thus the nature of the optimal solution, depend entirely on the curvature of $\widetilde{F}(\vy)$, which bifurcates at $p = 2$.

\paragraph{Case I: $1 \leq p < 2$.}
In this regime, as the exponent of $\widetilde{F}(\cdot)$ satisfies $\frac{2}{p} > 1$, $\widetilde{F}(\vy)$ is a \emph{strictly convex} function.
Minimizing a strictly convex function subject to a linear equality and non-negativity constraints (which trivially satisfies Slater's condition) guarantees that strong duality holds~\citep[Section 5.2.3]{boyd-vandenberghe}.

We thus introduce Lagrangian multiplier vectors $\bm\lambda \in \sR^K_{\geq 0}$ and $\nu \in \sR$, and the corresponding Lagrangian
\begin{equation}
    \gL(\vy, \bm\lambda, \nu) \coloneq \tilde{F}(\vy) - \bm\lambda^\top \vy + \nu \left( \mathbf{1}^\top \vy - (2\varepsilon_p)^p \right).
\end{equation}
The Karush-Kuhn-Tucker (KKT) stationarity condition with respect to $y_k$ requires:
\begin{equation}
    \frac{\partial \gL}{\partial y_k} = \frac{2}{p} w_k y_k^{\frac{2-p}{p}} - \lambda_k + \nu = 0.
\end{equation}

First, we establish that $\nu$ must be strictly negative. Because the primal constraint requires $\sum_k y_k = (2\varepsilon_p)^p > 0$, there must be at least one index $i$ where $y_i > 0$.
By complementary slackness, the corresponding dual variable is $\lambda_i = 0$, which implies that $\nu = -\frac{2}{p} w_i y_i^{\frac{2-p}{p}}$.
As $w_i > 0$, $y_i > 0$, and for $1 \leq p < 2$ the exponent $\frac{2-p}{p} > 0$, it follows that $\nu < 0$.

Next, we prove that the solution must be strictly interior (i.e., $y_k > 0$ for all $k$).
Assume for contradiction that there exists some index $j$ such that $y_j = 0$.
Because $\frac{2-p}{p} > 0$, the $j$-th stationarity condition becomes
\begin{equation}
    0 - \lambda_j + \nu = 0 \Longrightarrow \lambda_j = \nu < 0,
\end{equation}
a contradiction, i.e., any solution must be strictly interior.
Again, by complementary slackness, we then must have that $\lambda_k = 0$ for all $k$, which implies that
\begin{equation}
    y_k = \left( \frac{-p \nu}{2 w_k} \right)^{\frac{p}{2-p}}, \quad \forall k \in [K].
\end{equation}
Denoting $C \coloneq (- \frac{p \nu}{2})^{\frac{p}{2-p}} > 0$, we can rewrite the solution as $y_k = C \cdot w_k^{-\frac{p}{2-p}}$.
To determine the constant $C$, we apply the primal equality constraint $\sum_k y_k = (2\varepsilon_p)^p$:
\begin{equation}
    \sum_k C \cdot w_k^{-\frac{p}{2-p}} = (2\varepsilon_p)^p \implies C = \frac{(2\varepsilon_p)^p}{\sum_k w_k^{-\frac{p}{2-p}}}.
\end{equation}
Substituting $C$ back into our expression for $y_k$ yields the optimal coordinates in the transformed space:
\begin{equation}
    y_k^\star = (2\varepsilon_p)^p \frac{w_k^{-\frac{p}{2-p}}}{\sum_i w_i^{-\frac{p}{2-p}}}.
\end{equation}
Converting back to our original variables via $x_k^\star = (y_k^\star)^{\frac{1}{p}}$, we recover the dense allocation:
\begin{equation}
    x_k^\star = 2\varepsilon_p \frac{w_k^{-\frac{1}{2-p}}}{\left( \sum_i w_i^{-\frac{p}{2-p}} \right)^{\frac{1}{p}}}.
\end{equation}
Finally, we substitute $x_k^\star$ into the original objective $F(\vx)$ to find the minimum value $V^\star$:
\begin{equation}
    V^\star = \sum_k w_k (x_k^\star)^2 
    = 4 \varepsilon_p^2 \sum_k w_k \frac{w_k^{-\frac{2}{2-p}}}{\left( \sum_i w_i^{-\frac{p}{2-p}} \right)^{\frac{2}{p}}} 
    = 4 \varepsilon_p^2 \frac{\sum_k w_k^{-\frac{p}{2-p}}}{\left( \sum_i w_i^{-\frac{p}{2-p}} \right)^{\frac{2}{p}}}
    = \frac{4 \varepsilon_p^2}{\left( \sum_k w_k^{-\frac{p}{2-p}} \right)^{\frac{2-p}{p}}}.
\end{equation}
Restoring back to $\vs^\star$ as $s^\star_k = s_k' \pm x_k^\star$ yields \textbf{Case I}.

\paragraph{Case II: $p \geq 2$.}
In this regime, as the exponent of $\widetilde{F}(\cdot)$ satisfies $\frac{2}{p} \leq 1$, $\widetilde{F}(\vy)$ is a \emph{concave} function (strictly concave if $p > 2$).
Also note that the domain $\gY$ is a simplex defined as follows:
\begin{equation}
    \gY := \mathrm{conv}\left( \left\{ (2\varepsilon_p)^p \ve_1, \cdots, (2\varepsilon_p)^p \ve_K \right\} \right), \quad \mathrm{Ext}(\gY) = \left\{ (2\varepsilon_p)^p \ve_1, \cdots, (2\varepsilon_p)^p \ve_K \right\}.
\end{equation}
Bauer's maximum principle~\citep{bauer1958maximum} then states that $\min_{\vy \in \gY} \widetilde{F}(\vy)$ is attained at at least one extreme point of $\gY$.

At each vertex $(2\varepsilon_p)^p \ve_k$, we have that
\begin{equation}
    \tilde{F}\left( (2\varepsilon_p)^p \ve_k \right) = w_k \left( (2\varepsilon_p)^p \right)^{\frac{2}{p}} = 4 \varepsilon_p^2 w_k.
\end{equation}
To find the global minimum across all vertices, we simply choose the index that minimizes the RHS, i.e., the minimum is obtained at $k^\star \in \argmin_k w_k$.
Converting back to $\vx$, the optimal solution is the following 1-sparse vector:
\begin{equation}
    \vx^\star = 2\varepsilon_p \ve_{k^\star}, \quad V^\star = 4 \varepsilon_p^2 \min_k w_k.
\end{equation}
Restoring back to $\vs^\star$ as $s^\star_k = s_k' \pm x_k^\star$ yields \textbf{Case II}.
\qed

%% file: 904Proof-Lower-inexp.tex
\section{\texorpdfstring{Proof of \Cref{thm:lower-bound-exp}: Assouad-Type In-Expectation Lower Bound}{Proof of Theorem 5.2: Assouad-Type In-Expectation Lower Bound}}
\label{app:lower-bound-exp}

By Yao's minimax principle~\citep{yao1977}, for any prior $\mu$ over $\left\{ \vs : \bignorm{\vs - \vs^\star}_p \leq \xi(B) \right\}$,
\begin{equation}
    \inf_{\hat{\vs}} \sup_{\vs : \bignorm{\vs - \vs^\star}_p \leq \xi(B)} \E^{\hat{\vs}}_{\vs}\left[ \bignorm{\hat{\vs} - \vs}_p \right] \geq \inf_{\hat{\vs}} \E^{\hat{\vs}}_{\vs \sim \mu}\left[ \bignorm{\hat{\vs} - \vs}_p \right],
\end{equation}
and analogously for $\E^{\hat{\vs}}_{\vs \sim \mu}\left[ \bignorm{\hat{\vs} - \vs}_p^p \right]^{\frac{1}{p}}$ when $1\leq p<\infty$.
We will thus lower bound the RHS (Bayes error) with an appropriate choice of $\mu$.

For both in-expectation lower bounds, we utilize the \textbf{\textit{Assouad's lemma}}~\citep{assouad1983,yu1997lecam}\footnote{This is derived by applying the Le Cam's Two-Point Method~\citep{lecam1973} coordinate-wise to the hypercube: see \citet{yu1997lecam} and \citet[Chapter 2]{tsybakov} for more detailed discussions.} combined with the Bretagnolle-Huber inequality~\citep{bretagnolle-huber}, which we recall here:
\begin{lemma}[Assouad's lemma]
\label{lem:assouad}
    Let $\gF_K = \{-1, 1\}^K$ be the hypercube, $\{\sP_{\vv} : \vv \in \gF_K \}$ be a family of $2^K$ probability measures indexed by $\vv \in \gF_K$, and $w_k > 0$ be some fixed weights.
    Then, for any estimator $\hat{\vv} \in \gF_K$,
    \begin{equation}
        \E_{\vv \sim \mathrm{Unif}(\gF_K)} \left[ \sum_{k \in [K]} w_k \indicator[\hat{v}_k \neq v_k] \right] \ge \frac{1}{4} \sum_{k \in [K]} w_k e^{- \min\left\{ \KL(\bar{\sP}_{+,k}, \bar{\sP}_{-,k}), \KL(\bar{\sP}_{-,k}, \bar{\sP}_{+,k}) \right\}},
    \end{equation}
    where $\bar{\sP}_{i,k} \coloneq \frac{1}{2^{K-1}} \sum_{\vv \in \gF_K : v_k = i} \sP_{\vv}$ for $i \in \{-, +\}$ is the marginal distribution for $v_k = i$.
\end{lemma}
Throughout, we denote $\E_{\vv} = \E_{\vv \sim \mathrm{Unif}(\gF_K)}$.

\paragraph{Construction.}
Suppose that $(\bm\sigma, \vc)$ is given and fixed, and let $\vs^\star \in \sR^K$ be a given score vector.
To apply Assouad's lemma, we define the prior $\mu$ over the score vectors as the uniform distribution over the following hypercube:
\begin{equation}
    \mu = \mathrm{Unif}\left( \left\{ \vs(\vv) \triangleq \vs^\star + \vv \odot \bm\Delta \ : \ \vv \in \{-1, +1\}^K \right\} \right),
\end{equation}
where $\bm{\Delta} \coloneq (\Delta_1, \dots, \Delta_K) \in \sR_{\geq 0}^K$ is chosen a-priori as described for each lower bound.

We utilize the same construction as in \Cref{prop:beta-construction}, with $R_k \coloneq \min\{s_k^\star, R - s_k^\star\}> 0$, $\max_j \sigma_{k,j}^2 \leq R_k^2 / 2$, and $\Delta_k \leq R_k / 4$.
We set $\sP_\vv = \sP_{\vs^\star + \vv \odot \bm\Delta}$.

We, again, consider the canonical bandit model~\citep[Section 4.6]{banditalgorithms} over the observations, as the estimator may be adaptive.
Then, the following lemma, whose proof is provided in \Cref{app:assouad-adaptive}, rigorously establishes the adaptive extension of the Assouad's lemma:
\begin{lemma}
\label{lem:assouad-adaptive}
    With $\sP_\vv$'s as defined and $\bar{N}_{k,j}$ being the number of times $(k, j)$ has been sampled, averaged over $\vv \sim \mathrm{Unif}(\gF_K)$, we have that
    \begin{equation}
        \min\left\{ \KL(\bar{\sP}_{+,k}, \bar{\sP}_{-,k}), \KL(\bar{\sP}_{-,k}, \bar{\sP}_{+,k}) \right\} \leq C_2 \sum_{j \in [J]} \frac{\bar{N}_{k,j} \Delta_k^2}{\sigma_{k,j}^2},
    \end{equation}
    satisfying $\sum_{k,j} \bar{N}_{k,j} c_j \leq B.$
\end{lemma}
From hereon and forth, we will denote $V_k \coloneq \sum_{j \in [J]} \frac{\bar{N}_{k,j}}{\sigma_{k,j}^2},$ which is an algorithm-dependent quantity.

\subsection{\texorpdfstring{Lower Bound for Expected $\ell_p$-Error}{Lower Bound for Expected Lp-Error}}
We first consider $1\leq p<\infty$ and describe our a priori choice of $\bm\Delta$:
\begin{equation}
    \Delta_k = \frac{1}{4 \sqrt{V_k^\star}}, \quad
    V_k^\star \coloneq \frac{B \left( c_{j^*(k)} \sigma_{k,j^*(k)}^2 \right)^{-\frac{2}{p+2}}}{\sum_{k \in [K]} \left( c_{j^*(k)} \sigma_{k,j^*(k)}^2 \right)^{\frac{p}{p+2}}}.
\end{equation}
Then it is easy to check that $\bignorm{\Delta}_p = \frac{1}{4} \sqrt{\frac{\gA_p^*(\bm\sigma, \vc)}{B}} = \xi(B)$.
Also, we have that $\Delta_k \leq \frac{R_k}{4}$ for all $k$'s if
\begin{equation}
    B \geq \left( \sum_{k \in [K]} \left( c_{j^*(k)} \sigma_{k,j^*(k)}^2 \right)^{\frac{p}{p+2}} \right) \max_{k \in [K]} \frac{\left( c_{j^*(k)} \sigma_{k,j^*(k)}^2 \right)^{\frac{2}{p+2}}}{R_k^2},
\end{equation}
as given in the statement.

Here, as $\E_{\vv}\left[ \bignorm{\hat{\vs} - \vs(\vv)}_p \right]$ is not coordinate-wise decomposable for $p > 1$, we cannot directly utilize Assouad's lemma (\Cref{lem:assouad}).\footnote{Keen readers may wonder if Jensen's inequality may do the trick, but then would probably despair after finding out that the direction is not as we want.}
Instead, we utilize a concavity argument as follows.

For an arbitrary estimator $\hat{\vs}$, let $\hat v_k=+1$ if $\hat s_k\geq s_k^\star$ and $\hat v_k=-1$ otherwise.
The following lemma relates the two:
\begin{lemma}
\label{lem:sign-error}
    $\left| \hat{s}_k - (\vs(\vv))_k \right| \geq \Delta_k \indicator[\hat{v}_k \neq v_k]$ for any $k \in [K]$ and $\vv \in \gF_K.$
\end{lemma}
\begin{proof}
    W.l.o.g. suppose that $v_k = +1$.
    If $\hat{v}_k = v_k = +1$, then the bound trivially holds, and thus suppose that $\hat{v}_k = -1 \neq v_k$.
    Then,
    \begin{equation}
        \left| \hat{s}_k - \vs(\vv) \right| = \left| (\hat{s}_k - s_k^\star) - \Delta_k \right| = -(\hat{s}_k - s_k^\star) + \Delta_k \geq \Delta_k.
    \end{equation}
\end{proof}

By \Cref{lem:sign-error}, $\bignorm{\hat{\vs}-\vs(\vv)}_p\geq S^{1/p}$, where $S \coloneq \sum_{k \in [K]} (\Delta_k)^p \indicator[\hat{v}_k \neq v_k]$ satisfies $0\leq S\leq\bignorm{\bm\Delta}_p^p$.
By Assouad's lemma (\Cref{lem:assouad}), its expectation is bounded as follows:
\begin{align}
    \E[S] = \E_{\vv}\left[ \sum_{k \in [K]} (\Delta_k)^p \indicator[\hat{v}_k \neq v_k] \right]
    &\geq \frac{1}{4} \sum_{k \in [K]} (\Delta_k)^p \exp\left( - C_2 \Delta_k^2 V_k \right) \tag{\Cref{lem:assouad,lem:assouad-adaptive}} \\
    &= \frac{1}{4} \bignorm{\bm\Delta}_p^p \underbrace{\sum_{k \in [K]} \frac{(\Delta_k)^p}{\sum_k (\Delta_k)^p} \exp\left( - C_2 \Delta_k^2 V_k \right)}_{(*)}.
\end{align}
Note that from our choice of $\Delta_k$'s,
\begin{equation}
    q_k \coloneq \frac{(\Delta_k)^p}{\sum_k (\Delta_k)^p} = \frac{(V_k^\star)^{-p/2}}{\sum_k (V_k^\star)^{-p/2}}, \quad \Delta_k^2 V_k = \frac{1}{16} \frac{V_k}{V_k^\star}.
\end{equation}
One important observation is that
\begin{align}
    \sum_{k \in [K]} q_k \frac{V_k}{V_k^\star} 
    &= \frac{1}{B} \sum_{k \in [K]} c_{j^*(k)} \sigma_{k,j^*(k)}^2 V_k \tag{Direct computation} \\
    &= \frac{1}{B} \sum_{k \in [K]} \sum_{j \in [J]} \bar{N}_{k,j} \frac{c_{j^*(k)} \sigma_{k,j^*(k)}^2}{\sigma_{k,j}^2} \\
    &\leq \frac{1}{B} \sum_{k \in [K]} \sum_{j \in [J]} \bar{N}_{k,j} c_j
    \leq 1.
\end{align}
By Jensen's inequality, we have that
\begin{equation}
    (*) \geq \exp\left( -\frac{C_2}{16} \sum_{k \in [K]} q_k \frac{V_k}{V_k^\star} \right)
    \geq \exp\left( -\frac{C_2}{16} \right).
\end{equation}

Now, we present the concavity argument.
For $p \geq 1$, the function $x \mapsto x^{\frac{1}{p}}$ is concave on $[0, \infty)$ and $f(0) = 0$, i.e., geometrically,\footnote{For concave functions, any secant line lies below the function.} we have that for any $S \in [0, S_{\max}],$
\begin{equation}
    S^{\frac{1}{p}} \geq \frac{S_{\max}^{\frac{1}{p}} - 0^{\frac{1}{p}}}{S_{\max} - 0} S = S_{\max}^{\frac{1}{p} - 1} S.
\end{equation}

Plugging in $S_{\max} = \bignorm{\bm\Delta}_p^p$, we have that
\begin{equation}
    \E_{\vv}\left[ \bignorm{\hat{\vs} - \vs(\vv)}_p \right] \geq \E\left[ S^{\frac{1}{p}} \right]
    \geq \bignorm{\bm\Delta}_p^{1-p} \E[S]
    \geq \frac{\exp\left( -\frac{C_2}{16} \right)}{4} \bignorm{\bm\Delta}_p = \underbrace{\frac{\exp\left( -\frac{C_2}{16} \right)}{16}}_{\triangleq C_3} \sqrt{\frac{\gA_p^*(\bm\sigma, \vc)}{B}}.
\end{equation}

\paragraph{Constant-Probability Lower Bound.}
For $1\leq p<\infty$, the preceding argument gives $0\leq S\leq S_{\max}=\xi(B)^p$ and $\E[S]\geq4C_3S_{\max}$. For any $0<t<4C_3$, the reverse Markov inequality follows directly from
\begin{equation}
    4C_3S_{\max}\leq\E[S]\leq tS_{\max}+(1-t)S_{\max}\sP(S\geq tS_{\max}),
\end{equation}
and hence $\sP(S\geq tS_{\max})\geq(4C_3-t)/(1-t)$. Setting $t=(4C_3-\delta)/(1-\delta)$ yields the probability bound stated in \Cref{sec:lower-bound}, since $\bignorm{\hat{\vs}-\vs(\vv)}_p\geq S^{1/p}$. This probability averages over the hypercube and observations, so at least one fixed vertex satisfies the same bound for each estimator. The conversion is consistent with the local-risk-to-quantile approach of \citet{ma2024minimax}; it requires no independence between the coordinate errors.

\paragraph{The case $p=\infty$.}
We use the same hypercube construction with equal perturbations
\[
    \Delta_k=\xi(B)=\frac14\sqrt{\frac{\gA_\infty^*(\bm\sigma,\vc)}{B}},
    \qquad k\in[K].
\]
If $B\geq\gA_\infty^*(\bm\sigma,\vc)\max_kR_k^{-2}$, then $\Delta_k\leq R_k/4$ for every $k$, so the beta construction remains valid, and every vertex lies in the $\ell_\infty$-ball of radius $\xi(B)$ around $\vs^\star$.

For this case, take
\[
    q_k:=\frac{c_{j^*(k)}\sigma_{k,j^*(k)}^2}{\gA_\infty^*(\bm\sigma,\vc)},
    \qquad
    S:=\sum_kq_k\indicator[\hat v_k\neq v_k],
\]
using the same sign decoder as above.
Since $\gA_\infty^*(\bm\sigma,\vc)=\sum_kc_{j^*(k)}\sigma_{k,j^*(k)}^2$, we have $\sum_kq_k=1$ and $0\leq S\leq1$.
The budget constraint gives
\[
    \xi(B)^2\sum_kq_kV_k
    =\frac{1}{16B}\sum_kc_{j^*(k)}\sigma_{k,j^*(k)}^2V_k
    \leq\frac1{16}.
\]
Applying \Cref{lem:assouad,lem:assouad-adaptive} and Jensen's inequality, with expectation over the uniform hypercube and observations, yields
\begin{align}
    \E[S]
    &\geq\frac14\sum_kq_k\exp\left(-C_2\xi(B)^2V_k\right)\\
    &\geq\frac14\exp\left(-C_2\xi(B)^2\sum_kq_kV_k\right)
    \geq\frac14\exp\left(-\frac{C_2}{16}\right)=4C_3.
\end{align}
A single incorrect sign forces $\ell_\infty$-error at least $\xi(B)$, so
\[
    S\leq\indicator[\bignorm{\hat{\vs}-\vs(\vv)}_\infty\geq\xi(B)].
\]
Consequently, for every budget-feasible estimator, at least one fixed vertex $\vs=\vs(\vv)$ satisfies
\[
    \sP_{\vs}^{\hat{\vs}}\left(\bignorm{\hat{\vs}-\vs}_\infty\geq\xi(B)\right)\geq4C_3.
\]
At the same vertex,
\[
    \E_{\vs}^{\hat{\vs}}\left[\bignorm{\hat{\vs}-\vs}_\infty\right]
    \geq4C_3\xi(B)
    =C_3\sqrt{\frac{\gA_\infty^*(\bm\sigma,\vc)}{B}}.
\]
This proves the expected-error lower bound for $p=\infty$ and gives its constant-probability counterpart directly.

\subsection{\texorpdfstring{Lower Bound for $p$-th Moment of the Error}{Lower Bound for p-th Moment of the Error}}
For $1\leq p<\infty$, we choose $\bm\Delta$ slightly differently as $\Delta_k = \frac{\sqrt{p}}{4 \sqrt{V_k^\star}}$ instead of $\Delta_k = \frac{1}{4 \sqrt{V_k^\star}}$, with the same $V_k^\star$ as previous.
This satisfies $\bignorm{\bm\Delta}_p = \frac{1}{4} \sqrt{\frac{p \gA_p^*(\bm\sigma, \vc)}{B}} = \sqrt{p} \xi(B)$.
Under the given assumption on $B$, we also have that $\Delta_k \leq \frac{R_k}{4}$.

Let $\hat{\vs}$ be any estimator, with the same sign decoder $\hat v_k$ defined above.
Then the absolute error of $\hat{s}_k$ to $(\vs(\vv))_k$ with the sign-error is closely related as described in \Cref{lem:sign-error}.

As $\E_{\vv}\left[ \bignorm{\hat{\vs} - \vs(\vv)}_p^p \right]$ is coordinate-wise decomposable, applying \Cref{lem:sign-error,lem:assouad} and similar Jensen argument as previous, we have that
\begin{align}
    \E_{\vv}\left[ \bignorm{\hat{\vs} - \vs(\vv)}_p^p \right] \geq \E_{\vv}\left[ \sum_{k \in [K]} (\Delta_k)^p \indicator[\hat{v}_k \neq v_k] \right]
    &\geq \frac{1}{4} \bignorm{\bm\Delta}_p^p \exp\left( -\frac{C_2 p}{16} \right) \\
    &= \frac{\exp\left( -\frac{C_2 p}{16} \right)}{4^{1+p}} \left( \frac{p \gA_p^*(\bm\sigma, \vc)}{B} \right)^{\frac{p}{2}}
\end{align}
Taking the $1/p$-th power on both sides, we have that
\begin{equation}
    \E_{\vv}\left[ \bignorm{\hat{\vs} - \vs(\vv)}_p^p \right]^{1/p} \geq \underbrace{\frac{\exp\left( -\frac{C_2}{16} \right)}{4^{1+\frac{1}{p}}}}_{= 4^{1 - \frac{1}{p}} C_3} \sqrt{\frac{p \gA_p^*(\bm\sigma, \vc)}{B}},
\end{equation}
and we are done.
\qed

\subsection{\texorpdfstring{Proof of \Cref{lem:assouad-adaptive}: Adaptive Extension of Assouad's Lemma}{Proof of Lemma D.2: Adaptive Extension of Assouad's Lemma}}
\label{app:assouad-adaptive}
For $\vu \in \gF_{K-1}$ and $i \in \{-, +\}$, let us denote $\sP_{(i,\vu),k}$ as the probability measure indexed by $\vv$ defined as $v_k = i$ and $\vv_{-k} = \vu$.
Also, let $\sP_{i,k,j}$ be the probability measure for observations of $(k, j)$ when $v_k = i$.

By the joint convexity of KL and the Divergence Decomposition Lemma~\citep[Lemma 15.1]{banditalgorithms}, we have that
\begin{align}
    \KL(\bar{\sP}_{+,k}, \bar{\sP}_{-,k}) &\leq \frac{1}{2^{K-1}} 
    \sum_{\vu \in \gF_{K-1}} \KL(\sP_{(+,\vu),k}, \sP_{(-,\vu),k}) \\
    &= \sum_{j \in [J]} \underbrace{\frac{1}{2^{K-1}} 
    \sum_{\vu \in \gF_{K-1}} \E_{(+,\vu),k}[N_{k,j}]}_{\triangleq \bar{N}_{+,k,j}} \KL(\sP_{+,k,j}, \sP_{-,k,j}) \\
    &\leq C_2 \sum_{j \in [J]}  \bar{N}_{+,k,j} \frac{\Delta_k^2}{\sigma_{k,j}^2} \tag{\Cref{prop:beta-construction}}
\end{align}
and similarly,
\begin{equation}
    \KL(\bar{\sP}_{-,k}, \bar{\sP}_{+,k}) \leq C_2 \sum_{j \in [J]}  \bar{N}_{-,k,j} \frac{\Delta_k^2}{\sigma_{k,j}^2}.
\end{equation}
Thus,
\begin{align}
    \min\left\{ \KL(\bar{\sP}_{+,k}, \bar{\sP}_{-,k}), \KL(\bar{\sP}_{-,k}, \bar{\sP}_{+,k}) \right\} &\leq \frac{1}{2} \left\{ \KL(\bar{\sP}_{+,k}, \bar{\sP}_{-,k}) + \KL(\bar{\sP}_{-,k}, \bar{\sP}_{+,k}) \right\} \\
    &\leq C_2 \Delta_k^2 \sum_{j \in [J]} \frac{\bar{N}_{+,k,j} + \bar{N}_{-,k,j}}{2} \frac{1}{\sigma_{k,j}^2}.
\end{align}
Note that
\begin{equation}
    \frac{\bar{N}_{+,k,j} + \bar{N}_{-,k,j}}{2} = 2^{-K} \sum_{\vv \in \gF_K} \E_\vv[N_{k,j}] = \bar{N}_{k,j},
\end{equation}
which satisfies the budget constraint
\begin{equation}
    \sum_{k,j} \bar{N}_{k,j} c_j = 2^{-K} \sum_{\vv \in \gF_K} \E_\vv\left[ \sum_{k,j} N_{k,j} c_j \right] \leq B.
\end{equation}
\qed

%% file: 905BetaConstruction.tex
\section{\texorpdfstring{Proof of \Cref{prop:beta-construction}: Construction of Beta Distributions}{Proof of Proposition C.1: Construction of Beta Distributions}}
\label{app:beta-construction}

\subsection{Preliminaries and Construction}
\paragraph{Beta Distribution.} For given $\alpha, \beta > 0,$ the probability density of $X \sim \Beta(\alpha, \beta)$ is given as
\begin{equation}
    f_X(x; \alpha, \beta) = \frac{1}{B(\alpha, \beta)} x^{\alpha - 1} (1 - x)^{\beta - 1}, \quad x \in [0, 1],
\end{equation}
where $B(\alpha, \beta) = \frac{\Gamma(\alpha) \Gamma(\beta)}{\Gamma(\alpha + \beta)}$ is the beta function, and $\Gamma(\cdot)$ is the gamma function.
We note that its mean and variance are characterized in closed-form as follows:
\begin{equation}
    \E[X] = \frac{\alpha}{\alpha + \beta}, \quad \Var[X] = \frac{\alpha \beta}{(\alpha + \beta)^2 (\alpha + \beta + 1)}.
\end{equation}

\paragraph{Construction.}
Recall that $R_k = \min\{s_k^\star, R - s_k^\star\}$.
Let $\vs^\star, \bm\Delta \in \sR^K$ be given, with $|\Delta_k| \leq R_k / 4$.
Then, for each $\vs = \vs^\star + \bm\Delta$, we define the probability measure $\sP_{\vs} = (\sP_{\vs,k})_{k \in [K]}$ as follows.
For each $k$, define
\begin{equation}
    q_{k,j} \coloneq \frac{\sigma_{k,j}^2}{4 R_k^2}, \quad d_k(\bm\Delta) \coloneq \frac{\Delta_k}{2 R_k},
\end{equation}
and define
\begin{equation}
    \alpha_k(\bm\Delta) \coloneq \frac{1 - 4 d_k(\bm\Delta)^2 - 4 q_{k,j}}{8 q_{k,j}} (1 + 2d_k(\bm\Delta)), \quad
    \beta_k(\bm\Delta) \coloneq \frac{1 - 4 d_k(\bm\Delta)^2 - 4 q_{k,j}}{8 q_{k,j}} (1 - 2d_k(\bm\Delta)).
\end{equation}
Note that under the assumptions that $|\Delta_k| \leq R_k / 4$ and $\max_j \sigma_{k,j}^2 \leq R_k^2 / 2$, we have that $\alpha_k(\bm\Delta), \beta_k(\bm\Delta) > 0$, and thus, the beta distribution $Y_k(\bm\Delta) \sim \Beta(\alpha_k(\bm\Delta), \beta_k(\bm\Delta))$ is well-defined.
With this, $\sP_{\vs,k}$ is defined as the measure of the random variable $X_k(\bm\Delta) \coloneq s_k^\star - R_k + 2 R_k Y_k(\bm\Delta)$.
One can immediately verify that $X_k(\bm\Delta) \in [s_k^\star - R_k, s_k^\star + R_k] \subseteq [0, R].$

\subsection{\texorpdfstring{Proof of $(i)$: Shifted Mean, Same Variance}{Proof of (i): Shifted Mean, Same Variance}}

\paragraph{Mean.}
By direct computation,
\begin{equation}
    \E[Y_k(\Delta)] = \frac{\alpha_k(\bm\Delta)}{\alpha_k(\bm\Delta) + \beta_k(\bm\Delta)}
    = \frac{1}{1 + \frac{\beta_k(\bm\Delta)}{\alpha_k(\bm\Delta)}}
    = \frac{1}{1 + \frac{1 - \frac{\Delta_k}{R_k}}{1 + \frac{\Delta_k}{R_k}}}
    = \frac{\Delta_k + R_k}{2 R_k},
\end{equation}
and thus,
\begin{equation}
    \E[X_k(\bm\Delta)] = s_k^\star - R_k + 2 R_k \E[Y_k(\Delta)] = s_k^\star + \Delta_k.
\end{equation}

\paragraph{Variance.}
By direct computation,
\begin{align}
    \Var[Y_k(\Delta)] &= \frac{\alpha_k(\bm\Delta) \beta_k(\bm\Delta)}{(\alpha_k(\bm\Delta) + \beta_k(\bm\Delta))^2 (\alpha_k(\bm\Delta) + \beta_k(\bm\Delta) + 1)} \\
    &= \left( \left( 1 + \frac{\beta_k(\bm\Delta)}{\alpha_k(\bm\Delta)} \right) \left( 1 + \frac{\alpha_k(\bm\Delta)}{\beta_k(\bm\Delta)} \right) \left( 1 + \alpha_k(\bm\Delta) + \beta_k(\bm\Delta) \right) \right)^{-1} \\
    &= \left( \left( \frac{2 R_k}{R_k + \Delta_k} \right) \left( \frac{2 R_k}{R_k - \Delta_k} \right) \left( \frac{1 - 4 d_k(\bm\Delta)^2}{4 q_{k,j}} \right) \right)^{-1} \\
    &= \frac{R_k^2 - \Delta_k^2}{4 R_k^2} \frac{\frac{\sigma_{k,j}^2}{R_k^2}}{1 - \frac{\Delta_k^2}{R_k^2}} = \frac{\sigma_{k,j}^2}{4 R_k^2},
\end{align}
and thus,
\begin{equation}
    \Var[X_k(\bm\Delta)] = 4 R_k^2 \Var[Y_k(\Delta)] = \sigma_{k,j}^2.
\end{equation}

\subsection{\texorpdfstring{Proof of $(ii)$: KL of True(Null) and Perturbed Models}{Proof of (ii): KL of True(Null) and Perturbed Models}}

Because the KL divergence is invariant under affine transformations, it suffices to compute the KL divergence between the standard beta distributions $Y_k(\bm\Delta)$ and $Y_k(\mathbf{0})$:
\begin{equation}
    \KL\left( \sP_{\vs^\star+\bm\Delta,k,j}, \sP_{\vs^\star,k,j} \right) = \KL\left( \Beta(\alpha_k(\bm\Delta), \beta_k(\bm\Delta)) \ \| \ \Beta(\alpha_k(\mathbf{0}), \beta_k(\mathbf{0})) \right).
\end{equation}

To simplify notation, let $q \coloneq q_{k,j} = \frac{\sigma_{k,j}^2}{4 R_k^2}$ and $d \coloneq d_k(\bm\Delta) = \frac{\Delta_k}{2 R_k}$.
Recall our structural assumptions: $|\Delta_k| \leq R_k / 4$, which implies $|d| \leq 1/8$, and $\max_j\sigma_{k,j}^2 \leq R_k^2 / 2$, which implies $q \leq 1/8$.
The parameters for the alternate and true models respectively are $\bm\theta(d) \coloneq (\alpha(d), \beta(d))$ and $\bm\theta(0) \coloneq (\alpha(0), \beta(0))$, defined as:
\begin{equation}
    \alpha(d) = \frac{1 - 4 d^2 - 4 q}{8 q} (1 + 2d), \quad \beta(d) = \frac{1 - 4 d^2 - 4 q}{8 q} (1 - 2d).
\end{equation}

\paragraph{Step 1: Uniform Lower Bound on the Parameter Space.}
We first establish that the parameters are strictly bounded away from zero.
Using $|d| \leq 1/8$ and $q \leq 1/8$, we can lower bound the common term:
\begin{equation}
    1 - 4 d^2 - 4 q \geq 1 - 4\left(\frac{1}{64}\right) - 4\left(\frac{1}{8}\right) = 1 - \frac{1}{16} - \frac{1}{2} = \frac{7}{16}.
\end{equation}
Similarly, $(1 \pm 2d) \geq 1 - 2(1/8) = \frac{3}{4}$. Thus, both $\alpha(d)$ and $\beta(d)$ are uniformly lower bounded by:
\begin{equation}
    \alpha(d), \beta(d) \geq \frac{7/16}{8 q} \left(\frac{3}{4}\right) = \frac{21}{512 q} \triangleq \frac{C_{\min}}{q},
\end{equation}
where $C_{\min} = \frac{21}{512}$. This lower bound identically holds for the null parameters $\bm\theta(0)$.

\paragraph{Step 2: Bounding the Parameter Perturbation.}
Next, we bound the squared Euclidean distance between the parameters $\bm\theta(d)$ and $\bm\theta(0)$ as a function of the shift $d$:
\begin{align}
    \alpha(d) - \alpha(0) &= \frac{(1 - 4 d^2 - 4 q)(1 + 2d) - (1 - 4 q)}{8 q} \\
    &= \frac{1 + 2d - 4 d^2 - 8 d^3 - 4 q - 8 d q - 1 + 4 q}{8 q} \\
    &= \frac{2d - 4 d^2 - 8 d^3 - 8 d q}{8 q} \\
    &= \frac{d}{4 q} (1 - 2d - 4 d^2 - 4 q).
\end{align}
Applying the absolute bounds $|d| \leq 1/8$ and $q \leq 1/8$ to the polynomial term yields:
\begin{equation}
    |1 - 2d - 4 d^2 - 4 q| \leq 1 + 2|d| + 4 d^2 + 4 q \leq 1 + \frac{1}{4} + \frac{1}{16} + \frac{1}{2} = \frac{29}{16}.
\end{equation}
Therefore, $|\alpha(d) - \alpha(0)| \leq \frac{29}{64} \frac{|d|}{q}$. By identical algebraic steps, $|\beta(d) - \beta(0)| \leq \frac{29}{64} \frac{|d|}{q}$, and thus, the squared Euclidean distance is thus strictly bounded by:
\begin{equation}
    \bignorm{\bm\theta(d) - \bm\theta(0)}_2^2 \leq 2 \left( \frac{29}{64} \right)^2 \frac{d^2}{q^2} = \frac{841}{2048} \frac{d^2}{q^2}.
\end{equation}

\paragraph{Step 3: Exact KL via Bregman Divergence.}
The Beta distribution forms an exponential family, meaning the KL divergence between two Beta distributions can be written exactly as the Bregman divergence on their log-partition function $m(\alpha, \beta) \triangleq \log B(\alpha, \beta)$~\citep{wainwright-jordan}.
By Taylor's theorem, there exists some intermediate parameter $\tilde{\bm\theta} = (\tilde{\alpha}, \tilde{\beta})$ lying on the line segment between $\bm\theta(d)$ and $\bm\theta(0)$ such that:
\begin{equation}
    \KL(\sP_{\vs^\star+\bm\Delta,k,j}, \sP_{\vs^\star,k,j}) = \frac{1}{2} (\bm\theta(d) - \bm\theta(0))^\top \bm{I}(\tilde{\bm\theta}) (\bm\theta(d) - \bm\theta(0)),
\end{equation}
where $\bm{I}(\tilde{\bm\theta}) = \nabla^2 A(\tilde{\bm\theta})$ is the Fisher Information Matrix (FIM) evaluated at $\tilde{\bm\theta}$~\citep{amari2016}.
The elements of the FIM for the Beta distribution are given by the trigamma function $\psi_1(z) \coloneq \frac{d^2}{dz^2} \log \Gamma(z)$:
\begin{equation}
    \bm{I}(\alpha, \beta) = \begin{pmatrix} \psi_1(\alpha) - \psi_1(\alpha+\beta) & -\psi_1(\alpha+\beta) \\ -\psi_1(\alpha+\beta) & \psi_1(\beta) - \psi_1(\alpha+\beta) \end{pmatrix}.
\end{equation}
For any vector $\vv = (v_1, v_2)^\top$, the quadratic form is $v_1^2 \psi_1(\alpha) + v_2^2 \psi_1(\beta) - (v_1+v_2)^2 \psi_1(\alpha+\beta)$. Since $\psi_1(z) > 0$ for all $z > 0$, the subtracted term is strictly negative, allowing us to upper-bound the quadratic form strictly by its diagonal elements:
\begin{equation}
    \vv^\top \bm{I}(\alpha, \beta) \vv \leq v_1^2 \psi_1(\alpha) + v_2^2 \psi_1(\beta) \leq \bignorm{\vv}_2^2 \max(\psi_1(\alpha), \psi_1(\beta)).
\end{equation}

\paragraph{Step 4: Trigamma Asymptotics and Final Assembly.}
We use the standard uniform upper bound for the trigamma function: $\psi_1(z) \leq \frac{1}{z} + \frac{1}{z^2}$ for $z > 0$.
Because $\tilde{\bm\theta}$ lies on the convex line segment between $\bm\theta(d)$ and $\bm\theta(0)$, and both endpoints are uniformly lower-bounded by $C_{\min}/q$, it must be that $\tilde{\alpha}, \tilde{\beta} \geq C_{\min}/q$. Thus,
\begin{equation}
    \max(\psi_1(\tilde{\alpha}), \psi_1(\tilde{\beta})) \leq \frac{q}{C_{\min}} + \frac{q^2}{C_{\min}^2} = q \left( \frac{1}{C_{\min}} + \frac{q}{C_{\min}^2} \right).
\end{equation}
Since $q \leq 1/8$ and $C_{\min} = \frac{21}{512}$, we can numerically upper bound the term in the parenthesis:
\begin{equation}
    \frac{1}{C_{\min}} + \frac{q}{C_{\min}^2} \leq \frac{512}{21} + \frac{1/8}{(21/512)^2} = \frac{512}{21} + \frac{32768}{441} = \frac{43520}{441} \triangleq C_\psi.
\end{equation}
Substituting this maximum eigenvalue bound back into the Bregman divergence yields:
\begin{align}
    \KL(\sP_{\vs^\star+\bm\Delta,k,j}, \sP_{\vs^\star,k,j}) &\leq \frac{1}{2} \bignorm{\bm\theta(d) - \bm\theta(0)}_2^2 \max(\psi_1(\tilde{\alpha}), \psi_1(\tilde{\beta})) \\
    &\leq \frac{1}{2} \left( \frac{841}{2048} \frac{d^2}{q^2} \right) (C_\psi q) = \frac{841 \cdot 43520}{4096 \cdot 441} \frac{d^2}{q} = \frac{71485}{3528} \frac{d^2}{q}.
\end{align}
Finally, mapping back to the original parameterization $d^2 = \frac{\Delta_k^2}{4 R_k^2}$ and $q = \frac{\sigma_{k,j}^2}{4 R_k^2}$, the $4 R_k^2$ denominators cancel out to yield the desired bound:
\begin{equation}
    \KL(\sP_{\vs^\star+\bm\Delta,k,j}, \sP_{\vs^\star,k,j}) \leq \underbrace{\frac{71485}{3528}}_{\triangleq C_1} \frac{\Delta_k^2}{\sigma_{k,j}^2}.
\end{equation}
\qed

\subsection{\texorpdfstring{Proof of $(iii)$: KL of Two ``Adjacent'' Perturbed Models}{Proof of (iii): KL of Two ``Adjacent'' Perturbed Models}}

Similarly, it suffices to bound
\begin{equation}
    \KL\left( \sP_{\vs^\star+\bm\Delta,k,j}, \sP_{\vs^\star-\bm\Delta,k,j} \right) = \KL\left( \Beta(\alpha_k(d), \beta_k(d)) \ \| \ \Beta(\alpha_k(-d), \beta_k(-d)) \right).
\end{equation}

First, we evaluate the squared Euclidean perturbation between the endpoints $\bm\theta(d)$ and $\bm\theta(-d)$:
\begin{equation}
    \alpha(d) - \alpha(-d) = \frac{1 - 4 d^2 - 4 q}{8 q} \left[ (1 + 2d) - (1 - 2d) \right] = \frac{1 - 4 d^2 - 4 q}{2} \frac{d}{q}.
\end{equation}
Since $d^2 \ge 0$ and $q \ge 0$, we unconditionally have $1 - 4 d^2 - 4 q \leq 1$. Therefore, $|\alpha(d) - \alpha(-d)| \leq \frac{1}{2} \frac{|d|}{q}$. By symmetry, the exact same bound applies to $|\beta(d) - \beta(-d)|$, and thus, we have that
\begin{equation}
    \bignorm{\bm\theta(d) - \bm\theta(-d)}_2^2 \leq 2 \left( \frac{1}{2} \right)^2 \frac{d^2}{q^2} = \frac{1}{2} \frac{d^2}{q^2}.
\end{equation}

Following the exact same Bregman divergence formulation from \textbf{Step 3} above, there exists an intermediate point $\hat{\bm\theta}$ on the line segment connecting $\bm\theta(-d)$ and $\bm\theta(d)$. By the convexity of the parameter space and the symmetry of our construction, the endpoints $\bm\theta(-d)$ and $\bm\theta(d)$ both identically satisfy the lower bound $C_{\min}/q$ derived in \textbf{Step 1}.
Thus, the entire line segment, and consequently $\hat{\bm\theta}$, also satisfies this uniform lower bound.

As a result, the maximum eigenvalue bound of the Fisher Information Matrix remains $\max(\psi_1(\hat{\alpha}), \psi_1(\hat{\beta})) \leq C_\psi q$, utilizing the identical constant $C_\psi = \frac{43520}{441}$.
Substituting this into the Bregman divergence quadratic form yields:
\begin{align}
    \KL(\sP_{\vs^\star+\bm\Delta,k,j}, \sP_{\vs^\star-\bm\Delta,k,j}) &\leq \frac{1}{2} \bignorm{\bm\theta(d) - \bm\theta(-d)}_2^2 \max(\psi_1(\hat{\alpha}), \psi_1(\hat{\beta})) \\
    &\leq \frac{1}{2} \left( \frac{1}{2} \frac{d^2}{q^2} \right) (C_\psi q) = \frac{C_\psi}{4} \frac{d^2}{q} = \frac{10880}{441} \frac{d^2}{q}.
\end{align}
Mapping back to the original parameterization $d^2/q = \Delta_k^2 / \sigma_{k,j}^2$, we obtain the non-asymptotic adjacent pair bound:
\begin{equation}
    \KL(\sP_{\vs^\star+\bm\Delta,k,j}, \sP_{\vs^\star-\bm\Delta,k,j}) \leq \underbrace{\frac{10880}{441}}_{\triangleq C_2} \frac{\Delta_k^2}{\sigma_{k,j}^2}.
\end{equation}
\qed

%% file: 906Gaussian.tex
\section{Gaussian Scores}
\label{app:gaussian}

In this Appendix, we elaborate on how our results generalize to the case when the scores are Gaussian:
\begin{assumption}[Gaussian Scores]\label{assumption:gauss}
    If the learner queries $k\in[K]$ to a judge $j\in[J]$, she observes a noisy score $\tilde{s} \sim s_{k} + \varepsilon_{k, j}$, where $\varepsilon_{k,j}\sim\gN(0 , \sigma_{k,j}^2)$.
\end{assumption}
This then allows us to utilize the following standard Gaussian tail bound:
\begin{lemma}
\label{lem:gaussian}
    For $Z \sim \gN(0, \sigma^2)$, $\sP\left( |Z| \geq t \right) \leq 2 \exp\left( - t^2 / (2 \sigma^2) \right).$
\end{lemma}

For clarity, we recall the inverse-variance weighted estimator:
\begin{tcolorbox}[algobox,title=Inverse-Variance Weighted Estimator (IVWE)]
\textbf{Input:} An \textbf{\emph{allocation}} $(N_{k,j})_{k,j} \in \sN_0^{K \times J}$, variance profile $\bm\sigma = (\sigma_{k,j})_{k,j} \in \sR_{> 0}^{K \times J}$
\begin{enumerate}
    \item For each $(k, j)$, compute $\hat{s}_{k,j} \coloneq \frac{1}{N_{k,j}} \sum_{i=1}^{N_{k,j}} {s}_{k,j}^{(i)}$, if $N_{k,j} \geq 1$; else, set $\hat{s}_{k,j} = 0$;
    \item Aggregate the estimators per $k$ as follows:
    \begin{equation}
        \hat{s}_k \coloneq \left( \sum_{j=1}^J \frac{N_{k,j}}{\sigma_{k,j}^2} \right)^{-1} \sum_{j=1}^J \frac{N_{k,j} \hat{s}_{k,j}}{\sigma_{k,j}^2}.
    \end{equation}
\end{enumerate}
\textbf{Output:} $\hat{\vs} \coloneq (\hat{s}_k)_{k \in [K]}$
\end{tcolorbox}

\subsection{Upper Bounds}
\paragraph{Known Variances.}
First, when the variances are known, we have the following error bound:
\begin{proposition}
\label{prop:upper-bound-gaussian}
    Let $p \in [1, \infty]$, $B \in \sR_{> 0}$ be a total budget, and $\delta \in (0, 1)$ be a target confidence level.
    Also, let $(N_{k,j})_{k,j}$ be a fixed integer allocation satisfying $N_{k,j}\geq0$, $\sum_{k,j}c_jN_{k,j}\leq B$, and $\sum_jN_{k,j}\geq1$ for every $k$.
    Then, \textbf{IVWE} is $(B, \delta, \ell_p)$-budget efficient at any $(\vs, \bm\sigma, \vc) \in \sR^K \times \sR_{>0}^{K\times J} \times \sR_{>0}^J$ with the following error rate $\varepsilon_p$:
    \begin{equation}
        \varepsilon_p = \sqrt{\frac{2 \gA_p(\bm\omega; \bm\sigma, \vc) \log\frac{2K}{\delta}}{B}},
    \end{equation}
    where $\bm\omega \coloneq \left( \omega_{k,j} \triangleq \frac{c_j N_{k,j}}{B} \right)_{k,j}$, the \textbf{allocation objective} $\gA_p(\bm\omega; \bm\sigma, \vc)$ is defined the same as in \Cref{thm:optimal-allocation}.
    Then, with the optimal allocation $\bm\omega^*$ that maximizes $\gA_p(\bm\omega; \bm\sigma, \vc)$, we have $\varepsilon_p = \sqrt{\frac{2 \gA_p^*(\bm\sigma, \vc) \log\frac{2K}{\delta}}{B}}.$
\end{proposition}
\begin{proof}[Proof Sketch]
    This follows from applying \Cref{lem:gaussian} and taking a union bound over $k \in [K]$.
\end{proof}

\paragraph{A sharper confidence bound for $p=2$.}
For a fixed allocation with at least one observation per query, the Gaussian structure gives a simple refinement without a union bound over queries. Write $v_k\coloneq(\sum_j N_{k,j}/\sigma_{k,j}^2)^{-1}$ for the variance of $\hat s_k$. Since the coordinate errors are independent Gaussians, the weighted $\chi^2$ bound of \citet[Lemma~1]{laurent-massart} gives, with probability at least $1-\delta$,
\begin{align}
    \bignorm{\hat{\vs}-\vs}_2^2
    &\leq \sum_k v_k+2\sqrt{\log\frac1\delta\sum_k v_k^2}+2\log\frac1\delta\max_k v_k,\\
    &\leq \left( \sqrt{\sum_k v_k}+\sqrt{2\max_k v_k\log\frac1\delta} \right)^2.\label{eq:gaussian-l2-confidence}
\end{align}
The second inequality follows from $\sum_k v_k^2\leq(\max_k v_k)\sum_k v_k$. This separates the overall estimation error from the additional confidence requirement. With one unit-cost judge, identical variances $\sigma^2$, and $N_k=B/K$, it gives the example in \Cref{sec:lower-bound}. More generally, the allocation minimizing $\sum_k v_k$ need not minimize $\max_k v_k$; this refinement alone therefore does not settle the optimal joint dependence on allocation and confidence.

Readers familiar with Gaussian process concentration may recognize this two-term structure from the classical Tsirelson--Ibragimov--Sudakov (TIS) inequality~\citep{tis,katz-samuels2020linear}: an expected-supremum (Gaussian width) term and a dimension-free deviation term governed by the largest variance.
Indeed, applying TIS to the process $\langle u,\hat{\vs}-\vs\rangle$ indexed by $\|u\|_2\leq1$ also yields the second bound in Eqn.~\eqref{eq:gaussian-l2-confidence}, since its expected supremum is $\E[\bignorm{\hat{\vs}-\vs}_2]\leq\sqrt{\sum_k v_k}$ and its largest variance is $\max_k v_k$.

\paragraph{Unknown Variances.}
Here, as we have briefly elaborated in \Cref{sec:unknown}, we can utilize the following multiplicative concentration for the sample variance of a Gaussian:
\begin{lemma}[Eqn.~(4.3) \& (4.4) of \citet{laurent-massart}]
\label{lem:gaussian-variance}
    With $\hat{\sigma}_{k,j}^2$ as in line 3 of Algorithm~\ref{alg:estimate-then-ivwe-gaussian} and $\epsilon_0 \triangleq \log \frac{2KJ}{\delta}$, suppose that $N_0$ is an integer satisfying $N_0-1>4\epsilon_0$, so that the denominators below are positive. Then
    \begin{equation}
        \sP\left( \frac{\hat{\sigma}_{k,j}^2}{1 +2\sqrt{\frac{\epsilon_0 }{N_0-1}} + \frac{2\epsilon_0 }{N_0-1}} \leq \sigma_{k,j}^2 \leq \frac{\hat{\sigma}_{k,j}^2}{1 - 2\sqrt{\frac{\epsilon_0}{N_0-1}}} , \quad \forall (k, j) \in [K] \times [J] \right) \geq 1 - \delta.
    \end{equation}
\end{lemma}

\begin{algorithm2e}[H]
    \SetKwInput{Input}{Input}
    \SetKwComment{Comment}{$\triangleright$\ }{}

    \Input{Total budget $B > 0$, Cost vector $\vc \in \mathbb{R}^J_+$, Number of forced exploration per pair $N_0\geq2$}

    \BlankLine
    \tcp{Phase I: Forced Exploration}
    \For{each $(k, j) \in [K] \times [J]$}{
        Pull arm $(k,j)$ exactly $N_0$ times and observe noisy evaluation scores $\left\{ s_{k,j}^{(n)} \right\}_{n \in [N_0]}$ \;

        Compute the empirical mean and variance estimator
        \begin{eqnarray}
            \hat{s}_{k,j} \coloneq \frac{1}{N_0} \sum_{n=1}^{N_0} s_{k,j}^{(n)}, \quad \hat{\sigma}_{k,j}^2 \coloneq \frac{1}{N_0 - 1} \sum_{n=1}^{N_0} \left( {s}_{k,j}^{(n)} - \hat{s}_{k,j} \right)^2
        \end{eqnarray}
    }

    \tcp{Phase II: IVWE}
    Compute the optimal allocation $\left( \widehat{\omega}_{k,j}^* \right)_{k,j}$ as in \Cref{thm:optimal-allocation} using remaining budget $B' \triangleq B - N_0 K \sum_{j \in [J]} c_j - K (\max_{j\in[J]} c_{j}) >0$, based on $\left\{ \left( \hat{\sigma}_{k,j} \right)^2 \right\}_{k,j}$
    Set $\widehat N_{k,j}^* = \left\lceil B' \widehat \omega_{k,j}^* / c_j \right\rceil$, and query each pair $(k,j)$ exactly $\widehat N_{k,j}^*$ times\;
    
    \Return \textbf{IVWE} with $\left( \widehat{N}_{k,j}^* \right)_{k,j}$ as allocation and $\hat{\sigma}_{k,j}^2$ as variance proxies\;
    \caption{\label{alg:estimate-then-ivwe-gaussian}Estimate-then-IVWE (Gaussian) }
\end{algorithm2e}

Then we have the following error rate, whose proof is provided in \Cref{app:gaussian-unknown}:
\begin{theorem}[Error Rates of \textbf{Est-IVWE-Gaussian}]
\label{thm:est-ivwe-Gaussian}
    Let $p \in [1, \infty]$, $B \in \sR_{> 0}$, and $\delta \in (0, 1)$.
    Let us set $N_0 = 1 +\left\lceil 16 \log \frac{4KJ}{\delta} \right\rceil$ and suppose that $B > N_0 K \sum_{j\in[J]} c_j+Kc_{\max}$.
    Then, \textbf{Est-IVWE-Gaussian} is $(B, \delta, \ell_p)$-budget efficient at any $(\vs, \bm\sigma, \vc)$ with the following error rate $\varepsilon_p$:
    \begin{equation}
            \varepsilon_p = \sqrt{\frac{13\gA_p^*(\bm\sigma, \vc) \log \frac{4K}{\delta}}{2\left(B-N_0 K\sum_jc_j-Kc_{\max}\right)}}.
        \end{equation}
    Additionally, if $B\geq2\left(N_0K\sum_jc_j+Kc_{\max}\right)$, this rate reduces to
    \begin{equation}
        \varepsilon_p\leq\sqrt{\frac{13\gA_p^*(\bm\sigma,\vc)\log\frac{4K}{\delta}}{2B}}
        +\gO\left(\frac{K\sum_jc_j\,\sqrt{\gA_p^*(\bm\sigma,\vc)}\left(\log\frac{4KJ}{\delta}\right)^{3/2}}{B^{3/2}}\right).
    \end{equation}
\end{theorem}

\subsection{Minimax Lower Bounds}
We first have the following very similar-looking high-probability lower bound:
\begin{theorem}[Local \& High Probability Lower Bound (Gaussian)]
\label{thm:lower-bound-whp-gaussian}
    Let $p \in [1, \infty]$, $B > 0$, $\delta \in (0, e^{-2}]$, and $(\vs, \bm\sigma, \vc)$ be a given problem instance.
    Then, any algorithm that is $(B, \delta, \ell_p)$-budget efficient with an error rate of $\varepsilon_p > 0$ at any $(\vs', \bm\sigma, \vc)$ with $\bignorm{\vs' - \vs}_p \leq 2\varepsilon_p$ must satisfy the following: for some absolute constants $C_1, C_2 > 0$,
    \begin{equation}
        \varepsilon_p \geq \sqrt{\frac{\kl(1 - \delta, \delta)}{C_1 B} {\gA_{\frac{2p}{2-p}}^*(\bm\sigma, \vc)}} \ \ \ \text{for $1 \leq p < 2$}, \quad 
        \varepsilon_p \geq \sqrt{\frac{ \log\frac{1}{\delta}}{C_2 B} {\gA^*_{\infty}(\bm\sigma, \vc)}} \ \ \ \text{for $p \geq 2$}.
    \end{equation}
\end{theorem}
\begin{proof}
    The proof is almost the same as that of \Cref{thm:lower-bound-whp}, except we now construct the instances via Gaussians instead of the Beta distributions, resulting in constant-wise difference only.
\end{proof}

As in the bounded setting, this Fano-type change-of-measure argument captures the $\log\frac{1}{\delta}$ dependence.
We complement it with global in-expectation lower bounds that recover $\gA_p^*$, using different proof techniques that may be of independent interest.\footnote{One may wonder whether it is possible to utilize Assouad's lemma ~\citep{assouad1983,yu1997lecam} as done in our \Cref{thm:lower-bound-exp}. This is definitely possible that leads to a similar local minimax lower bound. Here, we intend to show a completely different argument for global minimax lower bounds.}
Again, the $\inf_{\hat{\vs}}$ is over all adaptive algorithms satisfying the budget constraint. In particular, the lower bound holds even when the learner is given the true variance profile \(\bm\sigma\).

\begin{theorem}[Global, In-Expectation Minimax Lower Bounds]
\label{thm:gaussian-global-exp-lower}
    For each $1 \leq p \leq \infty$,
    \begin{equation}
        \inf_{\hat{\vs}} \sup_{\vs \in \sR^K} \E^{\hat{\vs}}_{\vs \sim \mu}\left[ \bignorm{\hat{\vs} - \vs}_p \right] \geq \sqrt{\frac{2}{\pi}} \sqrt{\frac{\gA_p^*(\bm\sigma, \vc)}{B}}.
    \end{equation}
    For $1 \leq p < \infty$, we also have:
    \begin{equation}
        \inf_{\hat{\vs}} \sup_{\vs \in \sR^K} \E^{\hat{\vs}}_{\vs \sim \mu}\left[ \bignorm{\hat{\vs} - \vs}_p^p \right]^{\frac{1}{p}} \geq m_p \sqrt{\frac{\gA_p^*(\bm\sigma, \vc)}{B}},
    \end{equation}
    where\footnote{$\Gamma(\cdot)$ is the usual gamma function.} $m_p \coloneq \sqrt{2} \left( \frac{\Gamma((p+1)/2)}{\sqrt{\pi}} \right)^{\frac{1}{p}} = \E[|Z|^p]^{1/p}$ for $Z \sim \gN(0, 1)$.
\end{theorem}
\begin{proof}[Proof Sketch]
    The full proof is deferred to \Cref{app:lower-bound-exp-gaussian}.
    Instead of following the standard lower bound framework~\citep{yu1997lecam,tsybakov}, we heavily exploit the fact that everything is Gaussian.
    Specifically, this implies that the posterior risk is minimized at the posterior mean, and everything (e.g., posterior distribution, Fisher information) is computable in closed forms.
\end{proof}

When the \emph{variances are known} and the continuous oracle allocation has integer sample counts, the lower bound for the $p$-th moment is constant-wise tight for every finite $p$, as shown by the following matching upper bound:
\begin{theorem}
\label{thm:in-exp-upper-bound-known}
    Let $1\leq p<\infty$. When the variances are \emph{known} and the continuous oracle allocation at budget $B$ has integer sample counts, \textbf{IVWE} with optimal allocation achieves the following:
    \begin{equation}
        \E\left[ \bignorm{\hat{\vs} - \vs}_p^p \right]^{\frac{1}{p}} \leq m_p \sqrt{\frac{\gA_p^*(\bm\sigma, \vc)}{B}}.
    \end{equation}
\end{theorem}
\begin{proof}
    Recall that $\hat{s}_k - s_k \sim \gN\left( 0, \left( \sum_{j=1}^J \frac{N_{k,j}}{\sigma_{k,j}^2} \right)^{-1} \right)$.
    Thus, we have that
    \begin{equation}
        \E\left[ \bignorm{\hat{\vs} - \vs}_p^p \right]^{\frac{1}{p}} = \mathbb{E} \left[ \sum_{k\in[K]} |\hat{s}_k - s_k|^p \right]^{\frac{1}{p}} \leq m_p \left( \sum_{k \in [K]} \left( \sum_{j=1}^J \frac{N_{k,j}}{\sigma_{k,j}^2} \right)^{-\frac{p}{2}} \right)^{\frac{1}{p}},
    \end{equation}
    where $m_p = \E[|Z|^p]^{1/p}$ for $Z \sim \gN(0, 1)$ as in \Cref{thm:gaussian-global-exp-lower}.
    Then, we conclude by \Cref{thm:optimal-allocation}.
\end{proof}

We lastly show another global lower bound for $\E^{\hat{\vs}}_{\vs \sim \mu}\left[ \bignorm{\hat{\vs} - \vs}_p^p \right]^{\frac{1}{p}}$ based on a completely different (and from our perspective, slightly less well-known) technique that may be of an independent interest:
\begin{theorem}
\label{thm:lower-bound-van-trees}
    For $1 \leq p < \infty$, we have:
    \begin{equation}
        \inf_{\hat{\vs}} \sup_{\vs \in \sR^K} \E^{\hat{\vs}}_{\vs \sim \mu}\left[ \bignorm{\hat{\vs} - \vs}_p^p \right]^{\frac{1}{p}} \geq \frac{\sqrt{2 \pi e}}{C_{\mathrm{ME}}(p)}  \sqrt{\frac{\gA_p^*(\bm\sigma, \vc)}{B}},
    \end{equation}
    where $C_{\mathrm{ME}}(p) \triangleq 2 e^{\frac{1}{p}} \Gamma(p^{-1}) p^{\frac{1}{p} - 1}$ is the partition function of the $L_p$ max-entropy distribution.
\end{theorem}
\begin{proof}[Proof Sketch]
    The full proof is deferred to \Cref{app:van-trees}.
    The key technical tool is the \textbf{\textit{van Trees inequality}}~\citep{vanTrees,gill-levit}, recently extended to $L_p$-norm by \citet{chen2024vantrees}, with cosine-squared prior~\citep{tsybakov} and taking a limit to uninformative prior.
\end{proof}

Note that when $p = 2$, this bound is also constant-wise optimal: $m_2 = 1 = \left( \frac{\sqrt{2 \pi e}}{C_{\mathrm{ME}}(2)} \right)^2$.

%% file: 907Proof-Gaussian.tex
\section{\texorpdfstring{Proofs of Theorems in \Cref{app:gaussian} (Gaussian Scores)}{Proofs of Theorems in Appendix F (Gaussian Scores)}}\label{app:gaussian-proof}

\subsection{\texorpdfstring{Proof of \Cref{thm:est-ivwe-Gaussian}: Error Rate of Est-IVWE-Gaussian}{Proof of Theorem F.4: Error Rate of Est-IVWE-Gaussian}}
\label{app:gaussian-unknown}
Denote the lower and upper bounds on variances $\sigma_{k,j}^2$ after the exploration phase by
\begin{equation}
    \underline{\sigma}_{k,j}^2 = \frac{\hat{\sigma}_{k,j}^2}{1+2\sqrt{\frac{\epsilon_0}{N_0-1}} +\frac{2\epsilon_0}{N_0-1}}
    \quad \text{and} \quad
    \overline{\sigma}_{k,j}^2 = \frac{\hat{\sigma}_{k,j}^2}{1 - 2\sqrt{\frac{\epsilon_0}{N_0-1}}},
\end{equation}
where $\epsilon_0 = \log\frac{4 K J}{\delta}$.
By \Cref{lem:gaussian-variance}, we know that
\begin{equation}
    \sP\left( \underline{\sigma}_{k,j}^2 \leq \sigma_{k,j}^2 \leq \overline{\sigma}_{k,j}^2,\quad\forall(k,j)\in[K]\times[J] \right) \geq 1 - \frac{\delta}{2},
\end{equation}
which we will assume to hold throughout the proof.
We denote the ratio between the upper and lower bounds as $\kappa = \overline{\sigma}_{k,j}^2 / \underline{\sigma}_{k,j}^2$, which is the same across all pairs $(k, j)$.
Since $N_0 = 1 + \left\lceil 16 \log \frac{4KJ}{\delta} \right\rceil$, $\epsilon_0 / (N_0-1) \leq \frac{1}{16}$, we have that $\kappa \leq \frac{13}{4}$.

Recall that \textbf{Est-IVWE-Gaussian} (\Cref{alg:estimate-then-ivwe-gaussian}) designates an (predicted) optimal judge $\hat{j}^*(k) \triangleq \argmin_{j \in [J]} c_j \hat{\sigma}_{k,j}^2$ for each query $k \in [K]$, then computes the optimal allocation with variance estimator $\hat{\sigma}_{k,j}^2$ and available continuous budget $B' = B - C N_0-Kc_{\max}>0$ with $C \coloneq K \sum_{j\in[J]} c_j$ as follows:
\begin{equation}
    \widehat{\omega}_{k,\hat{j}^*(k)}^* = \frac{\left(c_{\hat{j}^*(k)} \hat{\sigma}_{k, \hat{j}^*(k)}^2 \right)^{\frac{p}{p+2}}}{\sum_{k'\in[K]} \left( c_{\hat{j}^*(k')} \hat{\sigma}_{k', \hat{j}^*(k')}^2 \right)^{\frac{p}{p+2}}}, \quad
    \widehat{N}_{k,j}^* = \left\lceil\frac{B' \widehat{\omega}_{k,j}^*}{c_j}\right\rceil.
\end{equation}

There is one positive allocation per query, so rounding costs at most $\sum_kc_{\hat j^*(k)}\leq Kc_{\max}$; together with exploration, the total cost is at most $B$. Conditional on Phase I, the allocation and selected judges are fixed, and the fresh Phase-II sample mean $\hat s_k$ is Gaussian with variance $\sigma_{k,\hat j^*(k)}^2/\widehat N_{k,\hat j^*(k)}^*$. Thus, by \Cref{lem:gaussian} and a union bound, the following event has conditional probability at least $1-\delta/2$:
\begin{equation}
    \sP\left( | \hat{s}_k - s_k | \leq \sqrt{\frac{2 \sigma_{k,\hat{j}^*(k)}^2 \log \frac{4K}{\delta}}{\widehat{N}_{k,\hat{j}^*(k)}^*}}, \quad \forall k \in [K] \right) \geq 1 - \frac{\delta}{2}.
\end{equation}
Thus, by summing up $p$-th power of these terms over $k\in[K]$ and substituting the definition of $\widehat{\omega}_{k,j}^*$,the following holds with probability at least $1 - \delta$:
\begin{align}
    \bignorm{\hat{\vs}-\vs}_p & \leq \sqrt{2\log \frac{4K}{\delta}} \left( \sum_{k\in[K]} \left( \frac{\widehat{N}_{k,\hat{j}^*(k)}^*}{\sigma_{k,\hat{j}^*(k)}^2} \right)^{-\frac{p}{2}}\right)^{\frac{1}{p}}\\
    & \leq \sqrt{\frac{2\log \frac{4K}{\delta}}{B'}} \left( \sum_{k\in[K]} \left( \frac{\widehat{\omega}_{k,\hat{j}^*(k)}}{c_{\hat{j}^*(k)} \sigma_{k,\hat{j}^*(k)}^2} \right)^{-\frac{p}{2}}\right)^{\frac{1}{p}}\\
    & = \sqrt{\frac{2\log \frac{4K}{\delta}}{B'}} \left( \sum_{k\in[K]} \left( \frac{1}{c_{\hat{j}^*(k)} \sigma_{k,\hat{j}^*(k)}^2}\cdot \frac{\left(c_{\hat{j}^*(k)} \hat{\sigma}_{k, \hat{j}^*(k)}^2 \right)^{\frac{p}{p+2}}}{\sum_{k'\in[K]} \left( c_{\hat{j}^*(k')} \hat{\sigma}_{k', \hat{j}^*(k')}^2 \right)^{\frac{p}{p+2}}} \right)^{-\frac{p}{2}}\right)^{\frac{1}{p}}\\
    & = \sqrt{\frac{2\log\frac{4K}{\delta}}{B'}}\sqrt{\sum_{k\in[K]} \left( c_{\hat{j}^*(k)} \hat{\sigma}_{k,\hat{j}^*(k)}^2 \right)^{\frac{p}{p+2}}} \left( \sum_{k\in[K]} \left(\frac{\sigma_{k,\hat{j}^*(k)}}{\hat{\sigma}_{k,\hat{j}^*(k)}}\right)^p \left( c_{\hat{j}^*(k)} \hat{\sigma}_{k,\hat{j}^*(k)}^2 \right)^{\frac{p}{p+2}} \right)^{\frac{1}{p}}\\
    & \leq \sqrt{\frac{1}{1 - 2\sqrt{\frac{\epsilon_0}{N_0-1}}}} \sqrt{\frac{2\log\frac{4K}{\delta}}{B'}} \left( \sum_{k\in[K]} \left( c_{\hat{j}^*(k)} \hat{\sigma}_{k,\hat{j}^*(k)}^2 \right)^{\frac{p}{p+2}} \right)^{\frac{p+2}{2p}},
\end{align}
where the last inequality holds since $\sigma_{k,j}^2 / \hat{\sigma}_{k,j}^2 \leq \left(1 - 2\sqrt{\frac{\epsilon_0}{N_0-1}}\right)^{-1} $ for any $(k,j) \in [K] \times [J]$ by \Cref{lem:gaussian-variance}.
When $p = \infty$, we set $p/(p+2)=1$ and $(p+2)/p=1$ by convention, and the above inequality holds in that case as well:
\begin{align}
    \bignorm{\hat{\vs}-\vs}_\infty & \leq \max_{k\in[K]} \sqrt{2\log \frac{4K}{\delta} \frac{\sigma_{k,\hat{j}^*(k)}^2}{\widehat{N}_{k,\hat{j}^*(k)}^*}} \\
    & \leq \max_{k\in[K]} \sqrt{\frac{2\log \frac{4K}{\delta}}{B'} \frac{c_{\hat{j}^*(k)}\sigma_{k,\hat{j}^*(k)}^2}{\widehat{\omega}_{k,\hat{j}^*(k)}^*}}\\
    & = \sqrt{\frac{2\log \frac{4K}{\delta}}{B'}}\left( \max_{k\in[K]} \sqrt{\frac{\sigma_{k,\hat{j}^*(k)}^2}{\hat{\sigma}_{k,\hat{j}^*(k)}^2} \sum_{k'\in[K]} c_{\hat{j}^*(k')} \hat{\sigma}_{k', \hat{j}^*(k')}^2 }\right)\\
    & \leq \sqrt{\frac{1}{1 - 2\sqrt{\frac{\epsilon_0}{N_0-1}}}} \sqrt{\frac{2\log\frac{4K}{\delta}}{B'}}\sqrt{\sum_{k\in[K]} c_{\hat{j}^*(k)} \hat{\sigma}_{k, \hat{j}^*(k)}^2 }.
\end{align}

As $c_{\hat{j}^*(k)} \hat{\sigma}_{k, \hat{j}^*(k)}^2 \leq c_{j^*(k)} \hat{\sigma}_{k,j^*(k)}^2$ by the definition of $\hat{j}^*(k)$, we have
\begin{align}
    &\left( \sum_{k\in[K]} \left( c_{\hat{j}(k)} \hat{\sigma}_{k,\hat{j}(k)}^2 \right)^{\frac{p}{p+2}} \right)^{\frac{p+2}{2p}} \\
    &\leq \left( \sum_{k\in[K]} \left( c_{j^*(k)} \hat{\sigma}_{k,j^*(k)}^2 \right)^{\frac{p}{p+2}} \right)^{\frac{p+2}{2p}}\\
    & \leq \sqrt{1 + 2\sqrt{\frac{\epsilon_0}{N_0-1}} + \frac{2\epsilon_0}{N_0-1}} \left( \sum_{k\in[K]} \left( c_{j^*(k)} \sigma_{k,j^*(k)}^2 \right)^{\frac{p}{p+2}} \right)^{\frac{p+2}{2p}}. \tag{\Cref{lem:gaussian-variance}}
\end{align}
Combining the two events, with probability at least $1-\delta$ we have
\begin{equation}
    \bignorm{\hat{\vs}-\vs}_p
    \leq \sqrt{\frac{2\kappa\gA_p^*(\bm\sigma,\vc)\log\frac{4K}{\delta}}{B'}}
    \leq \sqrt{\frac{13\gA_p^*(\bm\sigma,\vc)\log\frac{4K}{\delta}}{2(B-CN_0-Kc_{\max})}},
\end{equation}
where $\kappa\leq13/4$. This proves the exact bound for every $B>CN_0+Kc_{\max}$.
If $B\geq2(CN_0+Kc_{\max})$, we may additionally apply \Cref{lem:algebraic-inequality} to obtain
\begin{align}
    \varepsilon_p
    &\leq\sqrt{\frac{13\gA_p^*(\bm\sigma,\vc)\log\frac{4K}{\delta}}{2B}}
    +\frac{(CN_0+Kc_{\max})\sqrt{\frac{13}{2}\gA_p^*(\bm\sigma,\vc)\log\frac{4K}{\delta}}}{B^{3/2}}\\
    &=\sqrt{\frac{13\gA_p^*(\bm\sigma,\vc)\log\frac{4K}{\delta}}{2B}}
    +\gO\!\left(\frac{C\sqrt{\gA_p^*(\bm\sigma,\vc)}\left(\log\frac{4KJ}{\delta}\right)^{3/2}}{B^{3/2}}\right).
\end{align}
Here $Kc_{\max}\leq C$ and $N_0=\gO(\log(4KJ/\delta))$, with universal constants. In particular, the expansion requires the stronger budget condition, while the unexpanded bound requires only a positive Phase-II budget.

\subsection{\texorpdfstring{Proof of \Cref{thm:gaussian-global-exp-lower}: Global In-Expectation Lower Bounds for Gaussian Scores}{Proof of Theorem F.6: Global In-Expectation Lower Bounds for Gaussian Scores}}
\label{app:lower-bound-exp-gaussian}

Again, we start with the Yao's minimax principle~\citep{yao1977}, but without the locality: for any prior $\mu$ over $\sR^K$ (note that now, we consider the prior over the entire $\sR^K$ instead of near some given $\vs^\star$),
\begin{equation}
    \inf_{\hat{\vs}} \sup_{\vs \in \sR^K} \E^{\hat{\vs}}_{\vs}\left[ \bignorm{\hat{\vs} - \vs}_p \right] \geq \inf_{\hat{\vs}} \E^{\hat{\vs}}_{\vs \sim \mu}\left[ \bignorm{\hat{\vs} - \vs}_p \right],
\end{equation}
and analogously for $\E^{\hat{\vs}}_{\vs \sim \mu}\left[ \bignorm{\hat{\vs} - \vs}_p^p \right]^{\frac{1}{p}}$.
We will thus lower bound the RHS (Bayes error) with an appropriate choice of $\mu$.

In particular, we consider an isotropic Gaussian prior: for a $\lambda > 0$,
\begin{equation}
    \mu \coloneq \gN(\vzero, \lambda^2 \mI_K).
\end{equation}

Let us fix an arbitrary, adaptive algorithm $\gA$ that interacts with the true environment $\vs^\star$ sequentially as follows: at each timestep $t$, when the learner selects query $k_t$ and judge $j_t$, they observe a noisy score $\tilde{s}_t \coloneq s_{k_t} + \varepsilon_{k_t,j_t}$.
Let us denote $N_{k,j} \coloneq \sum_t \indicator[k_t = k, j_t = j]$ as the number of times $\gA$ queries the pair $(k, j)$ throughout the interaction.
Suppose that $\gA$ interacts for $T$ rounds, while satisfying the budget constraint almost surely.
For each time $t$, let $\gH_t \coloneq (k_1, j_1, \tilde{s}_1, \cdots, k_t, j_t, \tilde{s}_t)$ be the interaction history til time $t$.
Then, at the end of the interaction, $\gA$ outputs an estimator $\hat{\vs} = \hat{\vs}(\gH_T)$.
For clarity, with a slight abuse of notation, we will distinguish between the two: $\gA$ as an adaptive data collection subprocedure, and $\hat{\vs}$ as an aggregation subprocedure that takes $\gH_T$ as an input.

Then, using the fact that everything is Gaussian, we can compute the posterior distribution of $\vs \mid \gH_T$ in closed-form, as shown in the following lemma whose proof is deferred to the end:
\begin{lemma}
\label{lem:gaussian-posterior}
    $\vs \mid \gH_T \sim \gN(\E[\vs \mid \gH_T], \gI(\gA)^{-1})$, where $\gI(\gA) = \mathrm{diag}\left( (\gI_k(\gA))_{k \in [K]} \right)$ is the (diagonal) posterior Fisher information of the algorithm $\gA$ defined as
    \begin{equation}
        \gI_k(\gA) = \lambda^{-2} +\sum_{j \in [J]} \frac{N_{k,j}}{\sigma_{k,j}^2}.
    \end{equation}
\end{lemma}

Denoting $\gH_T(\gA)$ to be the random full interaction history for $\gA$,
\begin{align}
    \inf_\gA \inf_{\hat{\vs}} \E_{\vs \sim \mu}^{\gA, \hat{\vs}}\left[ f\left( \hat{\vs}(\gH_T(\gA)) - \vs \right) \right] = \inf_\gA \inf_{\hat{\vs}} \E^\gA\left[ \underbrace{\E_{\vs \mid \gH_T(\gA)}^{\hat{\vs}} \left[ f\left( \hat{\vs}(\gH_T(\gA)) - \vs \right) \mid \gH_T(\gA) \right]}_{\textbf{\textit{conditional posterior risk}}} \right],
\end{align}
where either $f(\vx) = \bignorm{\vx}_p$ or $f(\vx) = \bignorm{\vx}_p^p$.

We now argue that the posterior mean minimizes the conditional posterior risk, i.e., when $\hat{\vs}(\gH_T) = \E[\vs \mid \gH_T]$.
To see this, let $Z := S - \E[\vs \mid \gH_T]$.
Then, conditioned on $\gH_T$, $Z$ is centered Gaussian and symmetric.
Also, both choices of $f$ are convex and even ($f(\vu) = f(-\vu)$).

For any estimator $\hat{\vs}$, define its posterior risk as
\begin{equation}
    L(\hat{\vs}) \coloneq \E[f(\hat{\vs} - \vs) \mid \gH_T] = \E[f(\vu - Z) \mid \gH_T],
\end{equation}
where $\vu \coloneq \hat{\vs} - \E[\vs \mid \gH_T]$.
Then as $L$ is convex, by Jensen's midpoint inequality, we have that
\begin{equation}
    L(0) \leq \frac{L(\vu) + L(-\vu)}{2} = L(\vu),
\end{equation}
i.e., the minimum is attained at $\vu = 0$.
This implies that the posterior mean $\E[\vs \mid \gH_T]$ minimizes the conditional posterior risk.\footnote{This resembles Anderson's lemma~\citep{anderson1955}, which states that a centered Gaussian measure assigns maximal mass to a symmetric convex set when the set is centered at the origin.}

Combining this with \Cref{lem:gaussian-posterior}, we have that
\begin{align}
    \inf_\gA \inf_{\hat{\vs}} \E_{\vs \sim \mu}^{\gA, \hat{\vs}}\left[ f\left( \hat{\vs}(\gH_T(\gA)) - \vs \right) \right] &\geq \inf_\gA \E^\gA\left[ \E_{\vs \mid \gH_T(\gA)} \left[ f\left( \E[\vs \mid \gH_T(\gA)] - \vs \right) \mid \gH_T(\gA) \right] \right] \\
    &= \inf_\gA \E_{\vx \sim \gN(\vzero, \gI(\gA)^{-1})}\left[ f(\vx) \right] \tag{\Cref{lem:gaussian-posterior}} \\
    &= \inf_\gA \E_{\vg \sim \gN(\vzero, \mI_K)}\left[ f\left( \gI(\gA)^{-\frac{1}{2}} \odot \vg \right) \right].
\end{align}

When $f(\vx) = \bignorm{\vx}_p$, by Jensen's inequality,\footnote{$\vg \mapsto \bignorm{\vu \odot \vg}_p$ is convex over $\sR_{>0}^K$, for any $\vu \in \sR_{>0}^K$.}
\begin{align}
    \E_{\vg \sim \gN(\vzero, \mI_K)}\left[ \bignorm{ \gI(\gA)^{-\frac{1}{2}} \odot \vg}_p \right]
    &= \E_{\vg \sim \gN(\vzero, \mI_K)}\left[ \bignorm{ \gI(\gA)^{-\frac{1}{2}} \odot |\vg|}_p \right] \\
    &\geq \bignorm{ \gI(\gA)^{-\frac{1}{2}} \odot \E_{\vg \sim \gN(\vzero, \mI_K)}[\vg]}_p 
    = \sqrt{\frac{2}{\pi}} \bignorm{\gI(\gA)^{-\frac{1}{2}}}_p,
\end{align}
where $|\cdot|$ is applied element-wise.

When $f(\vx) = \bignorm{\vx}_p^p$.
\begin{align}
    \E_{\vg \sim \gN(\vzero, \mI_K)}\left[ \bignorm{ \gI(\gA)^{-\frac{1}{2}} \odot \vg}_p^p \right]^{\frac{1}{p}}
    = \sum_{k \in [K]} \left( \gI(\gA)^{-\frac{1}{2}} \right) \E_{g \sim \gN(0, 1)}[|g|^p]^{\frac{1}{p}}
    = m_p \bignorm{\gI(\gA)^{-\frac{1}{2}}}_p.
\end{align}
Either way, we arrive at $\bignorm{\gI(\gA)^{-\frac{1}{2}}}_p$, which we lower bound as follows.
Let \(N_{k,j}\) denote the total number of times pair \((k,j)\) is queried by algorithm $\gA$ subject to the budget constraint, and let us reparametrize it by $\omega_{k,j} \coloneq \frac{c_j N_{k,j}}{B}$.
Then, we have that for each allocation $\{\omega_{k,j}\}_{k,j}$,
\begin{equation}
    \bignorm{\gI(\gA)^{-\frac{1}{2}}}_p = \left( \sum_{k \in [K]} \left( \lambda^{-2} + \sum_{j \in [J]} \frac{B \omega_{k,j}}{c_j \sigma_{k,j}^2} \right)^{-p/2} \right)^{1/p}, \quad \forall \lambda > 0.
\end{equation}
Taking the limit as $\lambda \to \infty$, we have that
\begin{equation}
    \bignorm{\gI(\gA)^{-\frac{1}{2}}}_p = \frac{1}{\sqrt{B}} \left( \sum_{k \in [K]} \left( \sum_{j \in [J]} \frac{\omega_{k,j}}{c_j \sigma_{k,j}^2} \right)^{-p/2} \right)^{1/p}
    = \sqrt{\frac{\gA_p(\bm\omega; \bm\sigma, \vc)}{B}}.
\end{equation}
Taking the infimum over $\gA$, which is equivalent to taking the infimum over $\bm\omega$ under the budget constraint, we conclude with \Cref{thm:optimal-allocation}.
\qed

\begin{remark}
    One could also utilize Gaussian hypercontractivity (specifically, Kahane's inequality) to transform a lower bound on $\E\left[\bignorm{\hat{\vs} - \vs}_p^p\right]^{1/p}$ into a lower bound on $\E\left[\bignorm{\hat{\vs} - \vs}_p\right]$, up to some Gaussian moment-dependent constants.
    But this leads to a bit more suboptimal constant.
\end{remark}

\begin{proof}[Proof of \Cref{lem:gaussian-posterior}]
    This follows from directly factorizing the likelihood and completing the square.
    Denoting $\pi_t(\cdot, \cdot)$ as the sampling kernel of the algorithm $\gA$, we can drop terms independent of $\vs$ by absorbing them into a proportionality constant. The likelihood is:
    \begin{align}
        p(\gH_T \mid \vs) &= \prod_{t=1}^T \left\{ \pi_t((k_t,j_t)\mid \gH_{t-1}) \frac{1}{\sqrt{2\pi}\sigma_{k_t,j_t}} \exp\left( - \frac{(\tilde{s}_t - s_{k_t})^2}{2\sigma_{k_t,j_t}^2} \right) \right\} \\
        &\propto \exp\left( - \frac{1}{2} \sum_{t=1}^T \frac{(\tilde{s}_t - s_{k_t})^2}{\sigma_{k_t,j_t}^2} \right).
    \end{align}
    By grouping the observations by $(k, j)$, and using the empirical mean $\bar{s}_{k,j} \coloneq \frac{1}{N_{k,j}} \sum_{t=1}^T \indicator[k_t = k, j_t = j] \tilde{s}_t$, the sum of squares decomposes as:
    \begin{equation}
        \sum_{t=1}^T \frac{(\tilde{s}_t - s_{k_t})^2}{\sigma_{k_t,j_t}^2} = \sum_{k,j} \frac{1}{\sigma_{k,j}^2} \left( \sum_{t: k_t=k, j_t=j} (\tilde{s}_t - \bar{s}_{k,j} + \bar{s}_{k,j} - s_{k})^2 \right).
    \end{equation}
    Since the cross-terms sum to zero ($\sum (\tilde{s}_t - \bar{s}_{k,j}) = 0$), this simplifies to:
    \begin{equation}
        \sum_{k,j} \frac{1}{\sigma_{k,j}^2} \left( \sum_{t: k_t=k, j_t=j} (\tilde{s}_t - \bar{s}_{k,j})^2 + N_{k,j}(\bar{s}_{k,j} - s_k)^2 \right).
    \end{equation}
    Absorbing the terms independent of $s_k$ into the proportionality constant, we have:
    \begin{equation}
        p(\gH_T \mid \vs) \propto \exp\left( - \frac{1}{2} \sum_{k \in [K]} \sum_{j \in [J]} \frac{N_{k,j}}{\sigma_{k,j}^2} (s_k - \bar{s}_{k,j})^2 \right).
    \end{equation}
    Now, applying Bayes' rule with the prior $\vs \sim \mu = \gN(0, \lambda^2 \mI_K)$, the posterior is:
    \begin{align}
        p(\vs \mid \gH_T) &\propto p(\gH_T \mid \vs) \mu(\vs) \\
        &\propto \exp\left( - \frac{1}{2} \sum_{k,j} \frac{N_{k,j}}{\sigma_{k,j}^2} (s_k - \bar{s}_{k,j})^2 \right) \exp\left( - \frac{1}{2} \sum_{k \in [K]} \frac{s_k^2}{\lambda^2} \right) \\
        &= \prod_{k \in [K]} \exp \left( - \frac{1}{2} \left[ \frac{s_k^2}{\lambda^2} + \sum_{j \in [J]} \frac{N_{k,j}}{\sigma_{k,j}^2} (s_k - \bar{s}_{k,j})^2 \right] \right).
    \end{align}
    Focusing on the exponent for each $k$, we expand and collect the $s_k$ terms:
    \begin{align}
        \frac{s_k^2}{\lambda^2} + \sum_{j \in [J]} \frac{N_{k,j}}{\sigma_{k,j}^2} (s_k^2 - 2 s_k \bar{s}_{k,j} + \bar{s}_{k,j}^2)
        &= s_k^2 \underbrace{\left( \frac{1}{\lambda^2} + \sum_{j \in [J]} \frac{N_{k,j}}{\sigma_{k,j}^2} \right)}_{= \gI_k(\gA)} - 2s_k \sum_{j \in [J]} \frac{N_{k,j}}{\sigma_{k,j}^2} \bar{s}_{k,j} + C_k \\
        &= \gI_k(\gA) (s_k - \mu_k)^2 + C_k',
    \end{align}
    where $C_k$ and $C_k'$ are terms independent of $s_k$, and $\mu_k \coloneq \gI_k(\gA)^{-1} \sum_j \frac{N_{k,j}}{\sigma_{k,j}^2} \bar{s}_{k,j}$.
    Thus, $p(\vs \mid \gH_T) \propto \prod_k \exp\left( - \frac{\gI_k(\gA)}{2} (s_k - \mu_k)^2 \right)$, which is exactly the density of a Gaussian distribution with mean $\boldsymbol{\mu}$ and a diagonal precision matrix whose $k$-th diagonal entry is $\gI_k(\gA)$.
\end{proof}

\subsection{\texorpdfstring{Proof of \Cref{thm:lower-bound-van-trees}: Global In-Expectation Lower Bound for Gaussian Scores via van Trees Inequality}{Proof of Theorem F.8: Global In-Expectation Lower Bound for Gaussian Scores via van Trees Inequality}}
\label{app:van-trees}
Our key technical tool is a generalization of van Trees inequality~\citep{vanTrees,gill-levit} to $L_p$-norm due to \citet{chen2024vantrees} via Efroimovich's inequality~\citep{efroimovich1979}:
\begin{lemma}[Scalar Version of Theorem 2.4 of \citet{chen2024vantrees}]
\label{lem:van-trees}
    Let $1 \leq p < \infty$ and define $C_{\mathrm{ME}}(p) \triangleq 2 e^{\frac{1}{p}} \Gamma(p^{-1}) p^{\frac{1}{p} - 1}$ to be the partition function of the $L_p$ max-entropy distribution.
    For a closed interval $\Theta \subseteq \sR$, suppose that we have a parametric family of distributions $\{\sP_\theta\}_{\theta \in \Theta}$ with density $f(x; \theta)$, and let $\mu$ be a prior density over $\Theta$ that is absolutely continuous, \textbf{vanishes at the boundary}, and has finite Fisher information.
    
    Then, for any estimator $\hat{\theta} : X \mapsto \hat{\theta}(X) \in \sR$ where $X \sim P_\theta$, the \textbf{expected $p$-th moment risk} is lower-bounded as follows:
    \begin{equation}
        \E_{\theta \sim \mu}\E_{\hat{\theta}, X}\left[ \bigabs{\theta - \hat{\theta}(X)}^p \right] \geq \left( \frac{\sqrt{2 \pi e}}{C_{\mathrm{ME}}(p)} \right)^p \left( \E_{\theta \sim \mu}[\gI(\theta)] + \gJ(\mu) \right)^{-\frac{p}{2}},
    \end{equation}
    where the \textbf{Fisher information} terms are defined as
    \begin{align}
        \gI(\theta) \coloneq \E_X\left[ -\frac{\partial^2}{\partial \theta^2} \log p(X \mid \theta) \right], \quad \gJ(\mu) \coloneq \E_{\theta \sim \mu} \left[ -\frac{\partial^2}{\partial \theta^2} \log \mu(\theta) \right].
    \end{align}
    Especially note that when $p = 2$, $C_{\mathrm{ME}}(2) = \sqrt{2 \pi e}$, and the constant multiple becomes $1$, as in the original van Trees inequality~\citep{gill-levit,vanTrees}.
\end{lemma}

And again, we begin By Yao's Minimax Principle~\citep{yao1977}: for \emph{any} decomposable prior $\mu = \bigotimes_{k \in [K]} \mu_k$ over $\vs$,
\begin{align}
    \inf_{\gA} \sup_{\vs} \E^\gA_{\vs}\left[ \bignorm{\hat{\vs} - \vs}_p^p \right] &\geq \inf_{\gA} \E_{\vs \sim \mu}\left[ \E^\gA_{\vs}\left[ \bignorm{\hat{\vs} - \vs}_p^p \right] \right] \\
    &= \inf_{\gA} \sum_k \E_{s_k \sim \mu_k}\left[ \E^\gA_{s_k}\left[ |\hat{s}_k - s_k|^p \right] \right],
\end{align}
where the algorithm $\gA$ must satisfy the knapsack constraint: if $\gA$ samples each $(k,j)$ pair $N_{k,j}$ times, then $\sum_j c_j \sum_k N_{k,j} \leq B$.
Let us fix an arbitrary algorithm $\gA$, and the corresponding numbers of pulls $N_{k,j}$'s.

As first suggested in \citet{borovkov1998,tsybakov}, we choose $\mu_k$ as the \emph{cosine-squared} prior density, indexed by $r > 0$:\footnote{One can verify using calculus of variation that this prior density is the ``optimal'' choice that minimizes $\gJ(\mu)$ among all continuously differentiable densities supported on $[-r, r]$ that vanish at the boundaries.}
\begin{equation}
    \mu_k^{(r)}(s) = \frac{1}{r} \cos^2\left(\frac{\pi (s - s_k)}{2 r}\right) \mathbb{I}_{[s_k - r, s_k + r]}(s).
\end{equation}

Its prior Fisher information $\gJ(\mu_k^{(r)}) \triangleq \E_{s \sim \mu_k^{(r)}}\left[ -\frac{\partial^2}{\partial s^2} \log \mu_k^{(r)}(s) \right]$ is computed as:
\begin{align}
    \gJ(\mu_k^{(r)}) &= \int_{s_k-r}^{s_k+r} \mu_k^{(r)}(s) \left( -\frac{\partial^2}{\partial s^2} \left[ \log\left(\frac{1}{r}\right) + 2 \log \cos\left(\frac{\pi (s - s_k)}{2r}\right) \right] \right) ds \\
    &= \int_{-r}^{r} \frac{1}{r} \cos^2\left(\frac{\pi s}{2 r}\right) \left( \frac{\pi^2}{2 r^2} \sec^2\left(\frac{\pi s}{2 r}\right) \right) ds \\
    &= \int_{-r}^{r} \frac{\pi^2}{2 r^3} ds = \frac{\pi^2}{r^2}.
\end{align}
The Fisher information of the Gaussian observations and algorithm $\gA$ is computed as follows: denoting $p(\cdot | s, \sigma^2)$ as the probability density of $\gN(s, \sigma^2),$
\begin{align}
    \E_{s_k \sim \mu_k^{(r)}}\left[ \gI(s_k) \right] &= \E_{s_k \sim \mu_k^{(r)}}^\gA \left[ -\frac{\partial^2}{\partial s_k^2} \left[ \sum_{t=1}^T \indicator[k_t = k] \log p(\tilde{s}_t \mid s_k, \sigma_{k_t,j_t}^2) \right] \right] \\
    &= \E_{s_k \sim \mu_k^{(r)}}^\gA \left[ -\frac{\partial^2}{\partial s_k^2} \left[ \sum_{j \in [J]} \sum_{t=1}^T \indicator[k_t = k, j_t = j] \left( -\frac{(\tilde{s}_t - s_k)^2}{2\sigma_{k,j}^2} + C \right) \right] \right] \\
    &= \E^\gA \left[ \sum_{j \in [J]} \sum_{t=1}^T \indicator[k_t = k, j_t = j] \frac{1}{\sigma_{k,j}^2} \right] \\
    &= \E^\gA \left[ \sum_{j \in [J]} \frac{N_{k,j}}{\sigma_{k,j}^2} \right]
    = \sum_{j \in [J]} \frac{\E^\gA[N_{k,j}]}{\sigma_{k,j}^2}.
\end{align}
With this, we apply the \textbf{\textit{$L_p$-van Trees inequality}} (\Cref{lem:van-trees}) for each $k$, which gives
\begin{align}
    \sum_k \E_{s_k \sim \mu_k^{(r)}}\left[ \E^\gA_{s_k}\left[ |\hat{s}_k - s_k|^p \right] \right]
    &\geq \left( \frac{\sqrt{2 \pi e}}{C_{\mathrm{ME}}(p)} \right)^p \sum_{k \in [K]} \left( \sum_{j \in [J]} \frac{\E^\gA[N_{k,j}]}{\sigma_{k,j}^2} + \frac{\pi^2}{r^2} \right)^{-\frac{p}{2}}.
\end{align}
Chaining everything, we then have that
\begin{equation}
    \sup_{\vs} \E^\gA_{\vs}\left[ \bignorm{\hat{\vs} - \vs}_p^p \right] \geq \left( \frac{\sqrt{2 \pi e}}{C_{\mathrm{ME}}(p)} \right)^p \sum_{k \in [K]} \left( \sum_{j \in [J]} \frac{B \omega_{k,j}}{c_j \sigma_{k,j}^2} + \frac{\pi^2}{r^2} \right)^{-\frac{p}{2}},
\end{equation}
where as in the proof of \Cref{thm:lower-bound-whp}, we define $\omega_{k,j} = \frac{\E^\gA[N_{k,j}] c_j}{B}$.

As this holds for any $r > 0$, taking the limit on both sides as $r \rightarrow \infty$ and $1/p$-th power, we have
\begin{equation}
    \sup_{\vs} \E^\gA_{\vs}\left[ \bignorm{\hat{\vs} - \vs}_p^p \right]^{\frac{1}{p}} \geq \frac{\sqrt{2 \pi e}}{C_{\mathrm{ME}}(p)} \left( \sum_{k \in [K]} \left( \sum_{j \in [J]} \frac{B \omega_{k,j}}{c_j \sigma_{k,j}^2} \right)^{-\frac{p}{2}} \right)^{\frac{1}{p}}.
\end{equation}
Taking the infimum over $\gA$ (and thus over $\bm\omega$) and invoking \Cref{thm:optimal-allocation}, we are done.
\qed

%% file: 908Experiments.tex
\section{Additional Experiments}\label{app:experiments}
This appendix provides additional experimental results that extend the analysis presented in Section 7. We provide further empirical evidence regarding the optimality of our algorithms across different norms and datasets.

\subsection{Details and Additional Results in Experiments on Synthetic Data}\label{app:synthetic_data_extended}

For each query $k \in [K]$, the ground-truth score $s_k$ is sampled uniformly from the interval $[0.1, 0.9]$, then the true variance $\sigma_{k,j}^2$ for each judge $j$ on query $k$ is sampled uniformly from $[10^{-4}, 0.9 s_k(1-s_k)]$. The upper bound of the interval is set to 90\% of the maximum possible variance for a bounded distribution on $[0,1]$ with a mean $s_k$.
Finally, the cost of each judge is sampled uniformly from the interval $[0.5, 1.5]$.

We use \(p\in\{1,2,\infty\}\), budgets \(B\in\{10^7,3\times10^7,5\times10^7,7\times10^7,9\times10^7\}\), and 50 independent runs. All experiments use random seed $0$.

In Figure \ref{fig:add-exp-synthetic}, we present the $\ell_1$ and $\ell_\infty$-errors of the synthetic experiments described in \Cref{sec:experiments} ($K=1000, J=10$). Consistent with the $\ell_2$-error in the main text, both \estivwe{} (Bounded) and \estivwe{} (Gaussian) show a clear convergence toward the \textsc{Oracle} as the budget $B$ increases. This confirms that our adaptive allocation strategy effectively minimizes the $\ell_p$-error for any $p$, outperforming the non-adaptive \textsc{Uniform} strategy.
\begin{figure}[!h]
    \centering
    \begin{subfigure}[b]{0.495\textwidth}
        \centering
        \includegraphics[width=\linewidth]{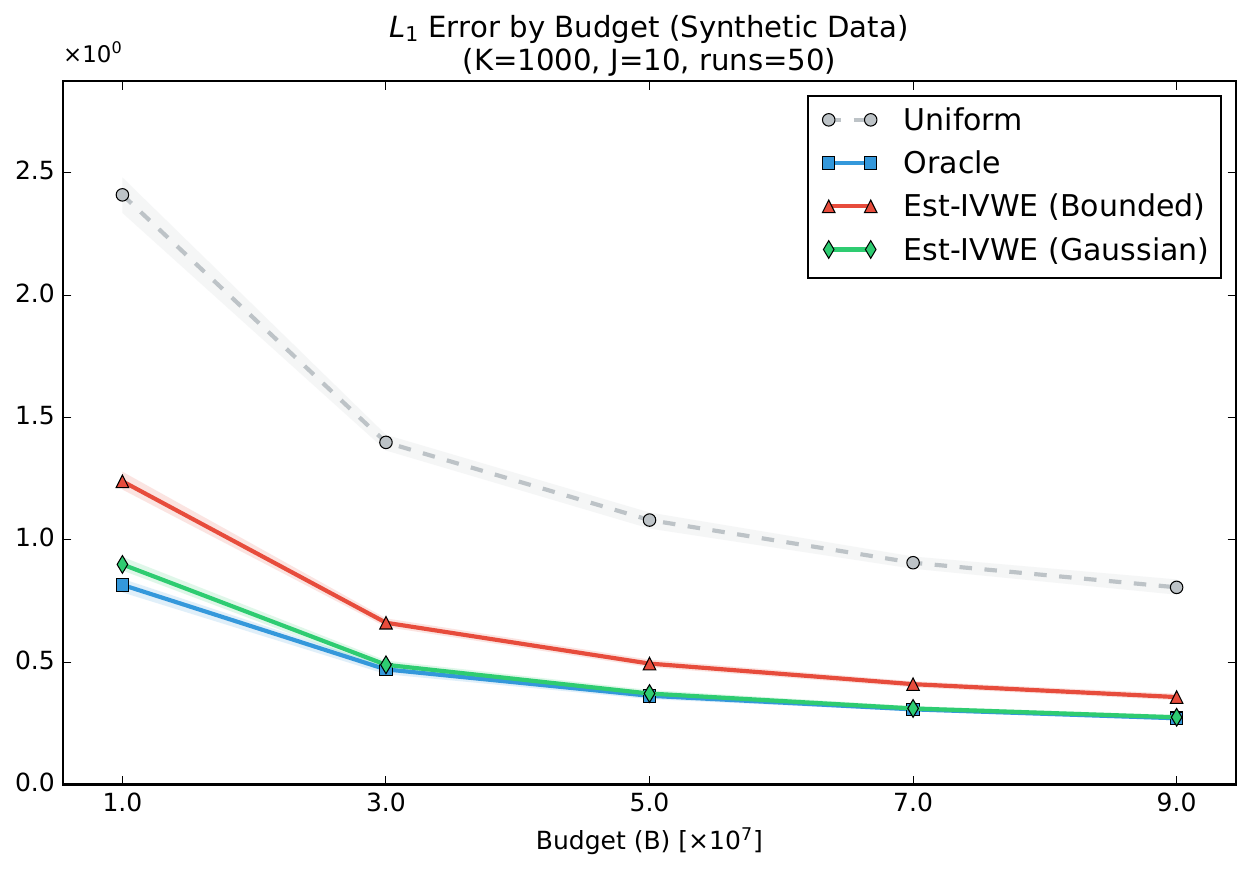}
        \caption{$\ell_1$-Error.}
        \label{fig:synthetic_l1}
    \end{subfigure}
    \hfill
    \begin{subfigure}[b]{0.495\textwidth}
        \centering
        \includegraphics[width=\linewidth]{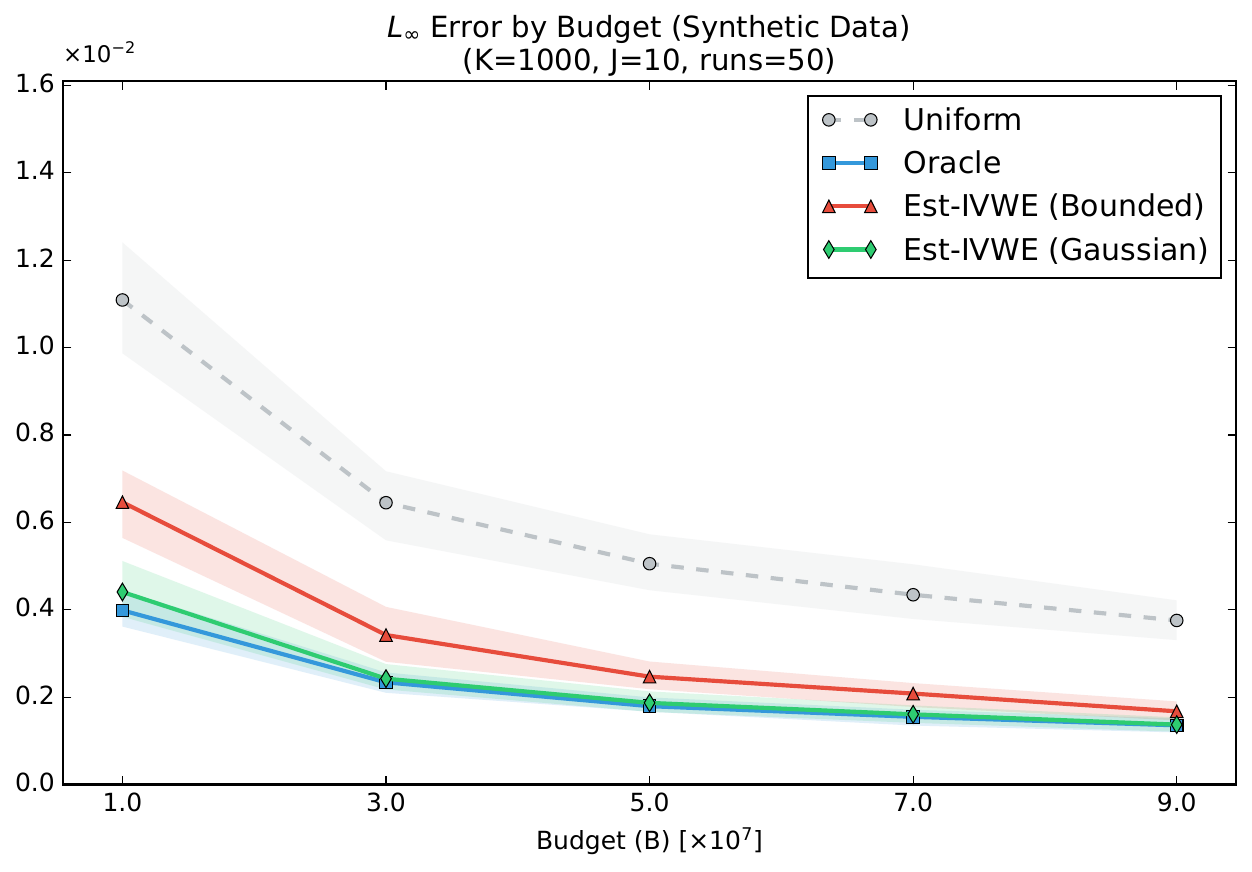}
        \caption{$\ell_\infty$-Error.}
        \label{fig:synthetic_linf}
    \end{subfigure}
    \caption{Additional Experimental results on synthetic datasets. All results are averaged over $50$ independent runs, with the $10\%$ and $90\%$ quantiles (shaded regions)}
    \label{fig:add-exp-synthetic}
\end{figure}
\subsection{Details and Additional Results in Experiments on Real-data}\label{app:real_data_extended}
Our preprocessing pipeline establishes a reliable baseline by first performing consensus filtering. We select only those queries where all three judges provide an identical mode score via majority voting. This ensures that the established $s_k$ serves as a credible proxy for the true score. Furthermore, we ensure statistical significance in variance estimation by retaining query-judge pairs that have at least 25 evaluation samples for all judges. Following this filtering, we obtain a subset of queries for four distinct categories: Complexity, Correctness, Helpfulness, and Verbosity. Heteroskedasticity is modeled by empirically calculating the true variance $\sigma_{k,j}^2$ as the mean squared deviation from the consensus score $s_k$.

We use budgets from $200$ to $5,000$ in increments of $400$, \(p\in\{1,2,\infty\}\), uniform judge costs of $0.1$, $50$ independent runs, and random seed $0$. Judge evaluations are simulated by sampling with replacement from the corresponding evaluation pools.

We provide the complete experimental results for all four datasets (Complexity, Correctness, Helpfulness, and Verbosity) preprocessed as described above. Figure \ref{fig:total_errors} illustrates the $\ell_1, \ell_2,$ and $\ell_\infty$ errors for each dataset. Across all twelve scenarios, the \estivwe{} variants consistently show the relative superiority over the \textsc{Uniform} baseline, performing as well as the \textsc{Oracle} strategy as the budget $B$ increases.

\begin{sidewaysfigure}[t]
    \centering
    \begin{subfigure}[b]{0.3\textwidth}
        \centering
        \includegraphics[width=\textwidth]{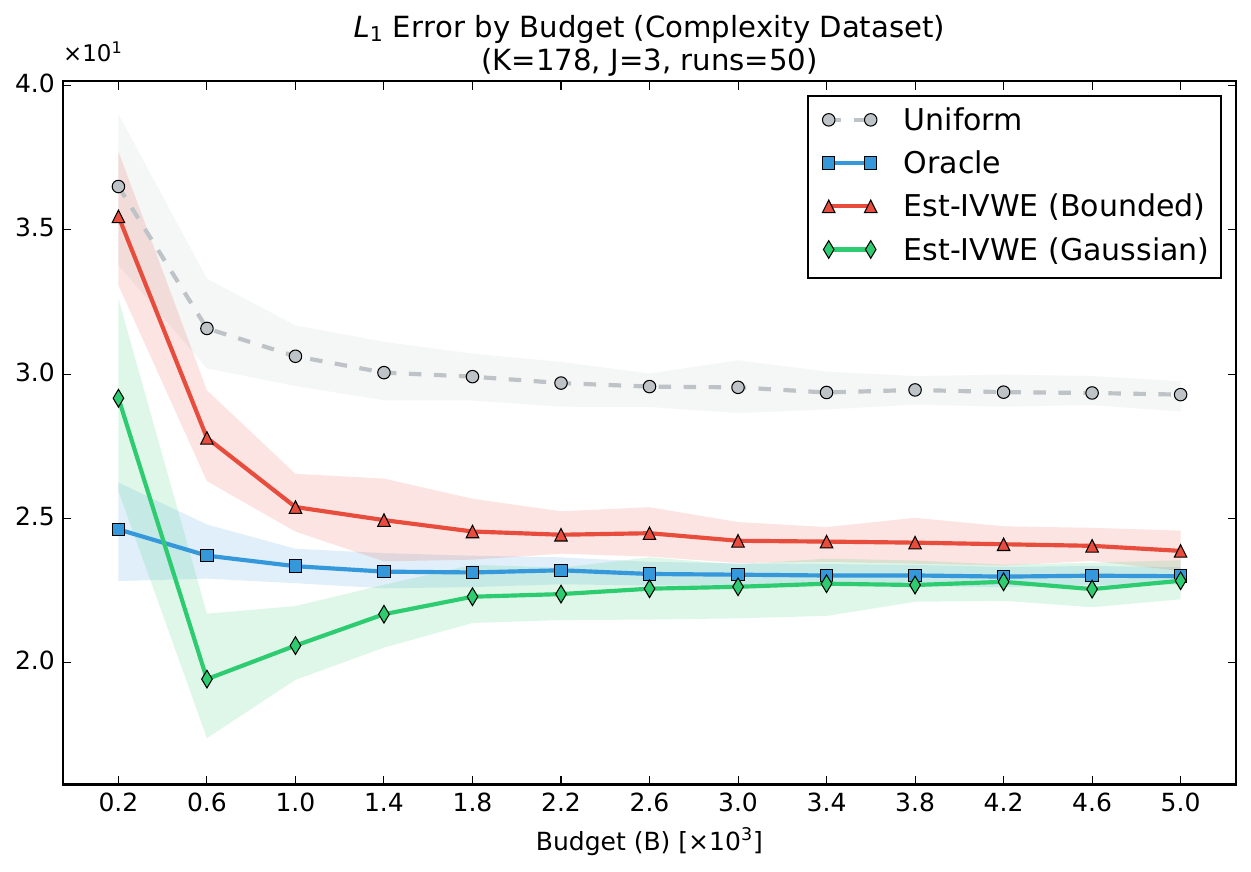}
    \end{subfigure}
    \hspace{10pt}
    \begin{subfigure}[b]{0.3\textwidth}
        \centering
        \includegraphics[width=\textwidth]{fig/real/complexity_K178_J3_runs50_l2.pdf}
    \end{subfigure}
    \hspace{10pt}
    \begin{subfigure}[b]{0.3\textwidth}
        \centering
        \includegraphics[width=\textwidth]{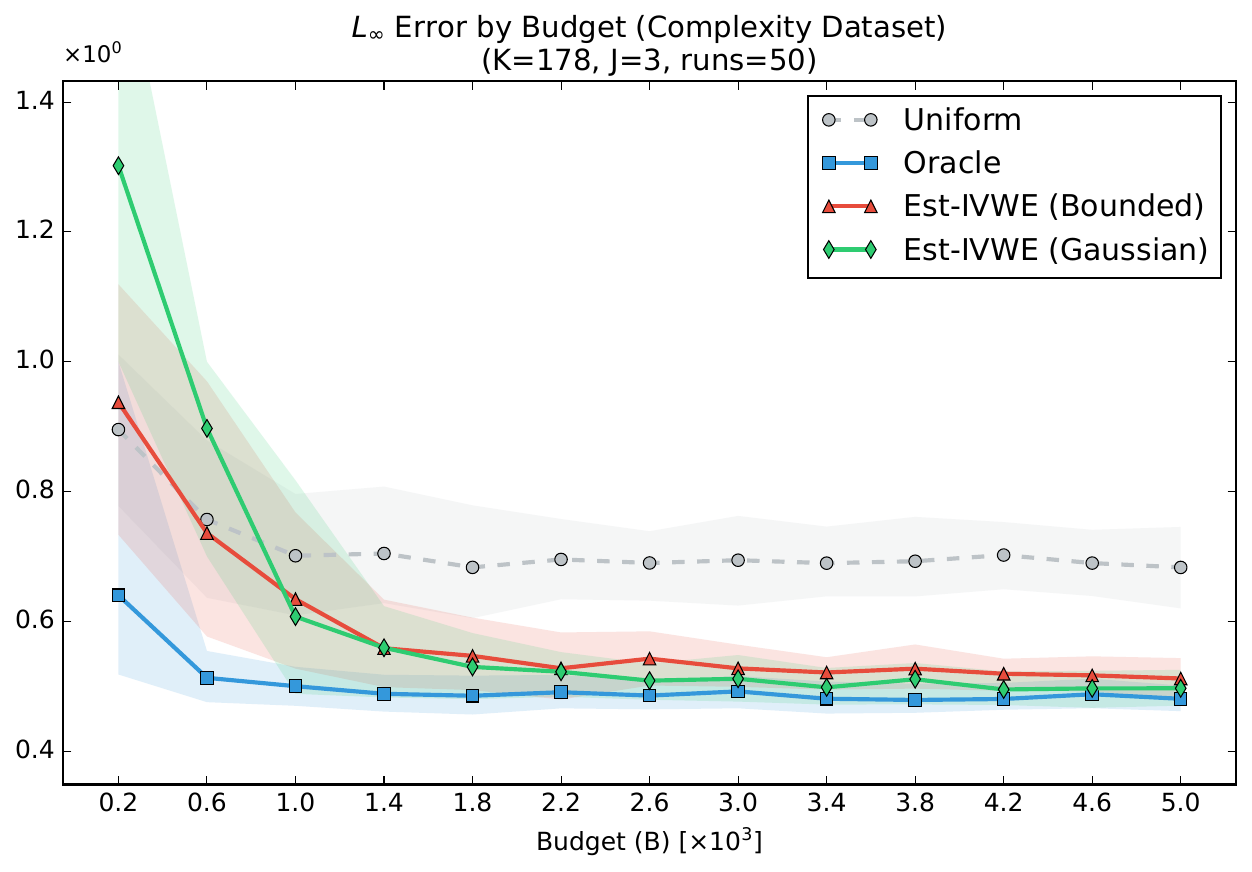}
    \end{subfigure}
    
    \vfill
    
    \begin{subfigure}[b]{0.3\textwidth}
        \centering
        \includegraphics[width=\textwidth]{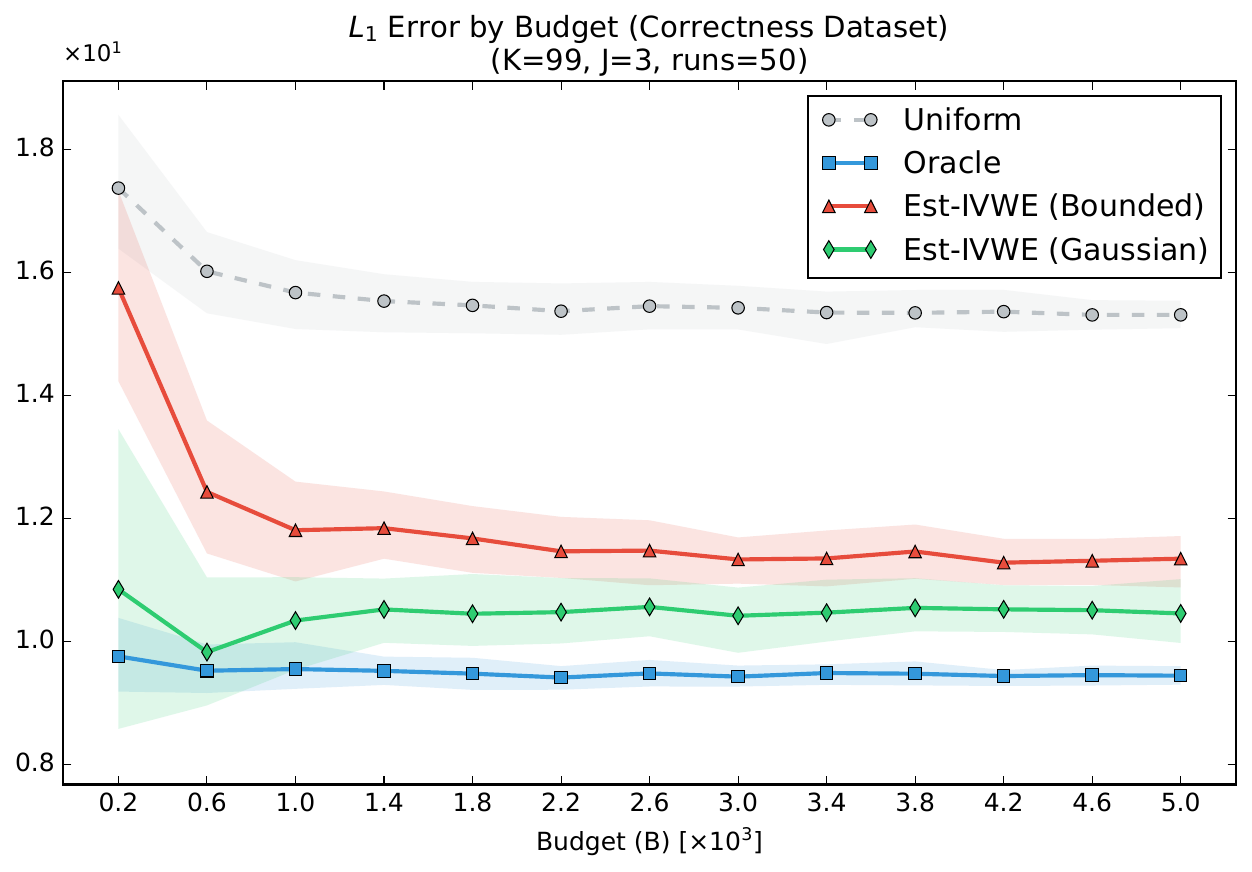}
    \end{subfigure}
    \hspace{10pt}
    \begin{subfigure}[b]{0.3\textwidth}
        \centering
        \includegraphics[width=\textwidth]{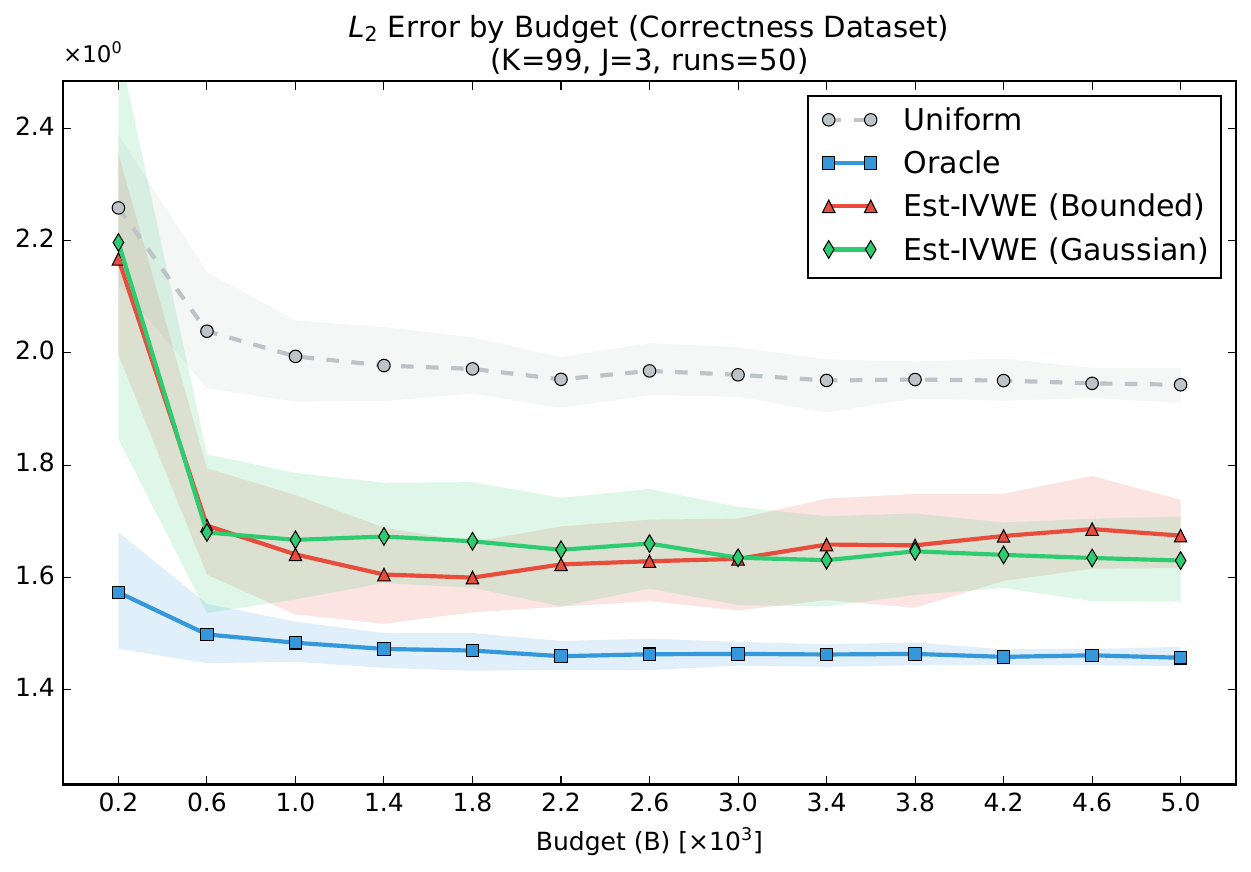}
    \end{subfigure}
    \hspace{10pt}
    \begin{subfigure}[b]{0.3\textwidth}
        \centering
        \includegraphics[width=\textwidth]{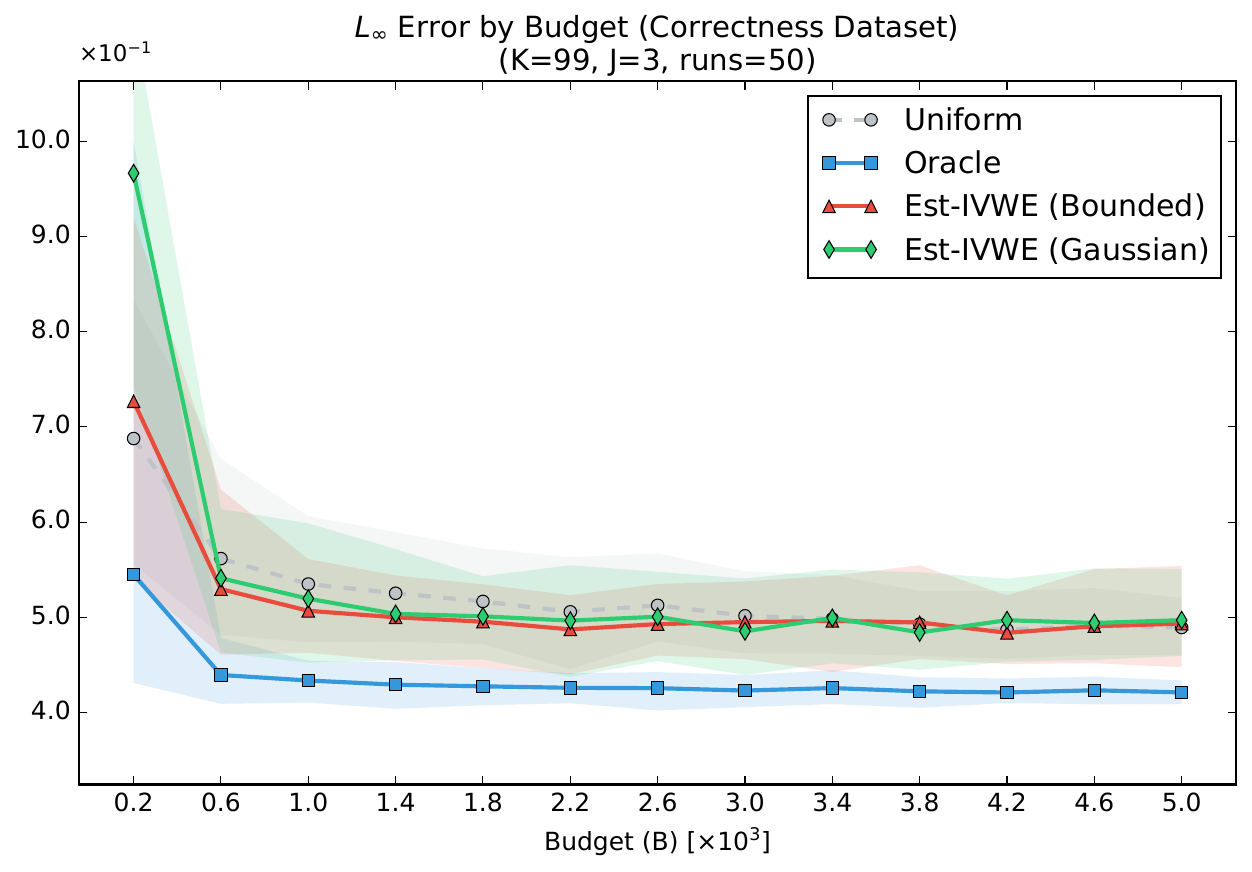}
    \end{subfigure}
    
    \vfill
    
    \begin{subfigure}[b]{0.3\textwidth}
        \centering
        \includegraphics[width=\textwidth]{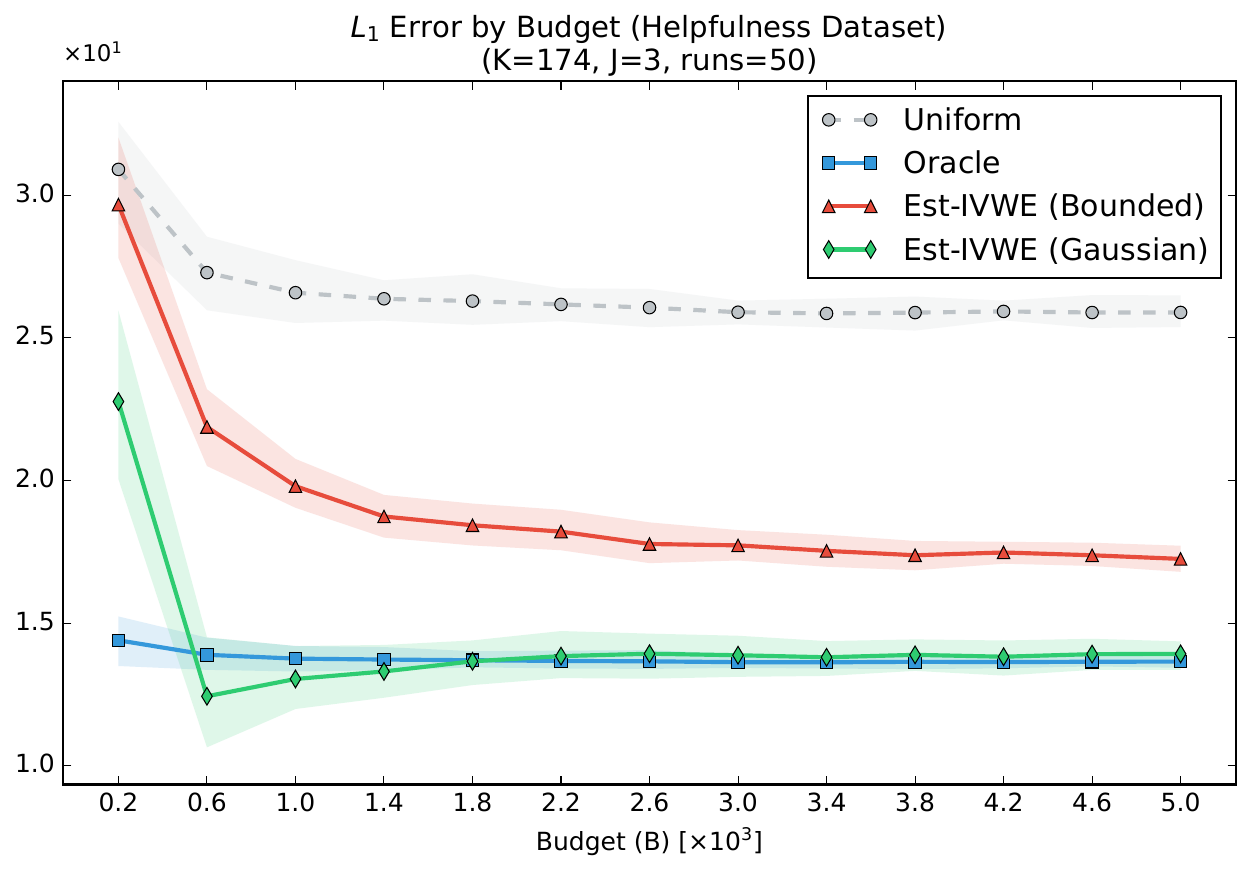}
    \end{subfigure}
    \hspace{10pt}
    \begin{subfigure}[b]{0.3\textwidth}
        \centering
        \includegraphics[width=\textwidth]{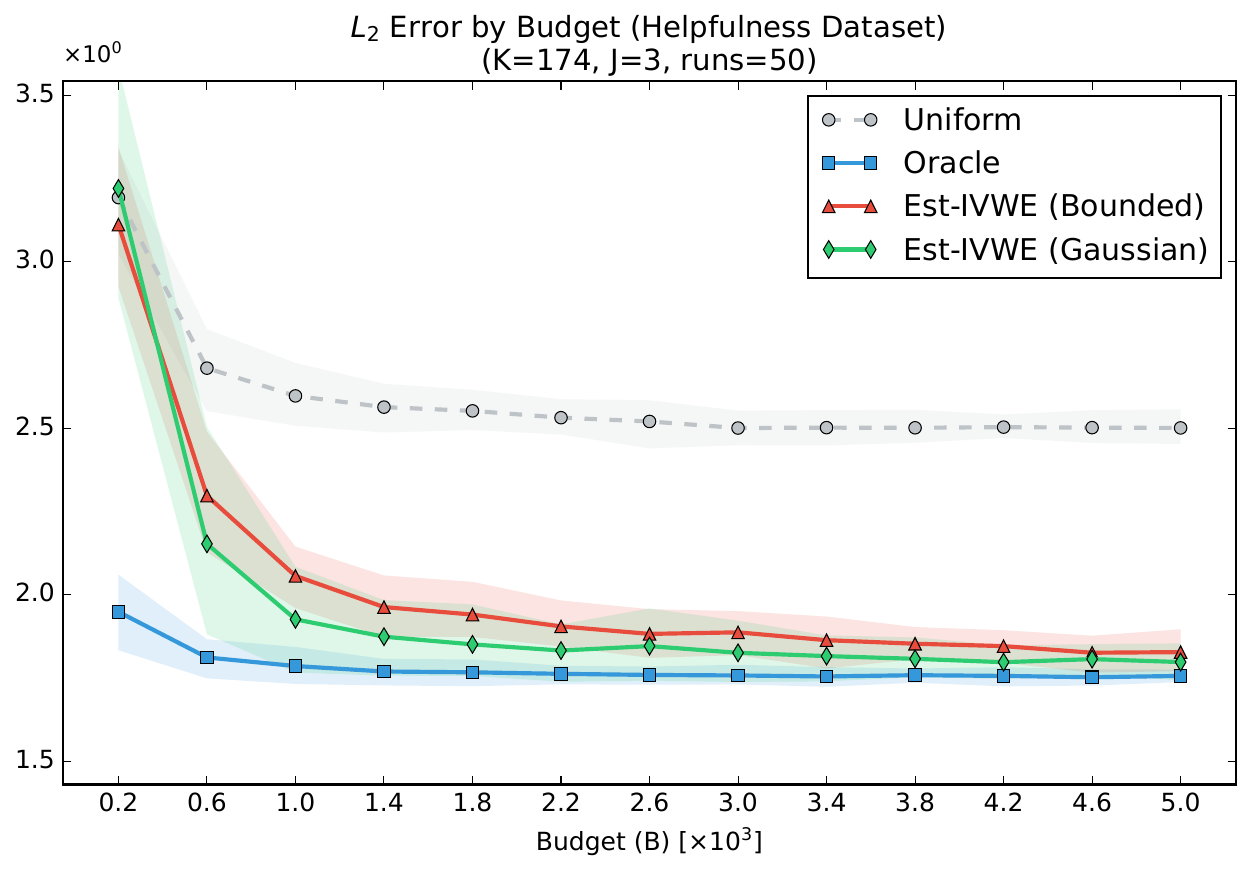}
    \end{subfigure}
    \hspace{10pt}
    \begin{subfigure}[b]{0.3\textwidth}
        \centering
        \includegraphics[width=\textwidth]{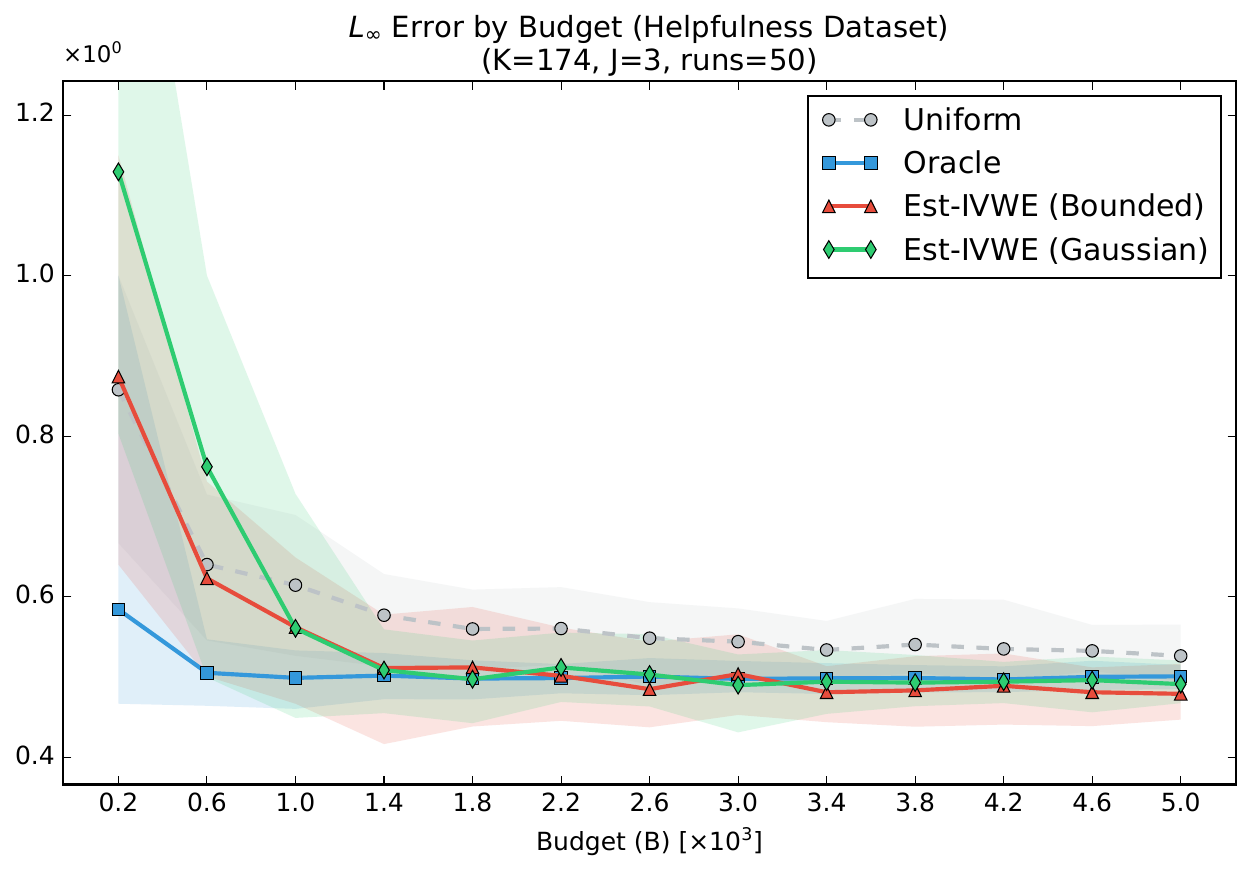}
    \end{subfigure}
    
    \vfill
    
    \begin{subfigure}[b]{0.3\textwidth}
        \centering
        \includegraphics[width=\textwidth]{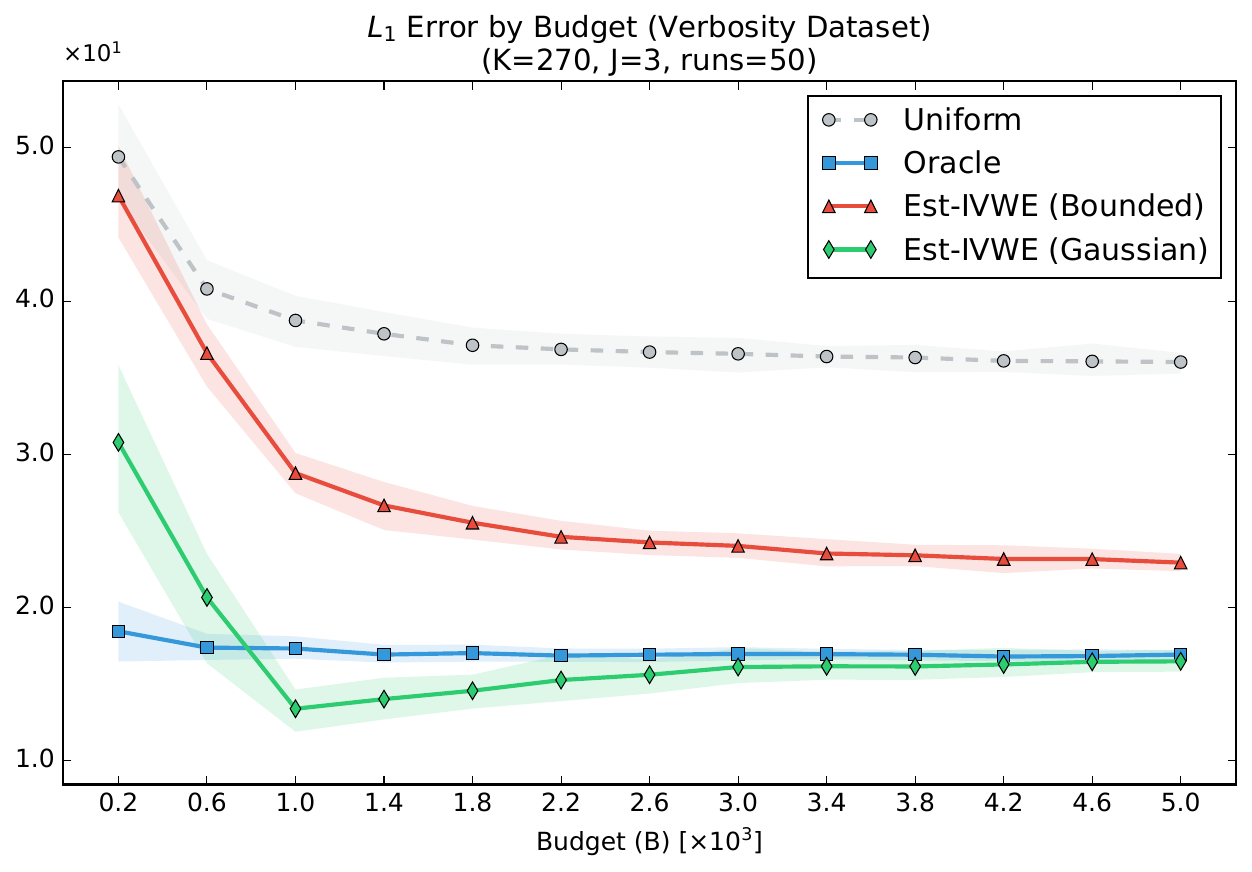}
    \end{subfigure}
    \hspace{10pt}
    \begin{subfigure}[b]{0.3\textwidth}
        \centering
        \includegraphics[width=\textwidth]{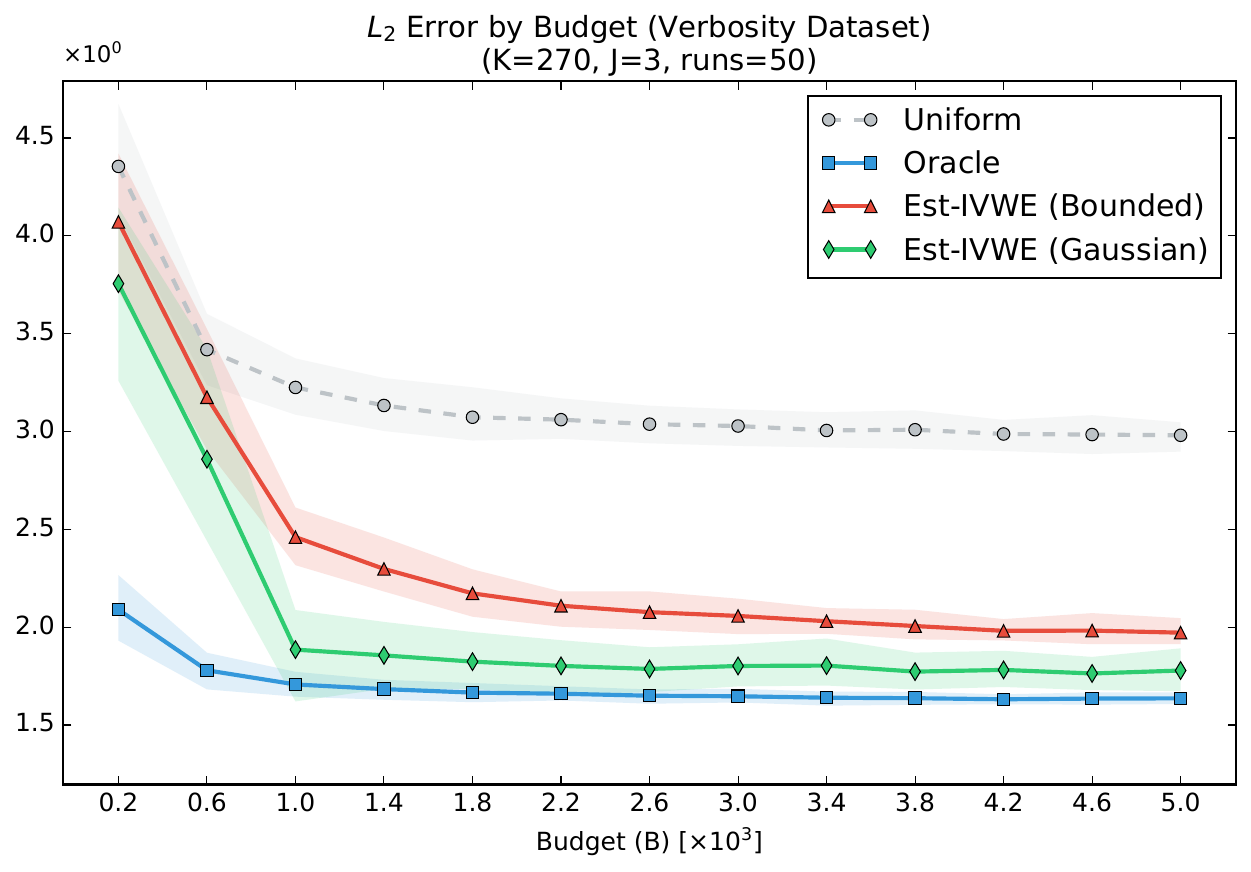}
    \end{subfigure}
    \hspace{10pt}
    \begin{subfigure}[b]{0.3\textwidth}
        \centering
        \includegraphics[width=\textwidth]{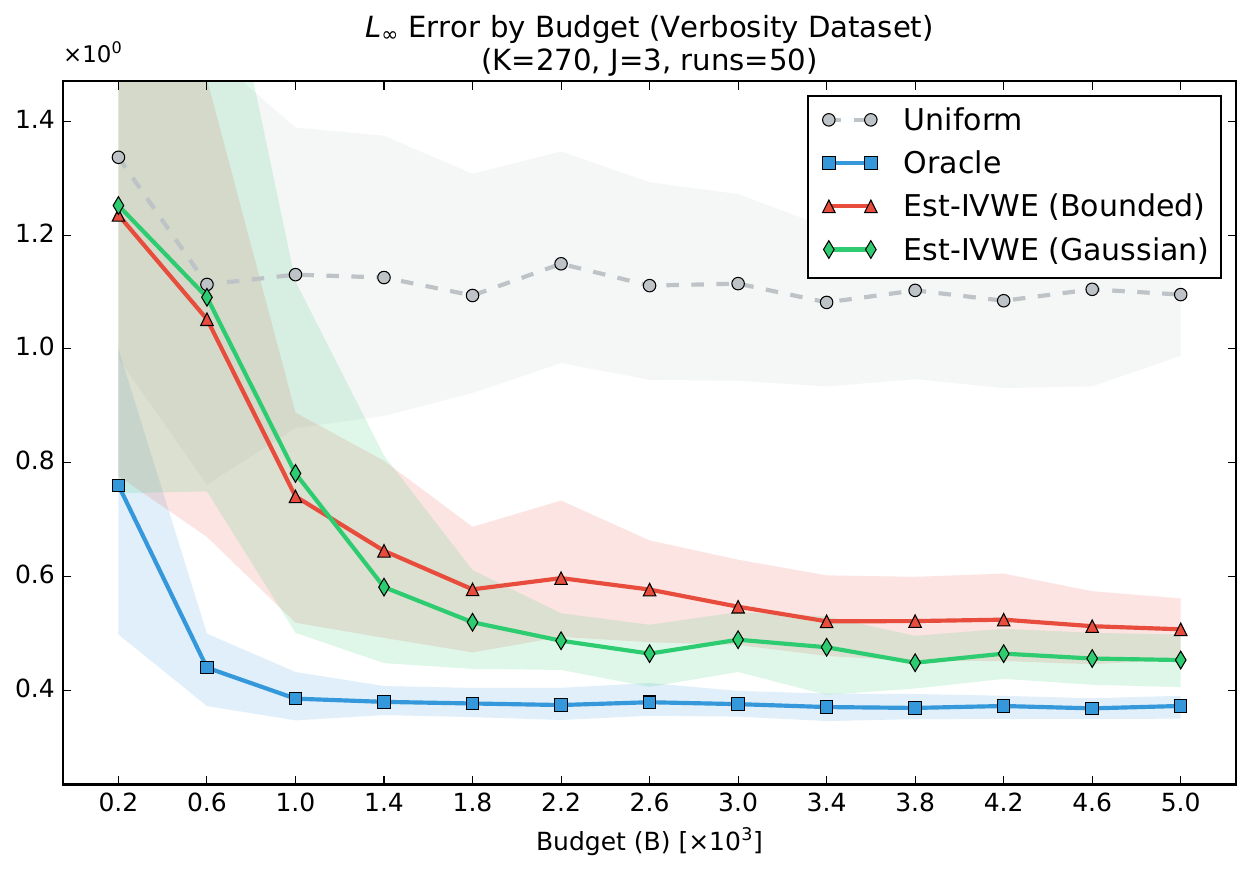}
    \end{subfigure}
    
    \caption{Full experimental results for the datasets (Complexity, Correctness, Helpfulness, Verbosity). Each column corresponds to an error metric ($\ell_1, \ell_2, \ell_\infty$).}
    \label{fig:total_errors}
\end{sidewaysfigure}

\subsection{Sensitivity Analysis on Cost-Quality Correlations}
\label{app:cost_analysis}

To evaluate the robustness of our framework beyond the uniform cost setting employed in \Cref{sec:experiments}, we investigate how varying correlations between a judge's cost and their evaluation accuracy (inverse variance) affect convergence behavior. Theoretically, the selection of an optimal judge is governed by the cost-weighted variance, $j^*(k) = \argmin_{j\in[J]} c_j \sigma_{k,j}^2$, which balances the trade-off between per-sample information gain and budget consumption. We conducted a series of ablation studies using \textsc{Oracle} and \estivwe{} (Bounded) for $p \in \{1, 2, \infty\}$ under three cost structures: (i) \textit{Bad is Expensive}, where variance and cost are positively correlated; (ii) \textit{Bad is Cheap}, where low-quality judges are priced cheaper; and (iii) \textit{Uniform Costs}, where the costs are all the same.
Precisely, for a fixed query $k$, the variance $\sigma_{k,j}^2$ for each judge $j\in[J]$ is sampled from the interval $[ j, J+j-1 ] \times s_k (1-s_k)/ (2J)$ as the larger the index $j$, the larger its variance $\sigma_{k,j}^2$ tends to be.
In \textit{Bad is Expensive} scenario, the cost is set to be proportional to $\sum_{k\in [K]} \sigma_{k,j}^2$; while in \textit{Bad is Cheap} scenario, the cost is set to be proportional to $\sum_{k\in [K]} \sigma_{k,j}^{-2}$.
In both cases, the costs are normalized so that the mean cost over $J$ judges becomes $0.1$.

We use \(K=1000\), \(J=10\), \(p\in\{1,2,\infty\}\), budgets \(B\in\{10^6,3\times10^6,5\times10^6,7\times10^6,9\times10^6\}\), $25$ independent runs, and random seed $0$.

As illustrated in \Cref{fig:cost-structure}, our empirical findings suggest that while the specific cost structure influences the constant factors of the convergence rate, it does not alter the fundamental asymptotic decay of the estimation error. The \textit{Bad is Cheap} scenario represents the most challenging regime for adaptive allocation; when a judge's high variance is offset by a significantly lower cost, the cost-weighted variances $\{c_j \sigma_{k,j}^2\}_{j\in[J]}$ across different judges become numerically similar. This reduced ``gap'' between the optimal judge and noisier alternatives necessitates a more extensive initial exploration ($N_0$) to ensure correct identification. Nevertheless, the results confirm that the algorithm remains robust across diverse cost scenarios, consistently achieving error reduction without failing to identify the optimal judges as the budget scales.

\begin{sidewaysfigure}[h!]
    \centering
    \begin{subfigure}[b]{0.3\textwidth}
        \centering
        \includegraphics[width=\textwidth]{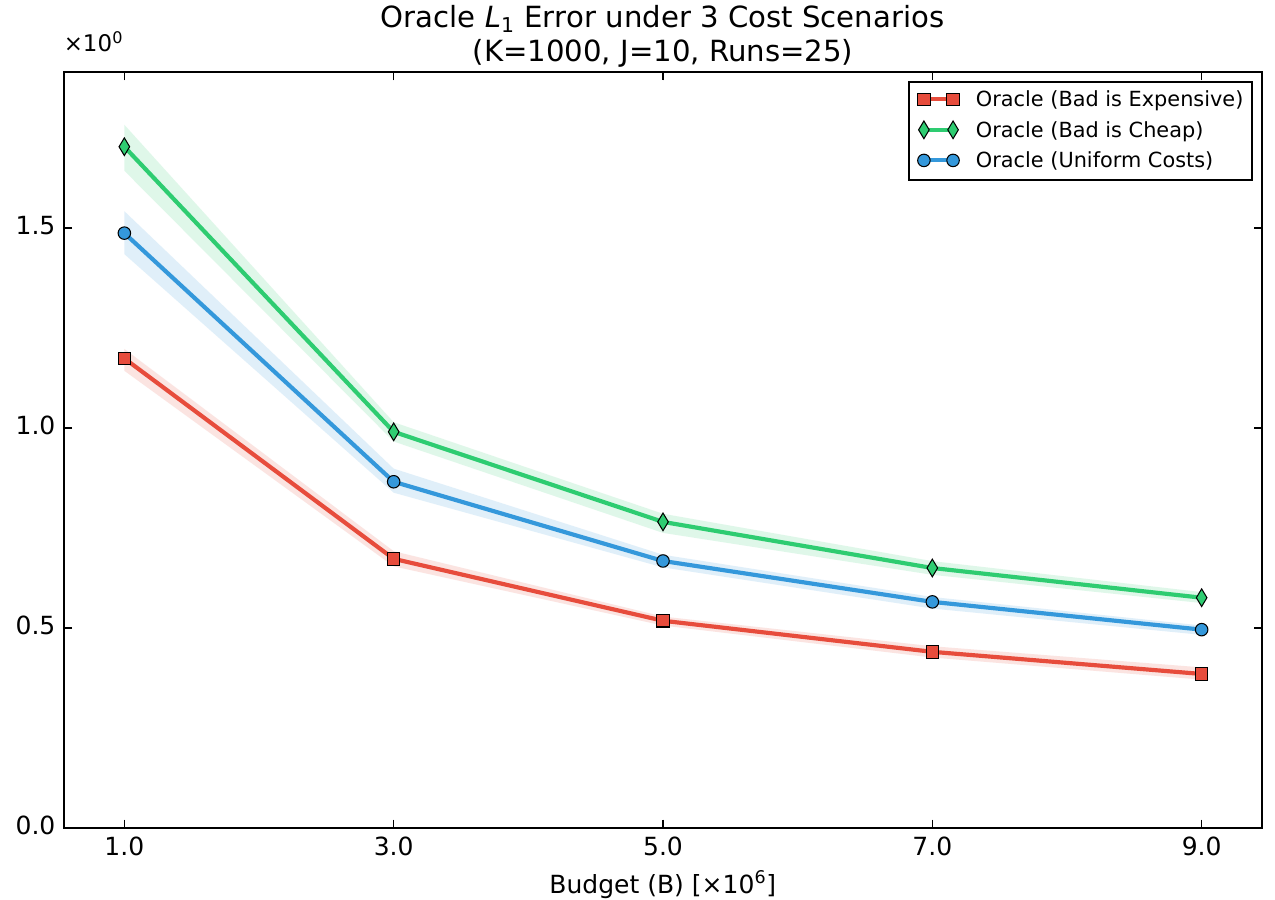}
     \end{subfigure}
    \hspace{10pt}
    \begin{subfigure}[b]{0.3\textwidth}
        \centering
        \includegraphics[width=\textwidth]{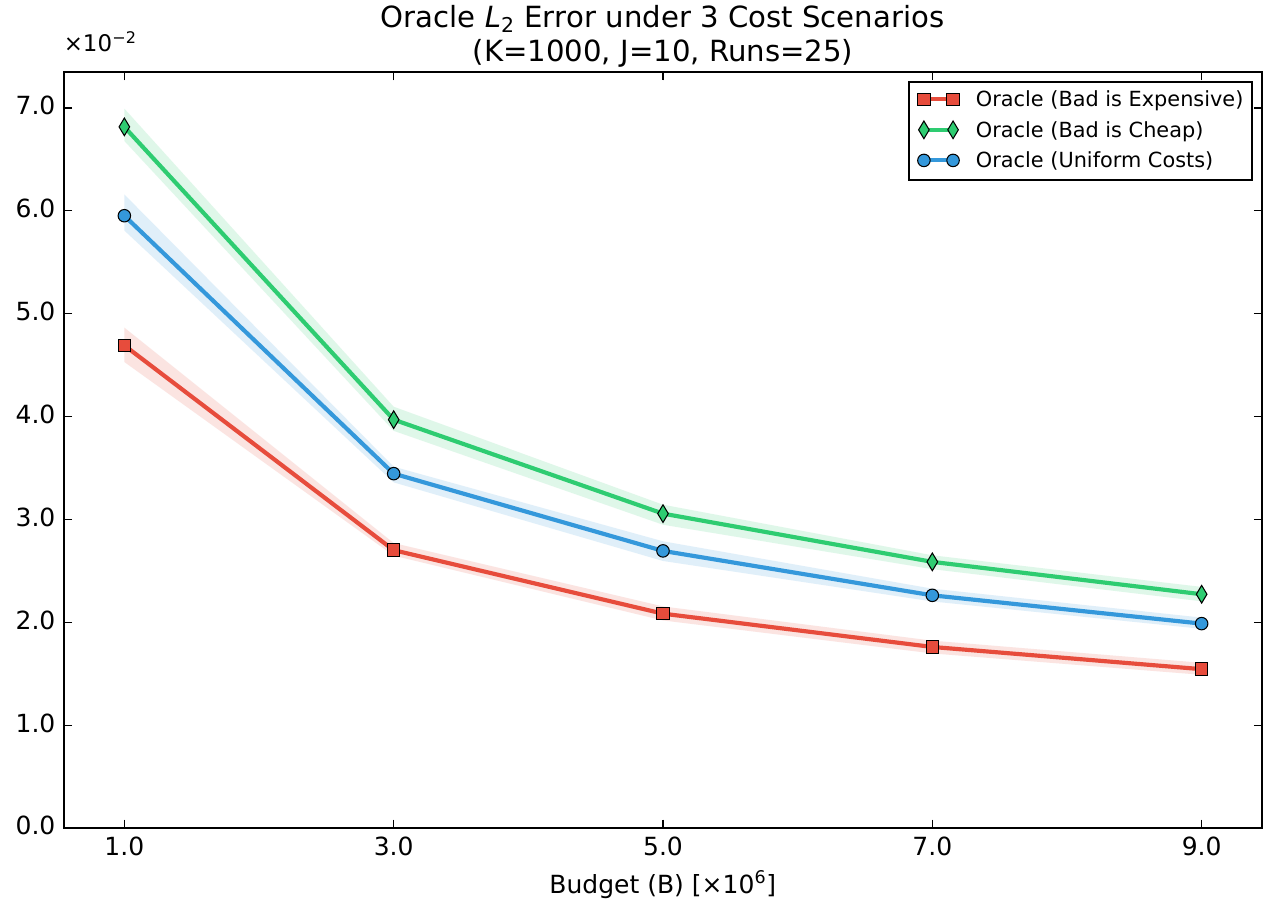}
    \end{subfigure}
    \hspace{10pt}
    \begin{subfigure}[b]{0.3\textwidth}
        \centering
        \includegraphics[width=\textwidth]{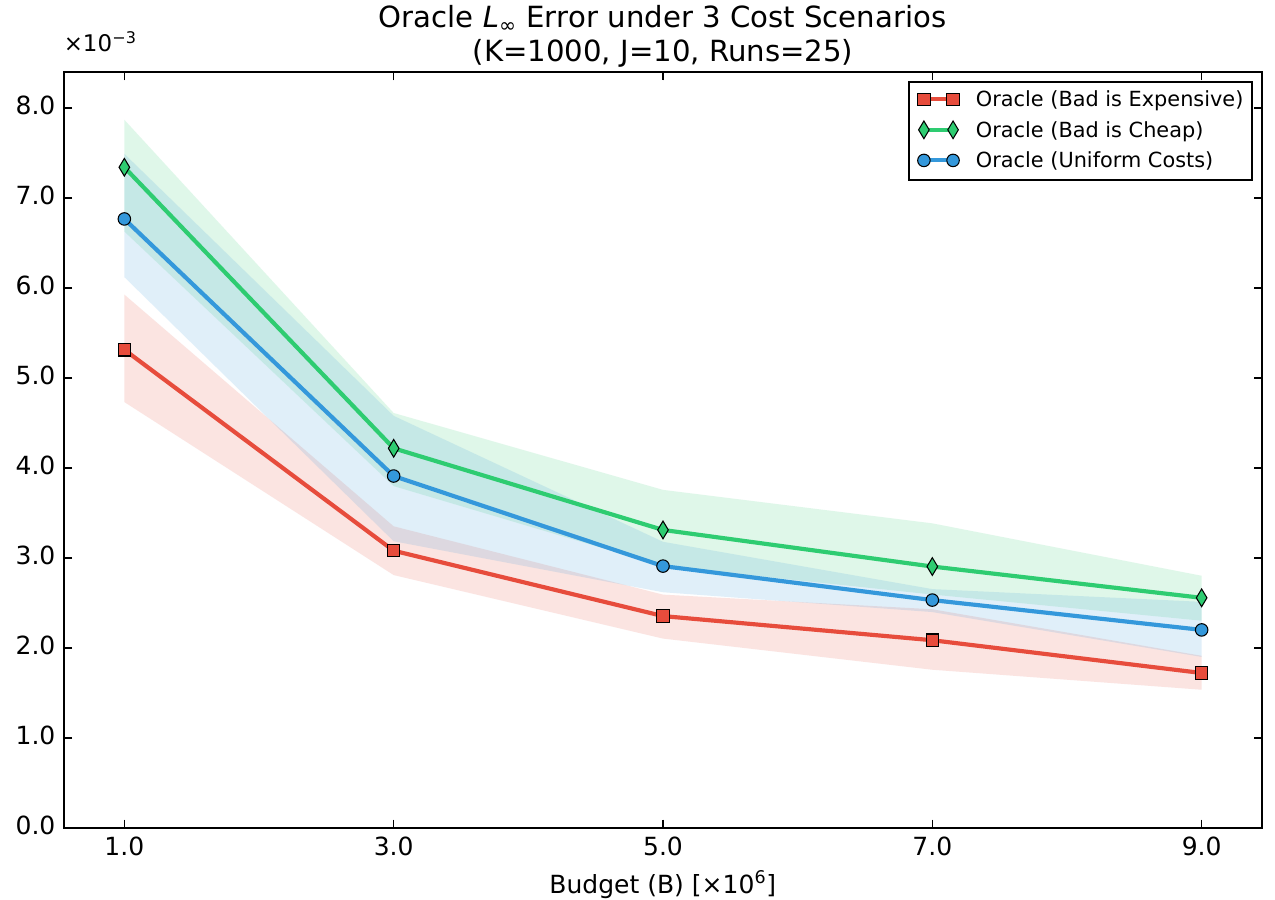}
    \end{subfigure}
    
    \vspace{20pt}
    
    \begin{subfigure}[b]{0.3\textwidth}
        \centering
        \includegraphics[width=\textwidth]{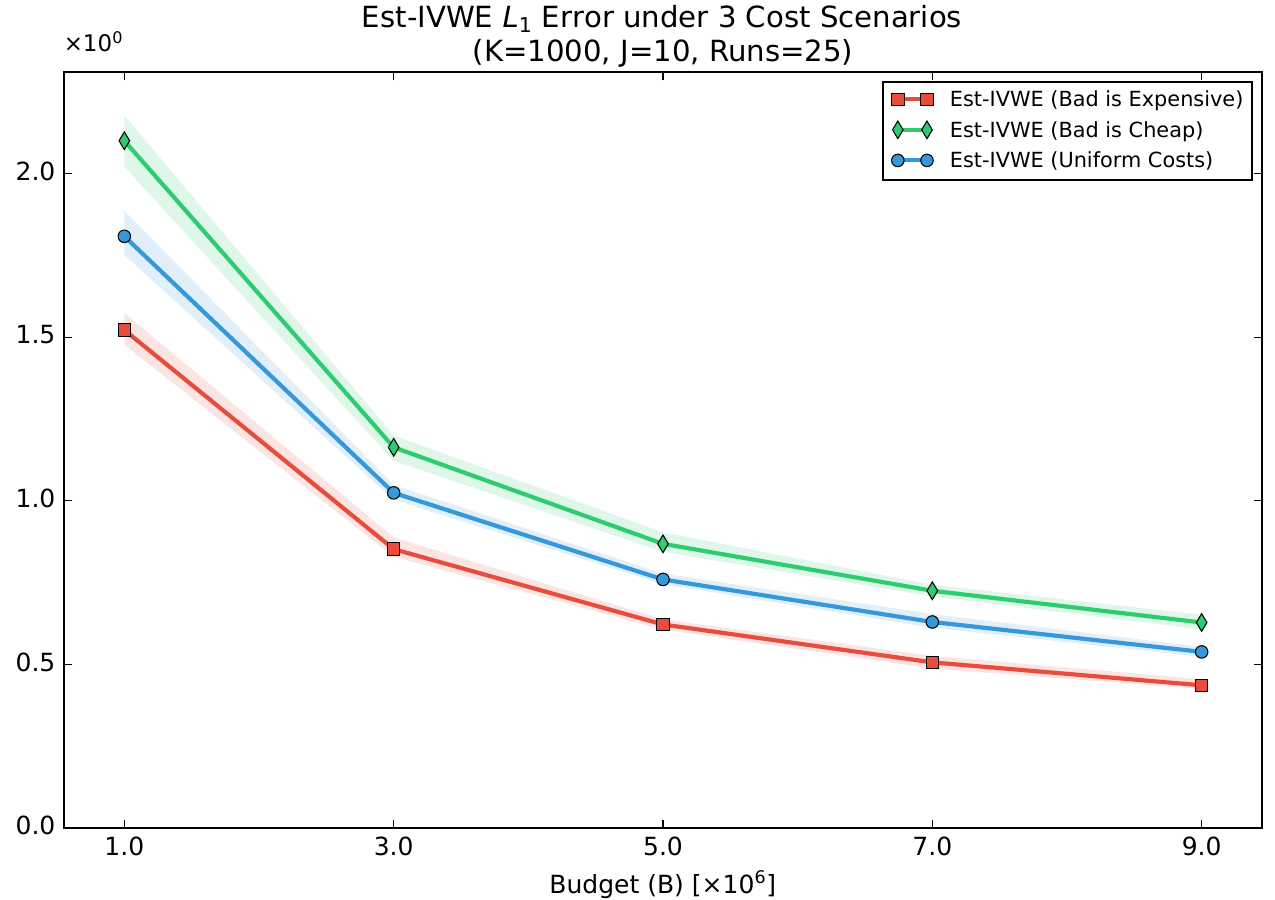}
    \end{subfigure}
    \hspace{10pt}
    \begin{subfigure}[b]{0.3\textwidth}
        \centering
        \includegraphics[width=\textwidth]{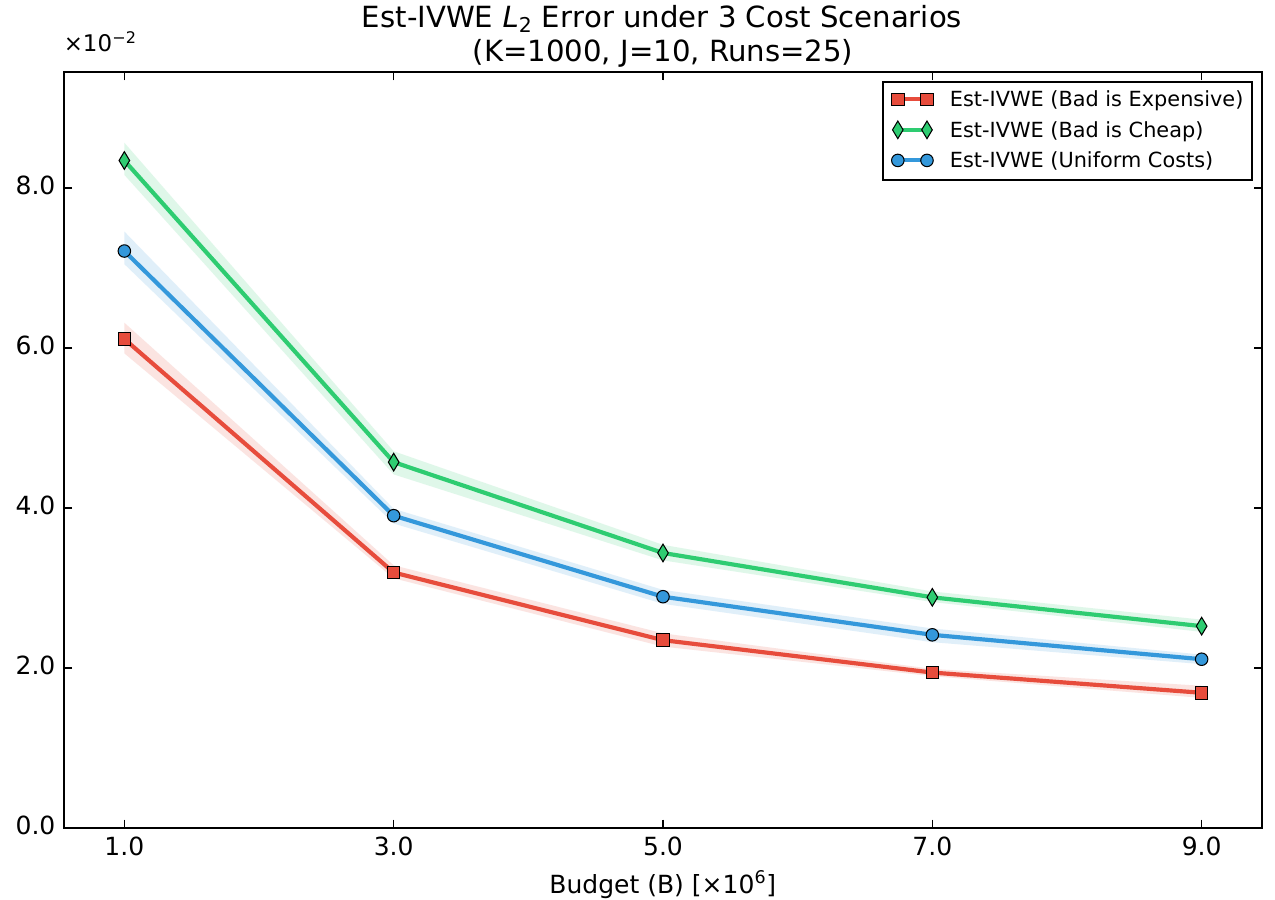}
    \end{subfigure}
    \hspace{10pt}
    \begin{subfigure}[b]{0.3\textwidth}
        \centering
        \includegraphics[width=\textwidth]{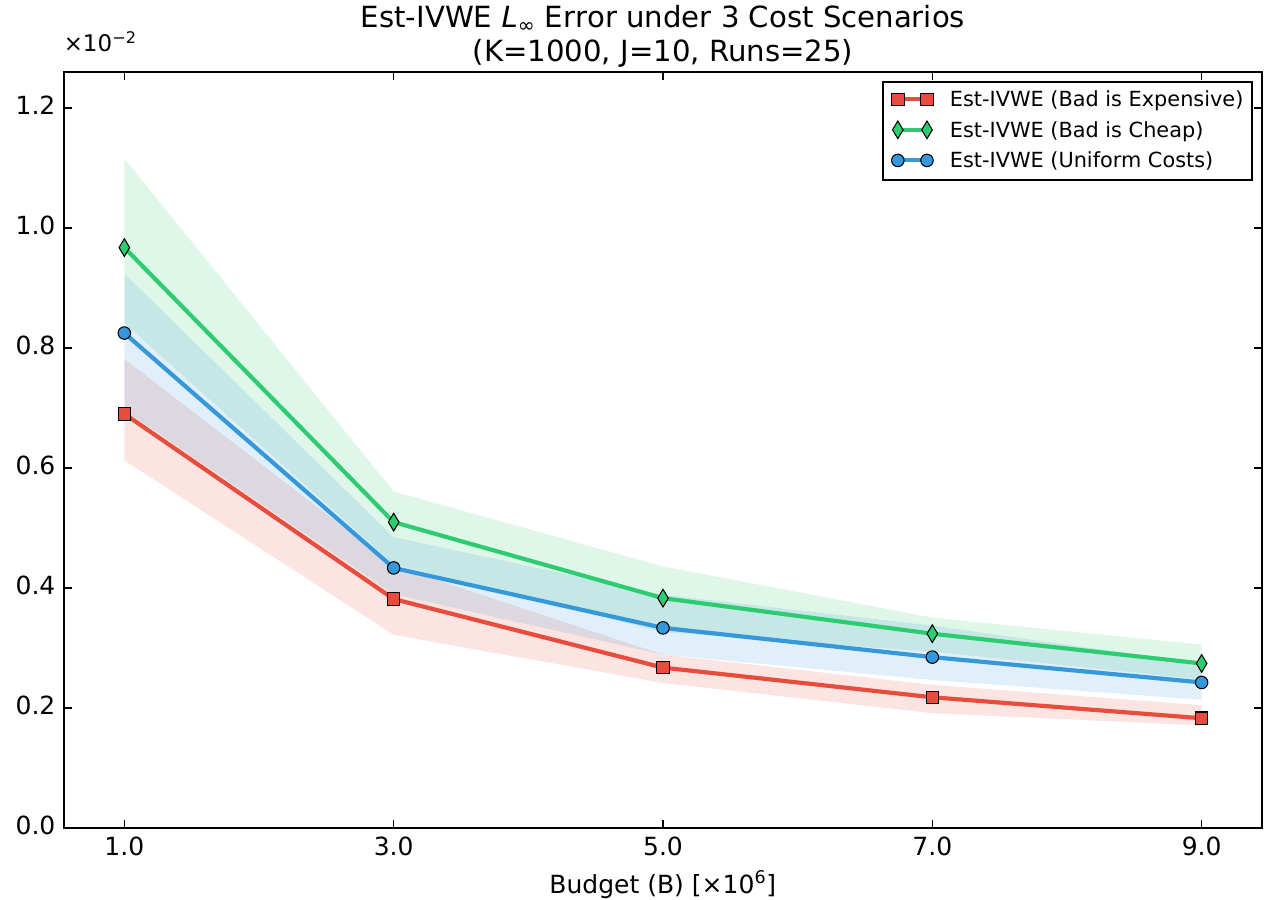}
    \end{subfigure}
    \caption{Impact of Cost Structure on Performances of \textbf{Est-IVWE} and \textbf{Oracle}}
    \label{fig:cost-structure}
\end{sidewaysfigure}
\clearpage

%% file: 909Limitation.tex
\section{Limitations and Future Directions}\label{app:limitation}

Here, we elaborate on the limitations and suggest some future directions for addressing them.

\paragraph{Query-independent costs.}
Our analyses rely on the assumption that querying to judge $j$ incurs a deterministic cost $c_j$ regardless of the query $k$.
One way to approximately enforce this assumption is to impose fixed input and output token lengths for each judge $j$, and define $c_j$ as the corresponding worst-case cost under theses limits.
This recovers our deterministic cost setting, although the realized cost may be smaller when generation terminates early before reaching the limit.

A more realistic model would allow the cost $c_{k,j}$ of allocating query $k$ to judge $j$ to depend on both the query and the generated output.
In such a setting, the costs may be \emph{stochastic} and revealed in an online fashion during the generation.
For example, under per-token API pricing, the input-side cost is known in advance, whereas the output-side cost depends on the stochastic generation length.
This turns the hard deterministic budget constraint considered in our work into a stochastic, knapsack-type constraint.
This setting is related to online learning and allocation under budget constraints \citep{balseiro2023robust, slivkins2023contextual}, where primal-dual algorithms attain $O(\sqrt{T})$ regret.

However, stochastic costs make it non-trivial to adapt the usual fixed-budget best arm identification-type approaches, and the dual algorithms do not transfer directly either, since our objective is estimation rather than regret minimization.
One potential direction is to \emph{reduce} the budget-constrained estimation problem to an online learning problem (see, e.g., \citet{rakhlin2017equivalence, lee2024improved, clerico2025confidence} for related reduction principle) and then combine the resulting formulation with primal--dual allocation
methods. Whether such a reduction can preserves the instance-wise allocation optimality identified in this work remains an open question.

\paragraph{Unbiasedness.}
Some judge biases can be corrected on the fly: swapping the order of responses counters position bias \citep{zheng2023judging}, and length-controlled win rates post-hoc debias verbosity \citep{dubois2024length}.
Statistically, however, if each judge carries a non-vanishing residual bias, then without further information (e.g., the density ratio between the biased and unbiased judge distributions \citep{clemencon2022statistical}), the true score is no longer identifiable, and the problem may not be solvable in general.

One natural structure to impose is that more powerful models have means closer to the ``true" score \emph{and} lower variance. Our framework currently captures only the latter, through the heteroskedastic variance profile.
This kind of cost--accuracy trade-off is known as \emph{multi-fidelity} structure \citep{peherstorfer2018survey, poiani2024optimal}: one has access to $M$ simulators, with costlier simulators providing more accurate simulations, and this literature develops techniques for minimizing cumulative querying cost while retaining a target notion of optimality.
Adapting these approaches to our estimation setting is a promising and challenging direction.

\paragraph{Cross-query structure.}
Our framework treats the scores and variances
across queries and judges as uncorrelated.
Consequently, an observation from one query provides no information about another, and \estivwe{} estimates the variance of every query--judge pair separately through forced exploration.
In practice, however, related queries may share features or latent structure.
For example, if scores or variances are governed by a shared low-dimensional model based on query features, observations from one query could provide information about many others and substantially reduce the required exploration.
Moreover, dependencies may arise among evaluations themselves; for example, \citet{balasubramanian2026dependence} model statistical dependencies among LLM judges using an Ising model.
Exploiting such shared structure and dependence would require replacing our pairwise variance-based formulation with a joint statistical model, and developing the corresponding budget-optimal allocation remains an interesting open direction.

\paragraph{Randomness of repeated evaluations.}
Our analysis assumes that repeated evaluations of each query--judge pair are independent draws from a stationary distribution with mean $s_k$ and variance $\sigma_{k,j}^2$. This assumption is natural only when judges are queried with a positive temperature, as is sometimes done to obtain stochastic judgments and characterize evaluation uncertainty. However, when LLM judges are queried at temperature~$0$, repeated evaluations can be identical or nearly identical. In such settings, repeated sampling provides little or no information for estimating the variance $\sigma_{k,j}^2$, which is essential to our allocation strategy.